\documentclass{article}

\usepackage[final,nonatbib]{neurips_2021}
\usepackage[utf8]{inputenc}
\usepackage{hyperref}       
\usepackage{url}            
\usepackage{booktabs}       
\usepackage{amsfonts}       
\usepackage{nicefrac}       
\usepackage{microtype}      
\usepackage{algorithm}
\usepackage[noend]{algorithmic}
\usepackage[linesnumbered,ruled,vlined, algo2e, resetcount]{algorithm2e}
\usepackage{graphicx}
\usepackage{subfigure}
\usepackage{booktabs} 
\usepackage{amsmath}
\usepackage{amsthm}
\usepackage{float}
\usepackage{bm}
\usepackage{amssymb}
\usepackage{bbm}
\usepackage{color}
\usepackage{mathtools}
\usepackage[super]{nth}
\usepackage{etoolbox}
\usepackage{adjustbox}
\usepackage{enumitem}

\newtheorem{propo}{Proposition}[section]
\newtheorem{lemma}[propo]{Lemma}
\newtheorem{definition}[propo]{Definition}
\newtheorem{coro}[propo]{Corollary}
\newtheorem{thm}{Theorem}
\newtheorem{asmp}{Assumption}

\usepackage[linesnumbered,ruled,vlined, algo2e, resetcount]{algorithm2e}

\usepackage[utf8]{inputenc} 
\usepackage[T1]{fontenc}    
\usepackage{hyperref}       
\usepackage{url}            
\usepackage{booktabs}       
\usepackage{amsfonts}       
\usepackage{nicefrac}       
\usepackage{microtype}      
\usepackage{xcolor}         
\usepackage{minitoc} 

\def\reals{{\mathbb R}}
\def\eps{\epsilon}

\newcommand{\ip}[2]{\left\langle #1, #2 \right \rangle}

\DeclareMathOperator*{\argmax}{arg\,max}

\def\reals{{\mathbb R}}
\def\prob{{\mathbb P}}

\def\cN{{\cal N}}
\def\eps{\varepsilon}
\def\tukey{\text{Tukey}}
\def\E{\mathbb E}

\DeclareMathOperator{\Tr}{Tr}

\title{Robust and differentially private mean estimation}

\author{%
  Xiyang Liu, \;\;
   Weihao Kong, \;\;
   Sham Kakade, \;\;
   Sewoong Oh \\
   Paul G. Allen School of Computer Science and Engineering, \\
  University of Washington \\
  \texttt{\{xiyangl,whkong,sham,sewoong\}@cs.washington.edu } 
}

\begin{document}

\maketitle

\begin{abstract}
 In statistical learning and analysis from shared data, which is increasingly widely adopted in platforms such as federated learning and meta-learning, there are two major concerns: privacy and robustness. Each participating individual should be able to contribute without the fear of leaking one's sensitive information. 
 At the same time, the system should be robust in the presence of malicious participants inserting corrupted data. Recent algorithmic advances in learning from shared data focus on either one of these threats, leaving the system vulnerable to the other. We bridge this gap for the canonical problem of estimating the mean from i.i.d.~samples. 
We introduce PRIME, which is the first efficient algorithm that achieves both privacy and robustness for a wide range of distributions. 
We further complement this result with a novel exponential time algorithm that improves the sample complexity of PRIME, achieving a near-optimal guarantee and matching a known lower bound for (non-robust) private mean estimation. This proves that  there is no extra statistical cost to simultaneously guaranteeing privacy and robustness.

\end{abstract}

\doparttoc 
\faketableofcontents

\section{Introduction}

When releasing database statistics  on a collection of entries from individuals, we would ideally like to make it impossible to reverse-engineer each individual's potentially sensitive information. 
Privacy-preserving techniques add just enough randomness tailored to the  statistical task  to guarantee protection. 
At the same time, it is becoming increasingly common to apply such techniques to databases collected from  multiple sources, not all of which  can be trusted.
Emerging data access frameworks, such as federated analyses across users' devices or data silos  \cite{kairouz2019advances},  make it easier to temper with such collected datasets, 
leaving private statistical analyses vulnerable to a malicious corruption of  a fraction of the data.

Differential privacy  has emerged as a widely accepted de facto measure of privacy, which is now a standard in releasing the statistics of the U.S.~Census data \cite{census} statistics and also 
deployed in  real-world commercial systems \cite{apple,google1,google2}. 
A statistical analysis is said to be  {\em differentially private} (DP) if the likelihood of the (randomized) outcome does not change significantly  when a single  arbitrary  entry is added/removed (formally defined in \S\ref{sec:dp}). 
This provides a strong privacy guarantee:  
even a powerful adversary who knows all the other entries in the database cannot confidently identify 
 whether a particular individual is participating in the database based on the outcome of the analysis. This ensures    {\em plausible deniability}, central to protecting an individual's privacy. 
 
 In this paper, we focus on one of the most canonical problems in statistics: 
  estimating the mean of a  distribution  from i.i.d.~samples. 
  For distributions with
  unbounded support, such as sub-Gaussian and heavy-tailed distributions, 
  fundamental trade-offs between accuracy, sample size, and privacy  have only recently been  identified  \cite{KV17,KLSU19,kamath2020private,aden2020sample} and  efficient private estimators proposed. 
  However, these approaches are brittle when a fraction of the data is corrupted, posing  a real threat, referred to as \textit{data poisoning} attacks \cite{chen2017targeted,xiao2015feature}. 
  In defense of such attacks, robust (but not necessarily private) statistics  has emerged as a popular setting of  recent algorithmic and mathematical breakthroughs \cite{steinhardt2018resilience,diakonikolas2017beingrobust}.

One might be misled into thinking that privacy   ensures robustness since DP guarantees that 
a single outlier cannot change the estimation too much. 
This intuition is true only in a low dimension; each sample has to be an obvious outlier to significantly change the mean. 
However, in a high dimension, each corrupted data point can look perfectly uncorrupted but still shift the mean significant when colluding together (e.g., see Fig.~\ref{fig:intro}). 
Focusing on the canonical problem of mean estimation, 
we introduce novel algorithms that achieve  robustness and privacy simultaneously  even when a fraction of data is corrupted arbitrarily. 
For such algorithms, 
there is a fundamental question of interest:  do we need more samples to make private mean estimation also robust against adversarial corruption?

\noindent{\bf Sub-Gaussian distributions.} 
If we can afford exponential run-time  in the dimension,  robustness can be achieved without extra cost in  sample complexity. 
We introduce a novel estimator that $(i)$ satisfies $(\varepsilon,\delta)$-DP, $(ii)$ achieves near-optimal robustness under $\alpha$-fraction of corrupted data, achieving accuracy of $O(\alpha\sqrt{\log(1/\alpha)})$ nearly matching the fundamental lower bound of $\Omega(\alpha)$ that holds even for a (non-private) robust mean estimation  with {\em infinite} samples, 
and $(iii)$ achieves near-optimal sample complexity matching that of a fundamental  lower bound for a (non-robust) private mean estimation as shown in Table~\ref{tbl:Gauss}.  In particular, we emphasize that the unknown true mean $\mu$ can be any vector in $\reals^d$, and we do not require a known bound on the norm, $\|\mu\|$, that some previous work requires (e.g., \cite{KV17}). 

\begin{thm}[Informal  Theorem~\ref{thm:subgauss},  exponential time]  
\label{thm:informal_exp} 
Algorithm~\ref{alg:exp} is $(\varepsilon,\delta)$-DP. 
When $\alpha$ fraction of the data is arbitrarily corrupted from $n$ samples from a $d$-dimensional sub-Gaussian distribution with mean $\mu$ and an identity sub-Gaussian parameter, if $n=\widetilde\Omega(d/\alpha^2 + (d+d^{1/2}\log(1/\delta))/(\alpha\varepsilon))$ then   Algorithm~\ref{alg:exp} achieves $\|\hat\mu-\mu\|_2=O(\alpha \sqrt{\log(1/\alpha)})$ w.h.p.
\end{thm}

We introduce PRIME (PRIvate and robust Mean Estimation)  in \S\ref{sec:prime} with details in Algorithm~\ref{alg:prime} in Appendix \ref{sec:prime_algorithm}, to achieve computational efficiency. It requires a run-time of only $\widetilde{O}(d^3+nd^2)$, but at the cost of requiring extra $d^{1/2}$ factor larger number of  samples. 
This cannot be improved upon with current techniques since 
efficient robust estimators  rely on the top PCA directions of the covariance matrix to detect outliers. \cite{wei2016analysis} showed that 
$\widetilde{\Omega}(d^{3/2})$ samples are necessary to compute PCA directions while preserving $(\varepsilon,\delta)$-DP when $\|x_i\|_2 = O(\sqrt{d})$. It remains an open question if this $\widetilde{\Omega}(d^{3/2}/(\alpha\varepsilon)) $ bottleneck is fundamental; no matching lower bound is currently known. 
We emphasize again that the unknown true mean $\mu$ can be any vector in $\reals^d$, and PRIME does not require a known bound on the norm, $\|\mu\|$, that some previous work requires (e.g., \cite{KV17}). 

\begin{thm}[Informal  Theorem~\ref{thm:main2},  polynomial time] 
    PRIME is  $(\varepsilon,\delta)$-DP and under the assumption of Thm.\ref{thm:informal_exp}, if $n = \widetilde\Omega(d/\alpha^2 + (d^{3/2}\log(1/\delta))/(\alpha \varepsilon))$,  achieves $\|\hat\mu-\mu\|_2 = O(\alpha\sqrt{\log (1/\alpha)})$ w.h.p.
\label{thm:informal_poly}
\end{thm}
\begin{table*}[h]
\begin{center}
\begin{tabular}{ |c|c |c| c| }
 \hline
  & Upper bound (poly-time)& Upper bound (exp-time)& Lower bound  \\\hline  
    $(\varepsilon,\delta)$-DP \cite{KLSU19} & $\widetilde{O}(\frac{d}{\alpha^2}+\frac{d \log^{1/2}(1/\delta)}{\alpha\varepsilon} )$ 
   & $\widetilde{O}(\frac{d}{\alpha^2}+\frac{d}{\alpha\varepsilon} )^\clubsuit$  
   & $\widetilde{\Omega}( \frac{d}{\alpha^2} + \frac{d}{\alpha\varepsilon}  )^{\spadesuit}$ \\\hline  
  $\alpha$-corruption \cite{dong2019quantum} & $\widetilde{O}(\frac{d}{\alpha^2})$ & $\widetilde{O}(\frac{d}{\alpha^2})$ & 
 $ \Omega(\frac{d}{\alpha^2})$ \\  \hline
 $\alpha$-corruption and & 
   $\widetilde{O}\big(\,\frac{d}{\alpha^2}+\frac{d^{3/2}\log(1/\delta)  }{\alpha\varepsilon}  \,\big)$ & 
   $\widetilde{O}(\frac{d}{\alpha^2}+\frac{d+d^{1/2}\log(1/\delta)}{\alpha \varepsilon}) $
   & $\widetilde{\Omega}( \frac{d}{\alpha^2} + \frac{d}{\alpha\varepsilon}  )^{\spadesuit}$ \\
  $(\varepsilon,\delta)$-DP (this paper) & 
  [Theorem~\ref{thm:main2}]&  
  [Theorem~\ref{thm:subgauss}]& \cite{KLSU19}
  \\\hline  
\end{tabular}
\end{center}
\caption{For estimating the mean $\mu\in \reals^d$ of a  {\em sub-Gaussian} distribution with a known covariance, 
we list the sufficient or necessary conditions on the sample sizes to achieve an error $\|\hat\mu-\mu\|_2 =\widetilde{O}(\alpha)$ under $(\varepsilon,\delta)$-DP, corruption of an $\alpha$-fraction of samples, and both. 
 $^\clubsuit$ requires the distribution to be a Gaussian \cite{bun2019private} and $^\spadesuit$ requires $\delta\leq \sqrt{d}/n$.}
\label{tbl:Gauss} 
\end{table*}

\noindent{\bf Heavy-tailed distributions.} 
When  samples are drawn from  a distribution with a bounded covariance, 
parameters of Algorithm~\ref{alg:exp} can be modified  to nearly match the optimal sample complexity of (non-robust) private mean estimation in Table~\ref{tbl:heavytail}. This algorithm also  matches the fundamental limit on the accuracy of (non-private) robust estimation, which in this case is $\Omega(\alpha^{1/2})$.

\begin{thm}[Informal  Theorem~\ref{thm:heavytail_exp},  exponential time]  
\label{thm:informal_exp_ht} 
 From a  distribution with mean $\mu\in {\mathbb R}^d$ and covariance $\Sigma \preceq {\mathbf I}$, 
 $n$ samples are drawn and $\alpha$-fraction is  corrupted. Algorithm~\ref{alg:exp} is $(\varepsilon,\delta)$-DP and
 if $n = \widetilde\Omega((d+d^{1/2}\log(1/\delta))/(\alpha\varepsilon) + d^{1/2}\log^{3/2}(1/\delta)/\varepsilon)$ 
 achieves  $\|\hat\mu-\mu\|_2=O(\alpha^{1/2} )$ w.h.p.
 \end{thm} 

The proposed  {\sc PRIME-ht} for covariance bounded distributions achieve computational efficiency at the cost of an extra factor of $d^{1/2}$ in sample size. This bottleneck is also due to DP PCA, and it remains open whether  this gap can be closed by an efficient estimator. 

\begin{thm}[Informal  Theorem~\ref{thm:heavytail_poly},  polynomial time]  
\label{thm:informal_poly_ht} 
{\sc PRIME-ht} is $(\varepsilon,\delta)$-DP and if $n = \widetilde\Omega((d^{3/2}\log(1/\delta))/(\alpha\varepsilon) )$ achieves $\|\hat\mu-\mu\|_2=O(\alpha^{1/2} )$ w.h.p.~under the assumptions of Thm.~\ref{thm:informal_exp_ht}.
\end{thm}

%
%
%
\begin{table*}[h]
\begin{center}
\begin{tabular}{ |c|c |c| c| }
 \hline
  & Upper bound (poly-time)& Upper bound (exp-time)& Lower bound  \\\hline  
   $(\varepsilon,\delta)$-DP \cite{kamath2020private} & $\widetilde{O}(\frac{d\,{\log^{1/2}(1/\delta)}}{\alpha\varepsilon} )$ 
   & $\widetilde{O}(\frac{d\log^{1/2}(1/\delta)}{\alpha\varepsilon} )$  
   & ${\Omega}(  \frac{d}{\alpha \varepsilon}  )$ \\\hline  
 $\alpha$-corruption \cite{dong2019quantum} & $\widetilde{O}(\frac{d}{\alpha})$ & $\widetilde{O}(\frac{d}{\alpha})$ & 
 $ \Omega(\frac{d}{\alpha})$ \\  \hline
 $\alpha$-corruption and & 
   $\widetilde{O}\big(\, \frac{d^{3/2}\log(1/\delta) }{\alpha \varepsilon}\,\big)$ & 
   $\widetilde{O}( \frac{d+
   d^{1/2}\log^{3/2}(1/\delta)
   }{\alpha \varepsilon})$
   &  $\Omega(\frac{d}{\alpha \varepsilon})$ \\
  $(\varepsilon,\delta)$-DP (this paper) & 
  [Theorem~\ref{thm:heavytail_poly}]&  
  [Theorem~\ref{thm:heavytail_exp}]& 
  (\cite{kamath2020private})
  \\\hline  
\end{tabular}
\end{center}
\caption{For estimating the mean $\mu\in \reals^d$ of a  covariance  bounded  distribution, 
we list the sufficient or necessary conditions on the sample size to achieve an error $\|\hat\mu-\mu\|_2 = O(\alpha^{1/2})$ under $(\varepsilon,\delta)$-DP, corruption of an $\alpha$-fraction of samples, and both. 
 }
\label{tbl:heavytail} 
\end{table*}

\subsection{Technical contributions}

    We introduce PRIME which simultaneously achieves $(\varepsilon,\delta)$-DP and robustness against $\alpha$-fraction of corruption. 
    A major challenge in making a standard filter-based robust estimation algorithm (e.g., \cite{diakonikolas2017beingrobust}) private is the high sensitivity of the filtered set that we pass from one iteration to the next. 
    We propose a new framework which  makes private only the statistics of the set, hence significantly reducing the sensitivity. Our major innovation is a tight analysis of the end-to-end sensitivity of this multiple interactive accesses to the database. This  is critical in achieving robustness while preserving privacy and is also of independent interest in making general  iterative filtering algorithms private.
    

    The classical filter approach (see, e.g.~\cite{diakonikolas2017beingrobust}) needs to access the database $O(d)$ times, which brings an extra $O(\sqrt{d})$ factor in the sample complexity due to DP composition. In order to reduce the iteration complexity, following the approach in~\cite{dong2019quantum}, we propose filtering multiple directions simultaneously using a new score based on the matrix multiplicative weights (MMW). In order to privatize the MMW filter, our major innovation is a novel adaptive filtering algorithm {\sc DPthreshold}($\cdot$) that outputs a {\em single private threshold} which guarantees sufficient progress at every iteration. This brings the number of database accesses from $O(d)$ to $O((\log d)^2)$.

One downside of PRIME is that it requires an extra $d^{1/2}$ factor in the sample complexity, compared to known lower bounds for (non-robust) DP mean estimation.
To investigate whether this is also necessary, we propose a {\em sample optimal} exponential time robust mean estimation algorithm in \S\ref{sec:subgauss} and prove that there is no extra statistical cost to jointly requiring privacy and robustness. Our major technical innovations is in using {\em resilience property of the dataset} to not only find robust mean (which is the typical use case of resilience) but also bound sensitivity of that robust mean. 



\subsection{Preliminary on differential privacy (DP)} 
\label{sec:dp}

DP is a formal  metric for measuring privacy leakage when a dataset is accessed with a query \cite{dwork2006calibrating}. 

\begin{definition}
Given two datasets $S =\{x_i\}_{i=1}^n$ and $S'=\{x_i'\}
_{i=1}^{n'}$, we say $S$ and $S'$ are {\em neighboring} if 
 $d_\triangle( S, S'  ) \leq 1$ 
     where $d_\triangle(S,S') \triangleq \max\{| S \setminus S'|,| S' \setminus S|\} $, which is denoted by  $S \sim S'$. 
For an output of a stochastic query $q$ on a database, we say $q$ satisfies 
$(\varepsilon,\delta)${\em-differential  privacy} for some $\varepsilon>0$ and $\delta\in(0,1)$ if ${\mathbb P}(q(S)\in A) \leq e^\varepsilon {\mathbb P}(q(S') \in A) + \delta$ for all $S\sim S'$ and all subset $A$. 
    \label{def:dp}
\end{definition}
Let $z\sim {\rm Lap}(b)$ be a random vector with entries  i.i.d.~sampled from Laplace distribution with pdf $(1/2b)e^{-|z|/b}$. Let $z\sim {\cal N}(\mu,\Sigma)$ denote a Gaussian random vector  with mean $\mu$ and covariance $\Sigma$. 
\begin{definition} 
\label{def:output} 
    The {\em sensitivity} of a  query $f(S) \in  {\mathbb R}^k$ is defined as $\Delta_p= \sup_{S \sim S'} \|f(S) - f
    (S') \|_p$ for a norm $\|x\|_p=(\sum_{i\in[k]} |x_i|^p)^{1/p}$. 
    For $p=1$, the Laplace mechanism outputs 
    $f(S)+{\rm Lap}(\Delta_1/ \varepsilon)$ 
    and achieves $(\varepsilon,0)$-DP  \cite{dwork2006calibrating}. 
    For $p=2$, the Gaussian mechanism outputs $f(S)+{\cal N}(0,(\Delta_2 (\sqrt{2\log(1.25/\delta)})/\varepsilon)^2{\mathbf I})$
    and 
    achieves $(\varepsilon, \delta)$-DP \cite{dwork2014algorithmic}. 
\end{definition}
We use these output perturbation mechanisms along with the  exponential mechanism \cite{mcsherry2007mechanism} as building blocks.  Appendix~\ref{sec:related} provides 
 detailed survey of  privacy  and robust estimation.

\subsection{Problem formulation}

We are given $n$ samples from a  sub-Gaussian distribution with a known covariance but unknown mean, and $\alpha$ fraction of the samples are  corrupted by an adversary.  Our goal is to estimate the unknown mean. 
We emphasize that the unknown true mean $\mu$ can be any vector in $\reals^d$, and we do not require a known bound on the norm, $\|\mu\|$, that some previous work requires (e.g., \cite{KV17}). 
We follow the standard definition of adversary in \cite{diakonikolas2017beingrobust}, which can adaptively choose which samples to corrupt and arbitrarily replace them with any points. 

\begin{asmp} 
    \label{asmp:adversary} 
        An uncorrupted dataset $S_{\rm good} $ consists of $n$ i.i.d.~samples from a $d$-dimensional sub-Gaussian distribution with  mean $\mu\in\reals^d$ and   covariance ${\mathbb E}[x x^\top]={\mathbf I_d}$, which is  $1$-sub-Gaussian, i.e., ${\mathbb E}[\exp(v^\top x)]\leq \exp(\|v\|_2^2/2)$ for all $v\in \reals^d$. 
        For some $\alpha\in(0,1/2)$, we are given a corrupted dataset $S=\{x_i\in{\mathbb R}^d \}_{i=1}^n$ where an adversary  adaptively inspects all the samples in $S_{\rm good}$, removes $\alpha n$ of them, and replaces them with $S_{\rm bad}$ which are $\alpha n$  arbitrary points in ${\mathbb R}^d$. 
\end{asmp} 
Similarly, we consider the same problem for heavy-tailed distributions with a bounded covariance.
We present the  assumption and main results for covariance bounded distributions in  Appendix~\ref{sec:heavytail}.

\noindent{\bf Notations.} 
Let  $[n]=\{1,2,\ldots,n\}$.
For $x\in{\mathbb R}^d$, we use $\|x\|_2=(\sum_{i\in[d]} (x_i)^2)^{1/2}$ to denote the Euclidean norm. 
For $X\in {\mathbb R}^{d\times d}$, we use $\|X\|_2=\max_{\|v\|_2=1} \| X v \|_2$ to denote the spectral norm. The $d\times d$ identity matrix is ${\mathbf I}_{d\times d}$. 
Whenever it is clear from context, we use $S$ to denote both a set of data points and also the set of indices of those data points. 
$\widetilde{O}$ and $\widetilde{\Omega}$ hide poly-logarithmic factors in $d,n,1/\alpha$, and the failure probability.

\noindent{\bf Outline.} 
We present PRIME for  sub-Gaussian distribution in \S\ref{sec:efficient}, 
and present theoretical analysis in \S\ref{sec:prime_analysis}. 
We then introduce  an exponential time algorithm with near optimal guarantee in 
\S\ref{sec:subgauss}. 
Due to space constraints, 
analogous results for heavy-tailed distributions are presented in 
Appendix~\ref{sec:heavytail}.

\section{PRIME: efficient algorithm for robust and DP mean estimation}
\label{sec:efficient}
In order to describe the proposed algorithm  PRIME, we need to first describe  a standard (non-private) iterative filtering algorithm for  robust mean estimation. 
   
   
   \subsection{Background on (non-private) iterative filtering for robust mean estimation} 
   \label{sec:standard} 
   
{\em Non-private} robust mean estimation approaches recursively apply the following {\em filter}, whose framework is first proposed in~\cite{diakonikolas2019robust}. 
Given a dataset $S=\{x_i\}_{i=1}^n$, the current set $S_0\subseteq [n]$ of data points is 
updated starting with  $S_1=[n]$. 
At each step, the following filter (Algorithm 1 in~\cite{li-notes}) attempts to  detect  the corrupted data points and remove them. 
\begin{enumerate} 
    \vspace{-0.15cm}
    \item Compute the top eigenvector $v_t\gets \argmax_{v: \|v\|_2=1}v^\top{\rm Cov}(S_{t-1})v$ of the covariance of the current data set $\{x_i\}_{i\in S_{t-1}}$\;;
    \vspace{-0.1cm}
    \item Compute scores for all data points $j\in S_{t-1}$: $\tau_j \gets \left(v_t^\top \left(x_j-{\rm Mean}(S_{t-1})\right)\right)^2 $ \;;
    \vspace{-0.1cm}
    \item Draw a random threshold: $Z_t \leftarrow {\rm Unif}([0,1])$\;;
    \vspace{-0.1cm}
    \item Remove outliers from $S_{t-1}$ defined as  $\{ i \in S_{t-1} \,:\, \tau_i$ is in the largest $2\alpha$-tail of $ \{\tau_j\}_{j\in S_{t-1}}$  and $\tau_i \geq  Z_t\,\tau_{\rm max}   \}$, where $\tau_{\rm max} = \max_{j\in S_{t-1}} \tau_j $\;
\end{enumerate}
This is repeated until the empirical covariance is sufficiently small and the empirical mean $\hat\mu$ is output. 
At a high level, the correctness of this algorithm relies on the key observation that the $\alpha$-fraction of adversarial corruption can not significantly change the mean of the dataset without introducing large eigenvalues in the empirical covariance. Therefore, the algorithm finds top eigenvector of the empirical covariance in step $1$, and tries to correct the empirical covariance by removing corrupted data points. Each data point is assigned a score in step 2 which indicates the ``badness'' of the data points, and a threshold $Z_t$ in step 3 is carefully designed   such that step 4 guarantees to remove more corrupted data points than good data points (in expectation). 
This guarantees the following bound achieving the near-optimal sample complexity shown in the second row of Table~\ref{tbl:Gauss}. 
A formal description of this  algorithm is in Algorithm~\ref{alg:nonprivatefiltering} in Appendix~\ref{sec:nonprivate}.

\begin{propo}[Corollary of {\cite[Theorem~2.1]{li-notes}}]
    \label{pro:nonprivatefiltering}
    Under assumption \ref{asmp:adversary},  the above filtering algorithm achieves accuracy $\|\hat\mu - \mu\|_2 \leq O(\alpha \sqrt{\log(1/\alpha)})$ w.p.~$0.9$
    if $n\geq \widetilde{\Omega}(d/\alpha^2)$ . 
\end{propo}

\noindent{\bf Challenges in making robust mean estimation private.} To get a DP and robust mean, a naive attempt is to apply a standard output perturbation mechanism to  $\hat\mu$. However, this is obviously  challenging
since the end-to-end sensitivity  is  intractable. The  standard recipe to circumvent  this  is to make the current ``state'' $S_t$ private at every iteration. Once  $S_{t-1}$ is private (hence, public knowledge), 
making the next ``state'' $S_{t}$ private is simpler. We only need to analyze the sensitivity of a single step and apply some output perturbation  mechanism with  $(\varepsilon_t,\delta_t)$. 
End-to-end privacy is guaranteed by accounting for all these $(\varepsilon_t,\delta_t)$'s  using the advanced composition \cite{composition}. This recipe has been quite successful, for example,  in training neural networks  with (stochastic) gradient descent  \cite{abadi2016deep}, where the current state can be the optimization variable $\mathbf{x}_t$. However, for the above (non-private) filtering algorithm, this standard recipe fails, since the state $S_t$ is a set and has large sensitivity. 
Changing a single data point in $S_{t}$ can significantly alter  which (and how many) samples are filtered out.

\subsection{A new framework for {\em private} iterative filtering }
\label{sec:novel} 

Instead of making the (highly sensitive) $S_t$ itself private, we  propose a new framework which  makes private only the statistics of $S_t$:  the mean $\mu_t$ and the top principal direction $v_t$. 
There are two versions of this algorithm,  which output the exactly same $\hat{\mu}$ with the exactly same privacy guarantees, but are written from two different perspectives. We present here the {\em interactive} version from the perspective of an analyst accessing the dataset  via DP  queries ($q_{\rm range},q_{\rm size},q_{\rm mean},q_{\rm norm}$ and $q_{\rm PCA}$), because this version makes clear the inner operations of each private mechanisms, hence making $(i)$ the sensitivity analysis  transparent, $(ii)$ checking the correctness of privacy guarantees easy,  and $(iii)$ tracking privacy accountant simple. 
In practice, one should implement  the {\em centralized} version (Algorithm~\ref{alg:privatefiltering} in Appendix~\ref{sec:app_novel}), which is significantly more efficient.

\begin{algorithm2e}[ht]
   \caption{Private iterative filtering (interactive version)}
   \label{alg:DPfilter_interactive_main}
   	\DontPrintSemicolon 
\KwIn{ $S=\{x_i\}_{i\in [n]}$, $\alpha\in (0,1/2)$,  probability $\zeta\in(0,1)$, \# of iterations $T= \Theta(d)$,  $(\varepsilon,\delta)$  }
	\SetKwProg{Fn}{}{:}{}
	{
	$(\bar x, B) \leftarrow q_{\rm range}(S,  0.01\varepsilon,0.01\delta)$  \\
	    $\varepsilon_1\gets  \min\{0.99\varepsilon,0.9\}/(4\sqrt{2T\log(2/\delta)}), \; \delta_1\gets0.99\delta/(8T) $\\
	    {\bf if} $n<(4/\varepsilon_1)\log(1/(2\delta_1))$  {\bf then Output:} $\emptyset$ \label{line3:DPfilter_interactive_main}\; 
	    \For{$t=1,\ldots,T$}{
	    $n_t\gets q_{\rm size}(\{(\mu_\ell,v_\ell,Z_\ell)\}_{\ell\in[t-1]},\varepsilon_1,\bar x, B)$,   
	   {\bf if} $n_t < 3n/4$ {\bf then} {\bf Output: $\emptyset$} \label{line5:DPfilter_interactive_main}\;
	   $ \mu_t \leftarrow q_{\rm mean}(\{(\mu_\ell,v_\ell,Z_\ell)\}_{\ell\in[t-1]}, \varepsilon_1 ,\bar x, B) $ \; 
	   $\lambda_t\gets q_{\rm norm}(\{(\mu_\ell,v_\ell,Z_\ell)\}_{\ell\in[t-1]},\mu_t,\varepsilon_1 ,\bar x, B)$\; 
	   {\bf if} $\lambda_t\leq (C-0.01)\alpha\log1/\alpha $ {\bf then} \KwOut{$\mu_{t}$} \label{line8:DPfilter_interactive_main}
	  $v_t \leftarrow q_{\rm PCA} (\{(\mu_\ell,v_\ell,Z_\ell)\}_{\ell\in[t-1]}, \mu_t, \varepsilon_1, \delta_1 ,\bar x, B) )$ \;
	        $Z_t \leftarrow {\rm Unif}([0,1])$ 
	    }
	    \KwOut{$\mu_t$} 
    } 
\end{algorithm2e} 

We give a high-level  explanation of each step of Algorithm~\ref{alg:DPfilter_interactive_main} here and give the formal definitions of all the queries in Appendix~\ref{sec:app_novel}. 
First, $q_{\rm range}$ returns (the parameters of) a hypercube $\bar{x}+[-B/2,B/2]^d$ that is guaranteed to include all uncorrupted samples while preserving privacy. 
This is achieved by running $d$ coordinate-wise private histograms and selecting $\bar{x}_j$ as the center of the  largest bin for the $j$-th coordinate. Since covariance is ${\bf I}$, $q_{\rm range}$ returns a fixed $B=8\sigma\sqrt{\log(d n /\zeta)}$.
Such an adaptive estimate of the support is critical in tightly bounding the sensitivity of all subsequent queries, which operate on the  clipped   dataset; all data points are projected as ${\cal P}_{\bar{x}+[-B/2,B/2]^d}(x)=\arg\min_{y\in \bar{x}+[-B/2,B/2]^d} \|y-x\|_2$ in all the queries that follow. With clipping, a single data point can now change at most by $B\sqrt{d}$.

The subsequent steps perform 
the non-private filtering algorithm of  \S\ref{sec:standard}, but with private statistics $\mu_t$ and $v_t$. As the set $S_t$ changes over time, we lower bound its size (which we choose to be $|S_t|>n/2$) to upper bound the sensitivity of other queries $q_{\rm mean}, q_{\rm norm}$ and $q_{\rm PCA}$. 


At the $t$-th iterations, every time a query is called the data curator $(i)$ uses $(\bar x, B)$ to clip the data, $(ii)$ computes $S_t$ by 
running $t-1$ steps of the non-private filtering algorithm of  \S\ref{sec:standard} but with a given {\em fixed} set of parameters $\{(\mu_\ell,v_\ell)\}_{\ell\in[t-1]}$ (and the given  randomness $\{Z_\ell\}_{\ell\in[t-1]}$), 
and $(iii)$ computes the queried private statistics of $S_t$.  
If the private spectral norm of the covariance  of $S_t$ (i.e., $\lambda_t$) is sufficiently small, we output the private and robust mean $\hat\mu=\mu_t$
(line~\ref{line8:DPfilter_interactive_main}). 
Otherwise, we compute the private top PCA direction $v_t$ and draw an  randomness $Z_t$ to be used in the next step of filtering, as in the non-private filtering algorithm. 
We emphasize that $\{S_\ell\}$ are not private, and hence never returned to the analyst. 
We also note that this interactive version is redundant as every query is re-computing $S_t$. 
In our setting, the analyst has the dataset and there is no need to separate them. This leads to a {\em centralized} version we provide in Algorithm~\ref{alg:privatefiltering} in the appendix, which avoids redundant computations and hence is significantly more efficient. 

The main  challenge in this framework is the privacy analysis. Because $\{S_\ell\}_{\ell\in[t-1]}$ is not private, each query runs $t-1$ steps of filtering whose end-to-end sensitivity could blow-up. Algorithmically, 
$(i)$ we start with a specific choice of a non-private iterative filtering algorithm (among several variations that are equivalent in non-private setting but widely differ in its sensitivity), and $(ii)$ make appropriate changes in the private queries (Algorithm~\ref{alg:DPfilter_interactive_main}) to keep the sensitivity small. 
Analytically, the following key technical lemma allows a sharp analysis of  the end-to-end  sensitivity of iterative filtering. 
 \begin{lemma}
     \label{lem:DPfilter_sensitivity}
     Let $S_t({\cal S})$ denote the resulting  subset of samples after $t$ iterations of the filtering in the queries ($q_{\rm size}$, $q_{\rm mean}$, $q_{\rm norm}$, and $q_{\rm PCA}$) are applied to a dataset ${\cal S}$ using {\em fixed} parameters $\{(\mu_\ell,v_\ell,Z_\ell)\}_{\ell=1}^{t}$. 
     Then, we have 
     $d_\triangle( S_t({\cal S}), S_t({\cal S}') ) \leq d_\triangle({\cal S},{\cal S'})$,
     where $d_\triangle({\cal S},{\cal S}') \triangleq \max\{|{\cal S}\setminus {\cal S'}|,|{\cal S}' \setminus {\cal S}|\} $. 
 \end{lemma}
 Recall that two datasets are neighboring, i.e.,  ${\cal S}\sim {\cal S}'$, iff $d_\triangle({\cal S},{\cal S}')\leq 1$. 
 This lemma implies that if two datasets are neighboring, then  
 they are still neighboring after filtering with the same parameters, no matter how many  times we filter them. 
 Hence, this lemma allows us to use the standard output-perturbation mechanisms with $(\varepsilon_1,\delta_1)$-DP. Advanced composition ensures that end-to-end guarantee of $4T$ such queries is $(0.99\varepsilon,0.99\delta)$-DP. Together with $(0.01\varepsilon,0.01\delta)$-DP budget used in $q_{\rm range}$, this satisfied the target privacy. 
Analyzing the utility of this algorithm, we get the following guarantee. 

\begin{thm} 
    \label{thm:main1}
    Algorithm~\ref{alg:DPfilter_interactive_main} is $(\varepsilon,\delta)$-DP. 
    Under  Assumption~\ref{asmp:adversary}, 
    there exists a universal constant $c\in(0,0.1)$ such that if $\alpha\leq c$ and   $n=\widetilde{\Omega}\left( (d/\alpha^2)+{d^{2}(\log(1/\delta))^{3/2}}/({\varepsilon\alpha })
    \right) $ 
    then    Algorithm~\ref{alg:DPfilter_interactive_main} achieves $\| \hat\mu-\mu\|_2 \leq O(\alpha\sqrt{\log(1/\alpha)}) $ with probability $0.9$. 
\end{thm}
 The first term $O(d/\alpha^2)$ in the sample complexity is optimal  (cf.~Table~\ref{tbl:Gauss}), but there is a factor of $d$ gap in the second term. 
 This is due to the fact that we need to run $O(d)$ iterations in the worst-case. 
 Such numerous accesses to the database result in large noise to be added at each iteration, requiring large sample size to combat that extra noise.  We introduce PRIME  to reduce the number of iterations to $O((\log d)^2)$ and  significantly reduce the  sample complexity.

\subsection{PRIME: novel robust and private mean estimator} 
\label{sec:prime}

Algorithm~\ref{alg:DPfilter_interactive_main}  (specifically Filter($\cdot$) in Algorithm~\ref{alg:DPfilter_interactive_main}) accesses the database $O(d)$ times. This is necessary for two reasons. 
First, the filter checks only one direction $v_t$ at each iteration. 
In the worst case, the corrupted samples can be scattered in $\Omega(d)$ orthogonal directions such that the filter needs to be repeated $O(d)$ times. 
Secondly, even if the corrupted samples are clustered together in one direction, the filter still needs to be repeated $O(d)$ times. 
This is because   
we had to use a large  (random) threshold of  $dB^2Z_t=O(d)$ to make the threshold data-independent  so that we can keep the sensitivity of  Filter($\cdot$) low, which results in slow progress. 
We propose filtering multiple directions simultaneously using a new score $\{\tau_i\}$ based on the matrix multiplicative weights. 
Central to this approach is a novel adaptive filtering algorithm {\sc DPthreshold}($\cdot$) that  guarantees  sufficient decrease in the total score at every iteration.

\subsubsection{Matrix Multiplicative Weight (MMW) scoring} 
\label{sec:MMWfilter}

The  MMW-based approach, pioneered in \cite{dong2019quantum} for non-private robust mean estimation, 
filters out multiple directions simultaneously. 
It runs over $O(\log d)$ epochs and every epoch consists of $O(\log d)$ iterations. At every epoch $s$ and iteration $t$, 
step 2 of the iterative filtering in \S\ref{sec:standard} is replaced by 
a new score $\tau_i =  (x_i-{\rm Mean}(S_{t}^{(s)}))^T U_t^{(s)}  (x_i-{\rm Mean}(S_{t}^{(s)}))$ where $U_t^{(s)}$ now accounts for all directions in ${\mathbb R}^d$ but appropriately weighted. 
Precisely, it is defined via the matrix multiplicative update: 
\begin{eqnarray*}
    U_t^{(s)} \;=\; \frac{\exp\Big( \alpha^{(s)} \sum_{r\in[t]} ({\rm Cov}(S_r^{(s)}) -{\mathbf I})\Big)}{{\rm Tr}\big(\,\exp( \alpha^{(s)} \sum_{r\in[t]} ({\rm Cov}(S_r^{(s)}) -{\mathbf I})) \,\big)}  \;, 
\end{eqnarray*}
for some choice of $\alpha^{(s)}>0$.
If we set the number of iterations to one, a choice of $\alpha^{(s)}=\infty$ recovers the previous score that relied on the top singular vector from \S\ref{sec:standard} and a choice of $\alpha^{(s)}=0$ gives a simple norm based score $\tau_i=\|x_i\|^2_2$. 
An appropriate choice of $\alpha^{(s)}$ smoothly interpolates between these two extremes, which  ensures 
that    $O(\log d)$ iterations are sufficient for the spectral norm of the covariance to decrease strictly by a constant factor. 
This guarantees that after $O(\log d)$ epochs, we sufficiently decrease the covariance to ensure that the empirical mean is accurate enough. 
Critical in achieving this gain is  our carefully designed filtering algorithm  {\sc DPthreshold} that uses the privately computed MMW-based scores using Gaussian mechanism on the covariance matrices as shown in  Algorithm~\ref{alg:1Dfilter} in Appendix \ref{sec:prime_appendix}.

\subsubsection{Adaptive filtering with {\sc DPthreshold}}
\label{sec:DPthreshold} 
{\bf Novelty.} 
The corresponding non-private  filtering of \cite[Algorithm 9]{dong2019quantum} for robust mean estimation takes advantage of  an {\em adaptive threshold}, but filters out each sample independently resulting in a prohibitively large sensitivity;  
the coupling between each sample and the randomness used to filter it can change widely between two neighboring datasets. 
On the other hand, 
Algorithm~\ref{alg:DPfilter_interactive_main}  (i.e., Filter($\cdot$) in Algorithm~\ref{alg:DPfilter_interactive})
takes advantage of jointly filtering  all points above a {\em single threshold} $B^2d Z_t$ with a single randomness $Z_t\sim {\rm  Unif}[0,1]$, but the non-adaptive (and hence large) choice of the range $B^2d$ results in a large number of iterations because each  filtering only decrease the score by little.
To sufficiently reduce the total score  while maintaining  a small sensitivity, we introduce  a filter with a single and  adaptive threshold.   

\noindent
{\bf Algorithm.} 
Our goal here is to privately find a single scalar $\rho$ such that when a randomized filter is applied on the scores $\{\tau_i\}$ with a (random) threshold $\rho Z$ (with $Z$ drawn uniform in $[0,1]$), we filter out enough samples to make progress in each iteration while ensuring that we do not remove too many uncorrupted samples. 
This is a slight generalization of the non-private algorithm in Section~\ref{sec:standard}, which simply set $\rho = \max_{j\in S_{t}}\tau_j$. While this guarantees the filter removes more corrupted samples than good samples, it does not make sufficient progress in reducing the total score of the samples. 

Ideally, we want the thresholding to decrease the total score by a constant multiplicative factor, which will in the end allow the algorithm to terminate within logarithmic iterations. 
To this end, we propose a new scheme of using  the largest $\rho$ such that the following inequality holds: 
\begin{eqnarray}
    \sum_{\tau_i>\rho }(\tau_i - \rho)  \geq  0.31  \sum_{\tau_i \in S_t }(\tau_i - 1)\;.\label{eq:thresholdrule}
\end{eqnarray}
We use a private histogram of the scores to approximate this threshold. 
Similar to \cite{kaplan2020privately,KV17}, we use geometrically increasing bin sizes such that we  use only $O(\log B^2d)$ bins while 
achieving a preferred {\em multiplicative} error in our quantization. 
At each epoch $s$ and iteration $t$, 
we run {\sc DPthreshold} sketched in the following to approximate $\rho$ followed by a random filter. 
Step 3 replaces the non-private condition in Eq.~\eqref{eq:thresholdrule}.  
A complete description is provided in  
Algorithm~\ref{alg:1Dfilter}.
    \vspace{-0.1cm}

\begin{enumerate}
    \itemsep0em 
    \item Privately compute scores for all data points $i\in S_{t}^{(s)}:\, \tau_i \gets (x_i-\mu_{t})^\top U_{t}^{(s)} (x_i-\mu_{t}) $\,;
    \vspace{-0.1cm}
    \item Compute a private histogram $\{\tilde{h}_j\}_{j=1}^{2+\log(B^2d)}$ of the scores over geometrically sized bins $I_1=[1/4,1/2)$, $I_2=[1/2,1),\ldots,I_{2+\log(B^2d)}=[2^{\log(B^2d)-1},2^{\log(B^2d)}]$\,;
        \vspace{-0.1cm}
    \item Privately find the largest $\ell$  satisfying $ \sum_{j \geq \ell} (2^j-2^\ell)\, \tilde h_j \geq 0.31 \sum_{i\in S_{t}^{(s)}} (\tau_i-1) $\;;
        \vspace{-0.1cm}
    \item Output $\rho=2^\ell$ \;.
\end{enumerate}
    \vspace{-0.2cm}


\section{Analyses of PRIME} \label{sec:prime_analysis}

Building on the framework of Algorithm~\ref{alg:DPfilter_interactive_main}, PRIME (Algorithm~\ref{alg:prime}) replaces the  score with the MMW-based score presented in  \S\ref{sec:MMWfilter} and the filter with the adaptive {\sc DPthreshold}. This reduces the number of iterations to $T=O((\log d)^2)$ achieving the following bound.  

\begin{thm}
    \label{thm:main2}
    PRIME is  $(\varepsilon,\delta)$-differentially private. Under  Assumption~\ref{asmp:adversary} there exists a universal constant $c\in(0,0.1)$ such that if $\alpha\leq c$ and   $n=\widetilde\Omega ( (d/\alpha^2) +  (d^{3/2}/(\varepsilon\alpha) )\log(1/\delta) )$, 
    then PRIME achieves $\|\hat\mu-\mu\|_2 = O(\alpha\sqrt{\log (1/\alpha)})$ with probability $0.9$.  
\end{thm}


A proof is provided in Appendix~\ref{sec:proof_main2}.
The notation $\widetilde\Omega(\cdot)$ hides logarithmic terms in $d$, $R$, and $1/\alpha$. 
To achieve  an error of  $O(\alpha\sqrt{\log(1/\alpha)})$, 
the first term $\widetilde{\Omega}(d/\alpha^2\log(1/\alpha))$ is necessary even if there is  no corruption.  
The accuracy of $O(\alpha\sqrt{\log(1/\alpha)})$ matches the lower bound shown in~\cite{diakonikolas2017statistical} for any polynomial time statistical query algorithm, and it nearly matches the information theoretical lower bound on robust estimation of $\Omega(\alpha)$. 
On the other hand, the second term of $\widetilde{\Omega}(d^{3/2}/(\varepsilon \alpha \log(1/\alpha)))$   has an extra factor of $d^{1/2}$ compared to the optimal one achieved by exponential time 
Algorithm~\ref{alg:exp}. 
It is an open question if this gap can be closed by a polynomial time algorithm. 

The bottleneck is the private matrix multiplicative weights. 
Such spectral analyses are crucial in filter-based robust estimators. 
Even for a special case of privately computing the top principal component, the best polynomial time algorithm requires $O(d^{3/2})$ samples \cite{dwork2014analyze,PPCA,wei2016analysis}, and 
this sample complexity is also necessary as shown in \cite[Corollary 25]{dwork2014analyze}. 


 To boost the success probability to $1-\zeta$ for some small $\zeta>0$, we need an extra $\log(1/\zeta)$ factor in the sample complexity to make sure the dataset satisfies the regularity condition with probability $\zeta/2$. Then we can run PRIME $\log(1/\zeta)$ times and choose the output of a run that satisfies $n^{(s)}>n(1-10\alpha)$ and $\lambda^{(s)}\leq C \alpha \log(1/\alpha)$ at termination.

\vspace{-0.3cm}
\begin{figure}[h]
    \centering 
    \begin{tabular}{ccc}
	\includegraphics[width=0.32\linewidth]{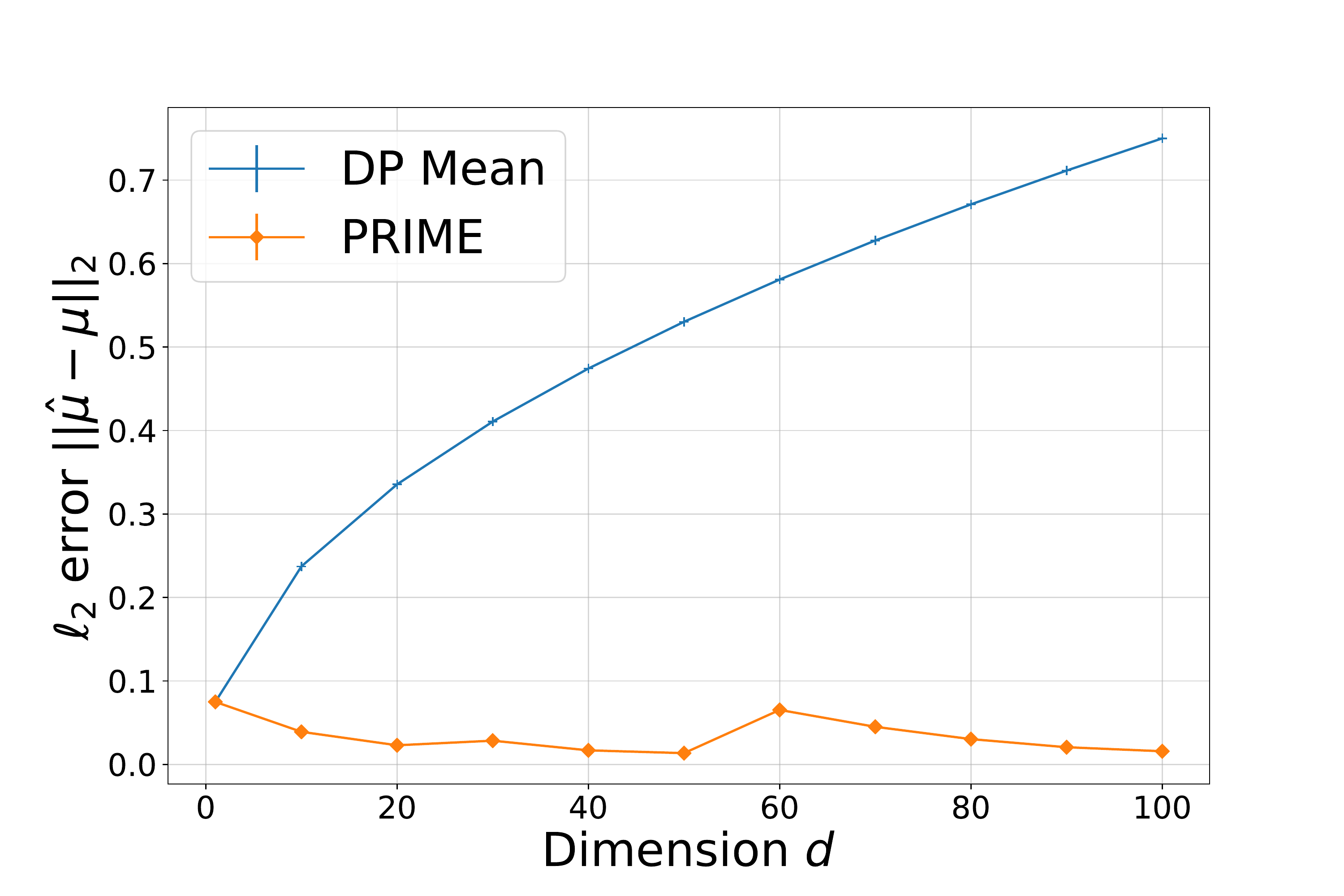}
	&
	\includegraphics[width=.32\linewidth]{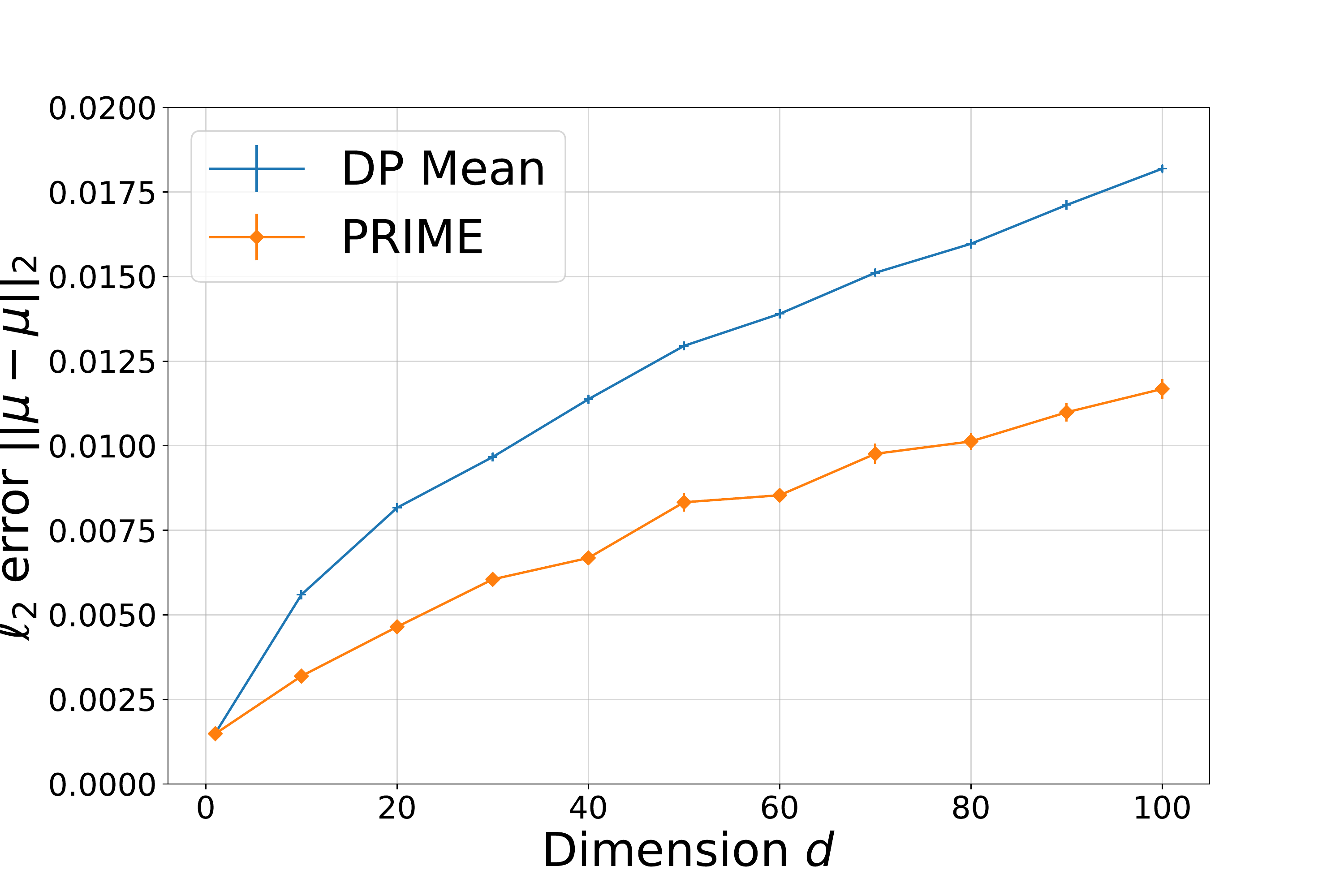}
	&
	\includegraphics[width=.32\linewidth]{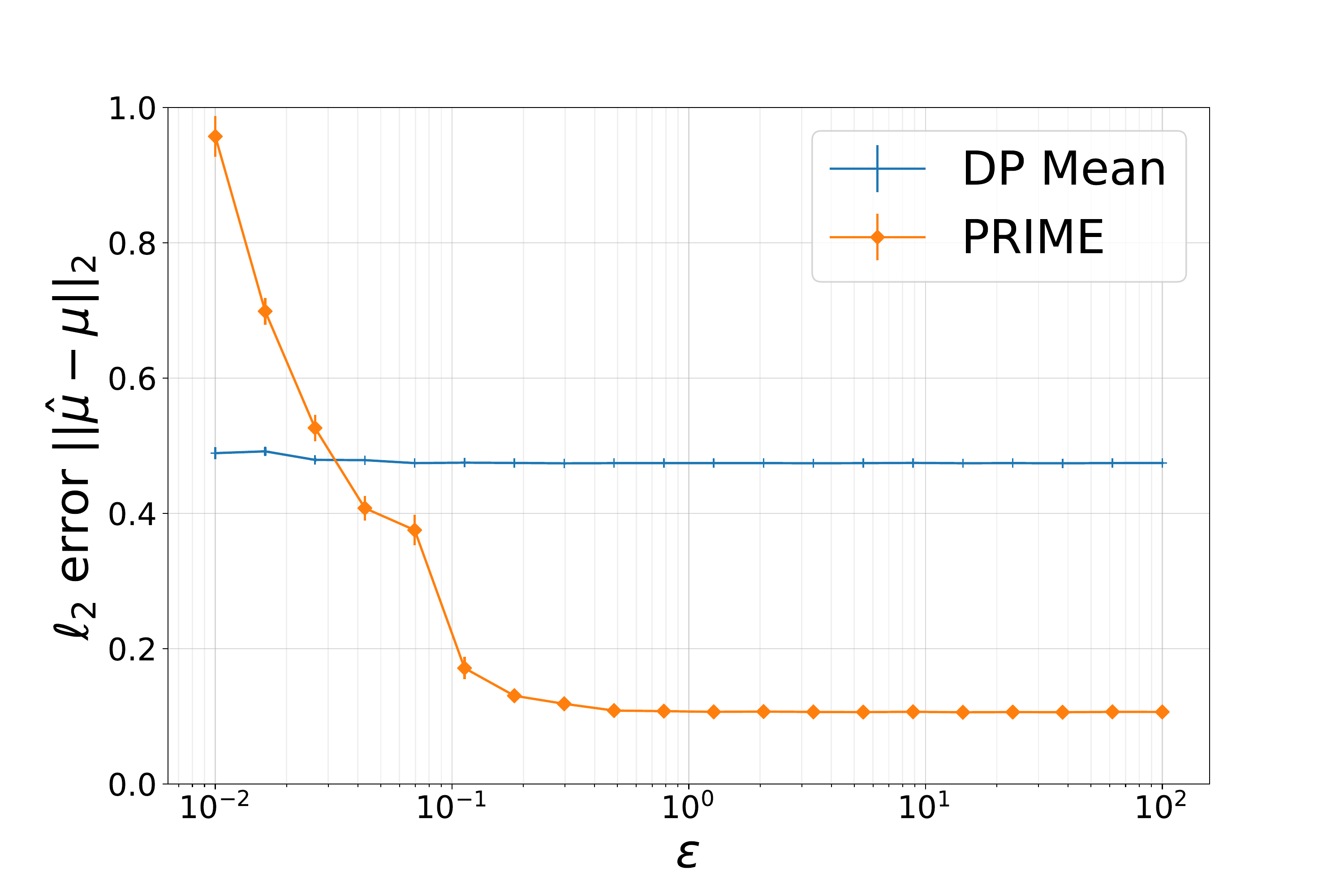}
	\end{tabular}
	\vspace{-0.1cm}
    \caption{Private mean estimators (e.g., DP mean \cite{KLSU19}) are vulnerable to adversarial corruption especially in high dimensions, while the proposed PRIME achieves   robustness (and privacy) regardless of the dimension of the samples. 
    }
    \label{fig:intro}
\end{figure}
\vspace{-0.2cm}

Numerical experiments support  our theoretical claims. 
The left figure with  $(\alpha,\varepsilon,\delta,n)=(0.05,20,0.01,10^6)$ is in the large $\alpha$ regime where the 
DP Mean error is dominates by $\alpha \sqrt{d}$ and PRIME error by $\alpha\sqrt{\log (1/\alpha)}$. Hence, PRIME error is constant whereas DP Mean error increases with the dimension $d$. 
The second figure with 
$(\alpha,\varepsilon,\delta,n)=(0.001,20,0.01,10^6)$ is in the small $\alpha$ regime when DP Mean error consists of $ \alpha\sqrt{d}+\sqrt{d/n}$ and PRIME is dominated by $\sqrt{d/n}$. Both increase with the dimension $d$, and the gap can be made large by increasing $\alpha$. The right figure with $(\alpha,\delta,d,n)=(0.1,0.01,10,10^6)$ is when DP Mean error is dominated by $\alpha\sqrt{d}$ and PRIME by $\alpha\sqrt{\log(1/\alpha)}$ when $\varepsilon > c d^{1.5}/ (\alpha n)$. 
Below this threshold, which happens in this example around $\varepsilon=0.05$, the added noise in the private mechanism starts to dominate with decreasing $\varepsilon$. 
Both algorithms have respective thresholds below which the error increases  with decreasing  $\varepsilon$.  
This threshold is larger for PRIME because it uses the privacy budget to perform multiple operations and hence the noise added to the final output is larger compared to DP Mean. Below this threshold, which can be easily determined based on the known parameters $(\varepsilon, \delta,n, \alpha)$, we should either collect more data (which will decrease the threshold) or give up filtering  and spend all privacy budget on $q_{\rm range}$ and the empirical mean  (which will reduce the  error). 
Details of the experiments are  in Appendix~\ref{sec:experiments}.

\section{Exponential time algorithm with near-optimal  sample  complexity}
\label{sec:subgauss}
\noindent {\bf Novelty.} An existing exponential time algorithm  for robust and private mean estimation in  \cite{bun2019private}
strictly requires the uncorrupted samples to be drawn from a Gaussian distribution. We also provide a similar algorithm based on private Tukey median in Appendix~\ref{sec:tukey} and its analysis in Appendix~\ref{sec:proof_tukey}.
In this section, we introduce a novel estimator that achieves near-optimal guarantees for more general sub-Gaussian  distributions (and also covariance  bounded distributions) but takes  an {\em exponential} run-time. Its innovation is in leveraging on the {\em resilience} property of well-behaved distributions not only to estimate the mean robustly (which is the standard use of the property) but also to adaptively bound the sensitivity of the estimator, thus achieving optimal privacy-accuracy tradeoff. 


\begin{definition}[Resilience from Definition 1 in~\cite{steinhardt2018resilience}] 
A set of points $\{x_i\}_{i\in S}$ lying in $\reals^d$ is $(\sigma,\alpha)$-resilient around a point $\mu$ if   $\|(1/|T|)\sum_{i\in T}(x_i-\mu) \|_2 \le \sigma$ for all subsets $T\subset S$ of size  $(1-\alpha)|S|$.
\end{definition}
{\bf Algorithm.} 
As data is  corrupted, we define $R(S)$ as a surrogate for resilience of the uncorrupted part of the set.
If $S$ indeed consists of a $1-\alpha$ fraction of independent samples from the promised class of distributions, the goodness score $R(S)$ will be close to the resilience property of the good data. 

\begin{definition}[Goodness of a set]
{For $\mu(S)=(1/|S|)\sum_{i\in S} x_i$,} let us define
\begin{eqnarray*}
R(S) &\triangleq& \min_{S'\subset S, |S'|=(1-2\alpha)|S|.}\,\,\max_{T\subset S', |T| = (1-\alpha)|S'|.}\,\|\mu(T)-\mu(S')\|_2\;.\label{eq:goodness}
\end{eqnarray*} 
\label{def:goodness}
\end{definition} 
Algorithm~\ref{alg:exp} first checks if the resilience   matches that of the promised distribution. 
The data is pre-processed with $q_{\rm range}$ to ensure we can check $R(S)$ privately. Once resilience is cleared, we can safely use the exponential mechanism based on the  score function $d(\hat{\mu},S)$ in Definition~\ref{def:dist_exp} 
to select an approximate robust mean $\hat\mu$ privately. The choice of the sensitivity critically relies on the fact that resilient datasets have small sensitivity of $O((1/n)\sqrt{\log(1/\alpha)})$. Without the resilience check, the sensitivity is  $O( d^{1/2}/n)$ resulting in an extra factor of $\sqrt{d}$   in the sample complexity.

\begin{algorithm2e}[ht]
   \caption{Exponential-time private and robust mean estimation}
   \label{alg:exp}
   	\DontPrintSemicolon 
	\KwIn{$S=\{x_i\}_{i\in[n]}$, $\alpha \in(0,1/2)$, $(\varepsilon,\delta)$ }
	\SetKwProg{Fn}{}{:}{}
	{
	{\bf if }{$n<cd^{1/2}\log(1/\delta) /\ (\varepsilon\alpha\sqrt{\log(1/\alpha)})$ }{\bf then }{{\rm {\bf Output:} $\emptyset$}} \hfill[
	    $cd^{1/2}\log(1/\delta) /\ (\varepsilon\alpha)$ for hevay-tail]\\
	    $(\bar{x},B) \gets  q_{\rm range}( S,  (1/3)\varepsilon,(1/3)\delta)$ \hfill[
	    $q_{\rm range-ht}(\cdot)$ for hevay-tail]\\
	    Project the data points onto the ball: $ {x}_i\gets {\cal P}_{{\cal B}_{\sqrt{d}B/2}(
	   \bar{x})}(x_i)$, for all $i\in[n]$ \;
	    $\widehat{R}(S) \gets R(S) + {\rm Lap}(3Bd^{1/2}/(n\varepsilon))$\;
	   {\bf if }{$\widehat{R}(S)> 2\, \alpha \sqrt{\log (1/\alpha)}$ }{\bf then }{{\rm {\bf Output:} $\emptyset$}} 
	   \hfill[$\widehat{R}(S) > 2c_\zeta \sqrt\alpha$ for hevay-tail]
	   \\
	   {\bf else Output:} a randomly drawn point $\hat{\mu}\in{\cal B}_{\sqrt{d}B/2}(
	   \bar{x})$  sampled from a density \;
	   \hspace{1cm} $r(\hat\mu) \propto e^{-(1/(24\sqrt{\log(1/\alpha)}))\varepsilon \,n\, d(\hat\mu,S) }$
	   \hfill[$e^{-(\varepsilon n \sqrt\alpha /(24 c_\zeta))d(\hat\mu,S)}$ for heavy-tail]
    } 
\end{algorithm2e}

We propose the score function $d(\hat\mu, S)$ in the following definition, which is a robust estimator of the distance between the mean and the candidate $\hat\mu$.


\begin{definition}
\label{def:dist_exp}
For a set of data $\{x_i\}_{i\in S}$ lying in $\reals^d$, for any $v\in \mathbb{S}^{d-1}$, define ${\cal{T}}^{v}$ to be the $3\alpha |S|$ points with the largest $v^\top x_i$ value, ${\cal{B}}^{v}$ to be the $3\alpha |S|$ points with the smallest $v^\top x_i$ value, and ${\cal{M}}^{v} = S\setminus({\cal{T}}^{v}\cup {\cal{B}}^{v})$. Define
$d(\hat\mu, S) \;\;\triangleq\;\; \max_{v\in \mathbb{S}^{d-1}}\left|v^\top\left(\mu({\cal M}^{v})-\hat\mu\right)\right|\;.$
\end{definition}

{\bf Analysis.} 
For any direction $v$, the truncated mean estimator $\mu({\cal M}^{v})$ provides a robust estimation of the true mean along the direction $v$, thus the distance can be simply defined by taking the maximum over all directions $v$. We show the sensitivity of this simple estimator is bounded by the resilience property $\sigma$ divided by $n$, which is $O((1/n)\sqrt{\log(1/\alpha)})$ once the resilience check is passed. This leads to the following near-optimal sample complexity. 
 We provide a proof in Appendix~\ref{sec:proof_subgaussian}.

\begin{thm}
    [Exponential time algorithm for sub-Gaussian distributions]
\label{thm:subgauss}
     Algorithm~\ref{alg:exp} is $(\varepsilon,\delta)$-DP. 
    Under Assumption~\ref{asmp:adversary}, this algorithm achieves 
    $\|\hat\mu-\mu\|_2=O(\alpha \sqrt{\log(1/\alpha)})$ with probability $1-\zeta$ if 
    $$ n=\widetilde\Omega\Big(\,\frac{d+\log\frac{1}{\zeta}}{\alpha^2\log\frac1\alpha}+\frac{ d\log\left(d \sqrt{\log(dn/\zeta)}/\alpha\right)+d^{1/2}\log\frac1\delta+\log\frac1\zeta }{\varepsilon \alpha}+\frac{\sqrt{d\log\frac1\delta}\log\frac{d}{\zeta\delta} }{\varepsilon} \,\Big)\;.$$ 
\end{thm}
\noindent{\bf Run-time.}
 Computing $R(S)$ exactly can take $O(d e^{\Theta(n)})$ operations.  The exponential mechanism implemented with $\alpha$-covering for $\hat\mu$ and a constant covering for $v$ can take $O(nd(\sqrt{\log(dn/\zeta)}/\alpha)^{d})$ operations.

\section{Conclusion}\label{sec:discussion}
Differentially private mean estimation is brittle against a small fraction of the samples being corrupted by an adversary. We show that robustness can be achieved without any increase in the sample complexity by introducing a novel DP mean estimator, which  requires run-time exponential in the dimension of the samples. 
We emphasize that the unknown true mean $\mu$ can be any vector in $\reals^d$, and we do not require a known bound on the norm, $\|\mu\|$, that some previous work requires (e.g., \cite{KV17}).   
The technical contribution is in leveraging the resilience property of well-behaved distributions in an innovative way to not only find robust mean (which is the typical use case of resilience) but also  bound sensitivity for optimal privacy guarantee. 
 To cope with the computational challenge, 
 we  propose an efficient algorithm, which we call PRIME, that achieves the optimal target accuracy at the cost of an increased sample complexity. 
  Again,  the unknown true mean $\mu$ can be any vector in $\reals^d$, and PRIME does not require a known bound on the norm, $\|\mu\|$, that some previous work requires (e.g., \cite{KV17}). 
 The technical contributions are  $(i)$ a novel framework for private iterative filtering and its tight analysis of the end-to-end sensitivity and $(ii)$ novel filtering algorithm of {\sc DPthreshold} which is critical in  privately running  matrix multiplicative weights and hence  significantly reducing the number of accesses to the database. 
 With appropriately chosen parameters, 
we show that  our exponential time approach achieves near-optimal guarantees for both sub-Gaussian and covariance bounded distributions and PRIME achieves the same accuracy efficiently but at the cost of an increased sample complexity by a $d^{1/2}$ factor.

There are several directions for improving our results further and applying the framework to solve other problems. 
PRIME provides a new design principle for private and robust estimation. This can be more broadly applied to fundamental statistical analyses such as 
robust covariance estimation  \cite{diakonikolas2019robust,diakonikolas2017beingrobust,li2020robust}
 robust PCA \cite{kong2020robust,jambulapati2020robust}, and robust linear regression \cite{klivans2018efficient,diakonikolas2019efficient}.

PRIME could be improved in a few directions. 
First, the sample complexity of   
$\widetilde\Omega ( (d/(\alpha^2\log(1/\alpha))) +  (d^{3/2}/(\varepsilon\alpha\log(1/\alpha)) )\log(1/\delta) )$
in Theorem~\ref{thm:main2}
is suboptimal in the second term. 
Improving the $d^{3/2}$ factor requires bypassing differentially private singular value decomposition, which seems to be a challenging task. 
However, it might be possible to separate the $\log(1/\delta)$ factor from the rest of the terms and get an additive error of the form 
$\widetilde\Omega ( (d/(\alpha^2\log(1/\alpha))) +  (d^{3/2}/(\varepsilon\alpha\log(1/\alpha)) )+(1/\varepsilon)\log(1/\delta) )$. This requires using Laplace mechanism in private MMW (line~\ref{line:mmw} Algortihm~\ref{alg:DPMMWfilter}). 
Secondly, the time complexity of PRIME is dominated by computation time of the matrix exponential in (line~\ref{line:mmw} Algortihm~\ref{alg:DPMMWfilter}). 
Total number of operations scale as $\widetilde{O}(d^3+nd^2)$. 
One might hope to achieve 
$\widetilde{O}(nd)$ time complexity using approximate computations of $\tau_j$'s using techniques from \cite{dong2019quantum}. This does not improve the sample complexity, as the number of times the dataset is accessed remains the same.  
Finally, 
for (non-robust) private mean estimation, {\sc CoinPress} provides a practical improvement in the small sample regime by progressively refining the search space  \cite{biswas2020coinpress}. 
The same principle could be applied to PRIME to design a robust version of {\sc CoinPress}. 
One important question remains open; how are differential privacy and robust statistics fundamentally related? We believe our exponential time algorithm hints on a fundamental connection between robust statistics of a data projected onto one-dimensional subspace and sensitivity of resulting score function for the exponential mechanism. It is an interesting direction to pursue this connection further to design novel algorithms that bridge privacy and robustness.

\section*{Acknowledgement}
Sham Kakade acknowledges funding from the National Science Foundation
under award CCF-1703574. 
Sewoong Oh acknowledges funding from Google faculty research award, NSF grants IIS-1929955, CCF-1705007, CNS-2002664, CCF 2019844 as a part of Institute for Foundation of Machine Learning, and CNS-2112471 as a part of Institute  for Future Edge Networks and Distributed Intelligence. 

\bibliographystyle{plain}
\bibliography{references}

\begin{thebibliography}{10}

\bibitem{abadi2016deep}
Martin Abadi, Andy Chu, Ian Goodfellow, H~Brendan McMahan, Ilya Mironov, Kunal
  Talwar, and Li~Zhang.
\newblock Deep learning with differential privacy.
\newblock In {\em Proceedings of the 2016 ACM SIGSAC Conference on Computer and
  Communications Security}, pages 308--318, 2016.

\bibitem{census}
John~M Abowd.
\newblock The us census bureau adopts differential privacy.
\newblock In {\em Proceedings of the 24th ACM SIGKDD International Conference
  on Knowledge Discovery \& Data Mining}, pages 2867--2867, 2018.

\bibitem{aden2020sample}
Ishaq Aden-Ali, Hassan Ashtiani, and Gautam Kamath.
\newblock On the sample complexity of privately learning unbounded
  high-dimensional gaussians.
\newblock {\em arXiv preprint arXiv:2010.09929}, 2020.

\bibitem{allen2015spectral}
Zeyuan Allen-Zhu, Zhenyu Liao, and Lorenzo Orecchia.
\newblock Spectral sparsification and regret minimization beyond matrix
  multiplicative updates.
\newblock In {\em Proceedings of the forty-seventh annual ACM symposium on
  Theory of computing}, pages 237--245, 2015.

\bibitem{amaldi1995complexity}
Edoardo Amaldi and Viggo Kann.
\newblock The complexity and approximability of finding maximum feasible
  subsystems of linear relations.
\newblock {\em Theoretical computer science}, 147(1-2):181--210, 1995.

\bibitem{anscombe1960rejection}
Frank~J Anscombe.
\newblock Rejection of outliers.
\newblock {\em Technometrics}, 2(2):123--146, 1960.

\bibitem{bakshi2020list}
Ainesh Bakshi and Pravesh Kothari.
\newblock List-decodable subspace recovery via sum-of-squares.
\newblock {\em arXiv preprint arXiv:2002.05139}, 2020.

\bibitem{BDLS17}
S.~Balakrishnan, S.~S. Du, J.~Li, and A.~Singh.
\newblock Computationally efficient robust sparse estimation in high
  dimensions.
\newblock In {\em Proceedings of the 30th Conference on Learning Theory, {COLT}
  2017}, pages 169--212, 2017.

\bibitem{beimel2019private}
Amos Beimel, Shay Moran, Kobbi Nissim, and Uri Stemmer.
\newblock Private center points and learning of halfspaces.
\newblock {\em arXiv preprint arXiv:1902.10731}, 2019.

\bibitem{BhatiaJKK17}
K.~Bhatia, P.~Jain, P.~Kamalaruban, and P.~Kar.
\newblock Consistent robust regression.
\newblock In {\em Advances in Neural Information Processing Systems 30: Annual
  Conference on Neural Information Processing Systems 2017}, pages 2107--2116,
  2017.

\bibitem{bhatia2015robust}
Kush Bhatia, Prateek Jain, and Purushottam Kar.
\newblock Robust regression via hard thresholding.
\newblock In {\em Advances in Neural Information Processing Systems}, pages
  721--729, 2015.

\bibitem{biswas2020coinpress}
Sourav Biswas, Yihe Dong, Gautam Kamath, and Jonathan Ullman.
\newblock Coinpress: Practical private mean and covariance estimation.
\newblock {\em arXiv preprint arXiv:2006.06618}, 2020.

\bibitem{blum2005practical}
Avrim Blum, Cynthia Dwork, Frank McSherry, and Kobbi Nissim.
\newblock Practical privacy: the sulq framework.
\newblock In {\em Proceedings of the twenty-fourth ACM SIGMOD-SIGACT-SIGART
  symposium on Principles of database systems}, pages 128--138, 2005.

\bibitem{bun2019private}
Mark Bun, Gautam Kamath, Thomas Steinke, and Steven~Z Wu.
\newblock Private hypothesis selection.
\newblock In {\em Advances in Neural Information Processing Systems}, pages
  156--167, 2019.

\bibitem{cai2019cost}
T~Tony Cai, Yichen Wang, and Linjun Zhang.
\newblock The cost of privacy: Optimal rates of convergence for parameter
  estimation with differential privacy.
\newblock {\em arXiv preprint arXiv:1902.04495}, 2019.

\bibitem{canonne2019private}
Cl{\'e}ment~L Canonne, Gautam Kamath, Audra McMillan, Jonathan Ullman, and
  Lydia Zakynthinou.
\newblock Private identity testing for high-dimensional distributions.
\newblock {\em arXiv preprint arXiv:1905.11947}, 2019.

\bibitem{charikar2017learning}
Moses Charikar, Jacob Steinhardt, and Gregory Valiant.
\newblock Learning from untrusted data.
\newblock In {\em Proceedings of the 49th Annual ACM SIGACT Symposium on Theory
  of Computing}, pages 47--60, 2017.

\bibitem{PPCA}
Kamalika Chaudhuri, Anand~D Sarwate, and Kaushik Sinha.
\newblock A near-optimal algorithm for differentially-private principal
  components.
\newblock {\em The Journal of Machine Learning Research}, 14(1):2905--2943,
  2013.

\bibitem{chen2017targeted}
Xinyun Chen, Chang Liu, Bo~Li, Kimberly Lu, and Dawn Song.
\newblock Targeted backdoor attacks on deep learning systems using data
  poisoning.
\newblock {\em arXiv preprint arXiv:1712.05526}, 2017.

\bibitem{cheng2019high}
Yu~Cheng, Ilias Diakonikolas, and Rong Ge.
\newblock High-dimensional robust mean estimation in nearly-linear time.
\newblock In {\em Proceedings of the Thirtieth Annual ACM-SIAM Symposium on
  Discrete Algorithms}, pages 2755--2771. SIAM, 2019.

\bibitem{cheng2019faster}
Yu~Cheng, Ilias Diakonikolas, Rong Ge, and David~P Woodruff.
\newblock Faster algorithms for high-dimensional robust covariance estimation.
\newblock In {\em Conference on Learning Theory}, pages 727--757. PMLR, 2019.

\bibitem{cherapanamjeri2020list}
Yeshwanth Cherapanamjeri, Sidhanth Mohanty, and Morris Yau.
\newblock List decodable mean estimation in nearly linear time.
\newblock {\em arXiv preprint arXiv:2005.09796}, 2020.

\bibitem{dalalyan2019outlier}
Arnak Dalalyan and Philip Thompson.
\newblock Outlier-robust estimation of a sparse linear model using $
  \ell_1$-penalized huber's $ m $-estimator.
\newblock In {\em Advances in Neural Information Processing Systems}, pages
  13188--13198, 2019.

\bibitem{depersin2019robust}
Jules Depersin and Guillaume Lecu{\'e}.
\newblock Robust subgaussian estimation of a mean vector in nearly linear time.
\newblock {\em arXiv preprint arXiv:1906.03058}, 2019.

\bibitem{devroye2012combinatorial}
Luc Devroye and G{\'a}bor Lugosi.
\newblock {\em Combinatorial methods in density estimation}.
\newblock Springer Science \& Business Media, 2012.

\bibitem{dhar2020designing}
Aditya Dhar and Jason Huang.
\newblock Designing differentially private estimators in high dimensions.
\newblock {\em arXiv preprint arXiv:2006.01944}, 2020.

\bibitem{diakonikolas2020robustly}
Ilias Diakonikolas, Samuel~B Hopkins, Daniel Kane, and Sushrut Karmalkar.
\newblock Robustly learning any clusterable mixture of gaussians.
\newblock {\em arXiv preprint arXiv:2005.06417}, 2020.

\bibitem{diakonikolas2019robust}
Ilias Diakonikolas, Gautam Kamath, Daniel Kane, Jerry Li, Ankur Moitra, and
  Alistair Stewart.
\newblock Robust estimators in high-dimensions without the computational
  intractability.
\newblock {\em SIAM Journal on Computing}, 48(2):742--864, 2019.

\bibitem{diakonikolas2019sever}
Ilias Diakonikolas, Gautam Kamath, Daniel Kane, Jerry Li, Jacob Steinhardt, and
  Alistair Stewart.
\newblock Sever: A robust meta-algorithm for stochastic optimization.
\newblock In {\em International Conference on Machine Learning}, pages
  1596--1606, 2019.

\bibitem{diakonikolas2017beingrobust}
Ilias {Diakonikolas}, Gautam {Kamath}, Daniel~M. {Kane}, Jerry {Li}, Ankur
  {Moitra}, and Alistair {Stewart}.
\newblock {Being Robust (in High Dimensions) Can Be Practical}.
\newblock {\em arXiv e-prints}, page arXiv:1703.00893, March 2017.

\bibitem{diakonikolas2018robustly}
Ilias Diakonikolas, Gautam Kamath, Daniel~M Kane, Jerry Li, Ankur Moitra, and
  Alistair Stewart.
\newblock Robustly learning a gaussian: Getting optimal error, efficiently.
\newblock In {\em Proceedings of the Twenty-Ninth Annual ACM-SIAM Symposium on
  Discrete Algorithms}, pages 2683--2702. SIAM, 2018.

\bibitem{diakonikolas2019recent}
Ilias Diakonikolas and Daniel~M Kane.
\newblock Recent advances in algorithmic high-dimensional robust statistics.
\newblock {\em arXiv preprint arXiv:1911.05911}, 2019.

\bibitem{diakonikolas2017statistical}
Ilias Diakonikolas, Daniel~M Kane, and Alistair Stewart.
\newblock Statistical query lower bounds for robust estimation of
  high-dimensional gaussians and gaussian mixtures.
\newblock In {\em 2017 IEEE 58th Annual Symposium on Foundations of Computer
  Science (FOCS)}, pages 73--84. IEEE, 2017.

\bibitem{diakonikolas2018list}
Ilias Diakonikolas, Daniel~M Kane, and Alistair Stewart.
\newblock List-decodable robust mean estimation and learning mixtures of
  spherical gaussians.
\newblock In {\em Proceedings of the 50th Annual ACM SIGACT Symposium on Theory
  of Computing}, pages 1047--1060, 2018.

\bibitem{diakonikolas2019efficient}
Ilias Diakonikolas, Weihao Kong, and Alistair Stewart.
\newblock Efficient algorithms and lower bounds for robust linear regression.
\newblock In {\em Proceedings of the Thirtieth Annual ACM-SIAM Symposium on
  Discrete Algorithms}, pages 2745--2754. SIAM, 2019.

\bibitem{dong2019quantum}
Yihe Dong, Samuel Hopkins, and Jerry Li.
\newblock Quantum entropy scoring for fast robust mean estimation and improved
  outlier detection.
\newblock In {\em Advances in Neural Information Processing Systems}, pages
  6067--6077, 2019.

\bibitem{dwork2006calibrating}
Cynthia Dwork, Frank McSherry, Kobbi Nissim, and Adam Smith.
\newblock Calibrating noise to sensitivity in private data analysis.
\newblock In {\em Theory of cryptography conference}, pages 265--284. Springer,
  2006.

\bibitem{dwork2014algorithmic}
Cynthia Dwork and Aaron Roth.
\newblock The algorithmic foundations of differential privacy.
\newblock {\em Foundations and Trends in Theoretical Computer Science},
  9(3-4):211--407, 2014.

\bibitem{dwork2014analyze}
Cynthia Dwork, Kunal Talwar, Abhradeep Thakurta, and Li~Zhang.
\newblock Analyze gauss: optimal bounds for privacy-preserving principal
  component analysis.
\newblock In {\em Proceedings of the forty-sixth annual ACM symposium on Theory
  of computing}, pages 11--20, 2014.

\bibitem{google1}
{\'U}lfar Erlingsson, Vasyl Pihur, and Aleksandra Korolova.
\newblock Rappor: Randomized aggregatable privacy-preserving ordinal response.
\newblock In {\em Proceedings of the 2014 ACM SIGSAC conference on computer and
  communications security}, pages 1054--1067, 2014.

\bibitem{google2}
Giulia Fanti, Vasyl Pihur, and {\'U}lfar Erlingsson.
\newblock Building a rappor with the unknown: Privacy-preserving learning of
  associations and data dictionaries.
\newblock {\em Proceedings on Privacy Enhancing Technologies}, 2016(3):41--61,
  2016.

\bibitem{gao2020robust}
Chao Gao et~al.
\newblock Robust regression via mutivariate regression depth.
\newblock {\em Bernoulli}, 26(2):1139--1170, 2020.

\bibitem{hopkins2020robust}
Sam Hopkins, Jerry Li, and Fred Zhang.
\newblock Robust and heavy-tailed mean estimation made simple, via regret
  minimization.
\newblock {\em Advances in Neural Information Processing Systems}, 33, 2020.

\bibitem{hopkins2020mean}
Samuel~B Hopkins.
\newblock Mean estimation with sub-gaussian rates in polynomial time.
\newblock {\em Annals of Statistics}, 48(2):1193--1213, 2020.

\bibitem{hopkins2018mixture}
Samuel~B Hopkins and Jerry Li.
\newblock Mixture models, robustness, and sum of squares proofs.
\newblock In {\em Proceedings of the 50th Annual ACM SIGACT Symposium on Theory
  of Computing}, pages 1021--1034, 2018.

\bibitem{hopkins2019hard}
Samuel~B Hopkins and Jerry Li.
\newblock How hard is robust mean estimation?
\newblock In {\em Conference on Learning Theory}, pages 1649--1682. PMLR, 2019.

\bibitem{huber1964robust}
Peter~J. Huber.
\newblock {Robust Estimation of a Location Parameter}.
\newblock {\em The Annals of Mathematical Statistics}, 35(1):73 -- 101, 1964.

\bibitem{jambulapati2020robust}
Arun Jambulapati, Jerry Li, and Kevin Tian.
\newblock Robust sub-gaussian principal component analysis and
  width-independent schatten packing.
\newblock {\em Advances in Neural Information Processing Systems}, 33, 2020.

\bibitem{jia2019robustly}
He~Jia and Santosh Vempala.
\newblock Robustly clustering a mixture of gaussians.
\newblock {\em arXiv preprint arXiv:1911.11838}, 2019.

\bibitem{kairouz2019advances}
Peter Kairouz, H~Brendan McMahan, Brendan Avent, Aur{\'e}lien Bellet, Mehdi
  Bennis, Arjun~Nitin Bhagoji, Keith Bonawitz, Zachary Charles, Graham Cormode,
  Rachel Cummings, et~al.
\newblock Advances and open problems in federated learning.
\newblock {\em arXiv preprint arXiv:1912.04977}, 2019.

\bibitem{composition}
Peter Kairouz, Sewoong Oh, and Pramod Viswanath.
\newblock The composition theorem for differential privacy.
\newblock In {\em International conference on machine learning}, pages
  1376--1385, 2015.

\bibitem{KLSU19}
Gautam Kamath, Jerry Li, Vikrant Singhal, and Jonathan Ullman.
\newblock Privately learning high-dimensional distributions.
\newblock In {\em Conference on Learning Theory}, pages 1853--1902, 2019.

\bibitem{kamath2020differentially}
Gautam Kamath, Or~Sheffet, Vikrant Singhal, and Jonathan Ullman.
\newblock Differentially private algorithms for learning mixtures of separated
  gaussians.
\newblock In {\em 2020 Information Theory and Applications Workshop (ITA)},
  pages 1--62. IEEE, 2020.

\bibitem{kamath2020private}
Gautam Kamath, Vikrant Singhal, and Jonathan Ullman.
\newblock Private mean estimation of heavy-tailed distributions.
\newblock {\em arXiv preprint arXiv:2002.09464}, 2020.

\bibitem{kaplan2020privately}
Haim Kaplan, Katrina Ligett, Yishay Mansour, Moni Naor, and Uri Stemmer.
\newblock Privately learning thresholds: Closing the exponential gap.
\newblock In {\em Conference on Learning Theory}, pages 2263--2285. PMLR, 2020.

\bibitem{karmalkar2019list}
Sushrut Karmalkar, Adam Klivans, and Pravesh Kothari.
\newblock List-decodable linear regression.
\newblock In {\em Advances in Neural Information Processing Systems}, pages
  7423--7432, 2019.

\bibitem{karmalkar2019compressed}
Sushrut Karmalkar and Eric Price.
\newblock Compressed sensing with adversarial sparse noise via l1 regression.
\newblock In {\em 2nd Symposium on Simplicity in Algorithms}, 2019.

\bibitem{KV17}
Vishesh Karwa and Salil Vadhan.
\newblock Finite sample differentially private confidence intervals.
\newblock {\em arXiv preprint arXiv:1711.03908}, 2017.

\bibitem{klivans2018efficient}
Adam Klivans, Pravesh~K Kothari, and Raghu Meka.
\newblock Efficient algorithms for outlier-robust regression.
\newblock In {\em Conference On Learning Theory}, pages 1420--1430, 2018.

\bibitem{kong2020robust}
Weihao Kong, Raghav Somani, Sham Kakade, and Sewoong Oh.
\newblock Robust meta-learning for mixed linear regression with small batches.
\newblock {\em Advances in Neural Information Processing Systems}, 33, 2020.

\bibitem{kothari2018robust}
Pravesh~K Kothari, Jacob Steinhardt, and David Steurer.
\newblock Robust moment estimation and improved clustering via sum of squares.
\newblock In {\em Proceedings of the 50th Annual ACM SIGACT Symposium on Theory
  of Computing}, pages 1035--1046, 2018.

\bibitem{lai2016agnostic}
Kevin~A Lai, Anup~B Rao, and Santosh Vempala.
\newblock Agnostic estimation of mean and covariance.
\newblock In {\em 2016 IEEE 57th Annual Symposium on Foundations of Computer
  Science (FOCS)}, pages 665--674. IEEE, 2016.

\bibitem{li-notes}
Jerry Li.
\newblock {CSE 599-M, Lecture Notes: Robustness in Machine Learning }, 2019.
\newblock URL: \url{https://jerryzli.github.io/robust-ml-fall19/lec7.pdf}.

\bibitem{li2020robust}
Jerry Li and Guanghao Ye.
\newblock Robust gaussian covariance estimation in nearly-matrix multiplication
  time.
\newblock {\em Advances in Neural Information Processing Systems}, 33, 2020.

\bibitem{liu2018high}
Liu Liu, Yanyao Shen, Tianyang Li, and Constantine Caramanis.
\newblock High dimensional robust sparse regression.
\newblock {\em arXiv preprint arXiv:1805.11643}, 2018.

\bibitem{liu2017fast}
Xiaohui Liu.
\newblock Fast implementation of the tukey depth.
\newblock {\em Computational Statistics}, 32(4):1395--1410, 2017.

\bibitem{liu2019fast}
Xiaohui Liu, Karl Mosler, and Pavlo Mozharovskyi.
\newblock Fast computation of tukey trimmed regions and median in dimension p>
  2.
\newblock {\em Journal of Computational and Graphical Statistics},
  28(3):682--697, 2019.

\bibitem{lugosi2019sub}
G{\'a}bor Lugosi, Shahar Mendelson, et~al.
\newblock Sub-gaussian estimators of the mean of a random vector.
\newblock {\em Annals of Statistics}, 47(2):783--794, 2019.

\bibitem{mcsherry2007mechanism}
Frank McSherry and Kunal Talwar.
\newblock Mechanism design via differential privacy.
\newblock In {\em 48th Annual IEEE Symposium on Foundations of Computer Science
  (FOCS'07)}, pages 94--103. IEEE, 2007.

\bibitem{mukhoty2019globally}
Bhaskar Mukhoty, Govind Gopakumar, Prateek Jain, and Purushottam Kar.
\newblock Globally-convergent iteratively reweighted least squares for robust
  regression problems.
\newblock In {\em The 22nd International Conference on Artificial Intelligence
  and Statistics}, pages 313--322, 2019.

\bibitem{prasad2018robust}
A.~Prasad, A.~S. Suggala, S.~Balakrishnan, and P.~Ravikumar.
\newblock Robust estimation via robust gradient estimation.
\newblock {\em arXiv preprint arXiv:1802.06485}, 2018.

\bibitem{raghavendra2020list}
Prasad Raghavendra and Morris Yau.
\newblock List decodable learning via sum of squares.
\newblock In {\em Proceedings of the Fourteenth Annual ACM-SIAM Symposium on
  Discrete Algorithms}, pages 161--180. SIAM, 2020.

\bibitem{steinhardt2018resilience}
Jacob Steinhardt, Moses Charikar, and Gregory Valiant.
\newblock Resilience: A criterion for learning in the presence of arbitrary
  outliers.
\newblock In {\em 9th Innovations in Theoretical Computer Science Conference
  (ITCS 2018)}. Schloss Dagstuhl-Leibniz-Zentrum fuer Informatik, 2018.

\bibitem{apple}
Jun Tang, Aleksandra Korolova, Xiaolong Bai, Xueqiang Wang, and Xiaofeng Wang.
\newblock Privacy loss in apple's implementation of differential privacy on
  macos 10.12.
\newblock {\em arXiv preprint arXiv:1709.02753}, 2017.

\bibitem{tao2012topics}
Terence Tao.
\newblock {\em Topics in random matrix theory}, volume 132.
\newblock American Mathematical Soc., 2012.

\bibitem{tukey1960survey}
John~W Tukey.
\newblock A survey of sampling from contaminated distributions.
\newblock {\em Contributions to probability and statistics}, pages 448--485,
  1960.

\bibitem{wainwright2019high}
Martin~J Wainwright.
\newblock {\em High-dimensional statistics: A non-asymptotic viewpoint},
  volume~48.
\newblock Cambridge University Press, 2019.

\bibitem{wei2016analysis}
Lu~Wei, Anand~D Sarwate, Jukka Corander, Alfred Hero, and Vahid Tarokh.
\newblock Analysis of a privacy-preserving pca algorithm using random matrix
  theory.
\newblock In {\em 2016 IEEE Global Conference on Signal and Information
  Processing (GlobalSIP)}, pages 1335--1339. IEEE, 2016.

\bibitem{xiao2015feature}
Huang Xiao, Battista Biggio, Gavin Brown, Giorgio Fumera, Claudia Eckert, and
  Fabio Roli.
\newblock Is feature selection secure against training data poisoning?
\newblock In {\em International Conference on Machine Learning}, pages
  1689--1698. PMLR, 2015.

\bibitem{zhang2020privately}
Huanyu Zhang, Gautam Kamath, Janardhan Kulkarni, and Zhiwei~Steven Wu.
\newblock Privately learning markov random fields.
\newblock {\em arXiv preprint arXiv:2002.09463}, 2020.

\bibitem{zhu2019generalized}
Banghua Zhu, Jiantao Jiao, and Jacob Steinhardt.
\newblock Generalized resilience and robust statistics.
\newblock {\em arXiv preprint arXiv:1909.08755}, 2019.

\bibitem{zhu2020does}
Banghua Zhu, Jiantao Jiao, and Jacob Steinhardt.
\newblock When does the tukey median work?
\newblock {\em arXiv preprint arXiv:2001.07805}, 2020.

\end{thebibliography}

\newpage 
\appendix

\addcontentsline{toc}{section}{Appendix} 
\part{Appendix} 
\parttoc



\section{Related work}
\label{sec:related}

{\bf Private statistical analysis.} 
Traditional private data analyses require bounded support of the samples to leverage the resulting bounded sensitivity. For example, each entry is constrained to have finite $\ell_2$ norm in standard private principal component analysis \cite{PPCA}, which does not apply to Gaussian samples.  
Fundamentally departing from these approaches,  \cite{KV17} first established an optimal mean estimation of Gaussian samples with {\em unbounded} support. The breakthrough is in first  adaptively estimating the range of the data using a private histogram, thus bounding the support and the resulting sensitivity.  This spurred the design of private algorithms for high-dimensional mean and covariance estimation \cite{KLSU19,biswas2020coinpress},  heavy-tailed mean estimation \cite{kamath2020private}, learning mixture of Gaussian   \cite{kamath2020differentially}, learning Markov random fields \cite{zhang2020privately}, and statistical testing \cite{canonne2019private}. 
Under the Gaussian distribution  with no adversary, \cite{aden2020sample} achieves an accuracy of $\|\hat{\mu}-\mu\|_2 \leq \tilde\alpha$ 
with the best known sample complexity of $n=\widetilde{O}((d/\tilde\alpha^2)+(d/\tilde\alpha \varepsilon) + (1/\varepsilon)\log(1/\delta))$ while guaranteeing $(\varepsilon,\delta)$-differential privacy. This nearly matches the known lower bounds of 
$\Omega(d/\tilde\alpha^2)$ for non-private finite sample complexity, $\widetilde\Omega((1/\varepsilon)\log(1/\delta))$ for privately learning one-dimensional unit variance Gaussian \cite{KV17}, and $\widetilde\Omega(d/\tilde\alpha\varepsilon)$ for multi-dimensional Gaussian estimation \cite{KLSU19}. However, this does not generalize to sub-Gaussian distributions and \cite{aden2020sample} does not provide a tractable algorithm. 
A polynomial time algorithm is proposed in \cite{KLSU19} that achieves a slightly worse sample complexity of $\widetilde{O}((d/\tilde\alpha^2)+(d\log^{1/2}(1/\delta)/\tilde\alpha \varepsilon) )$, which can also seamlessly generalized to sub-Gaussian distributions. 

\cite{cai2019cost} takes a different approach of deviating from standard definition of sub-Gaussianity to provide a larger lower bound on the sample complexity scaling as $n = \Omega(d \sqrt{\log(1/\delta)} / (\alpha \varepsilon))$ 
for mean estimation with a known covariance. Concretely, they consider distributions satisfying $ {\mathbb E}_{x\sim P}[e^{\lambda \langle x-\mu , e_k \rangle}] \leq e^{\lambda^2\sigma^2}$ for all $k\in[d]$ where $e_k$ is the $k$-th standard basis vector. Notice that this condition only requires sub-Gaussianity when projected onto standard bases. Standard definition of high-dimensional sub-Gaussianity (which is assumed in this paper)  requires sub-Gaussianity in  all directions. Therefore, their lower bound is not comparable with our achievable upper bounds. Further, the example they construct to show the lower bound does not satisfy our sub-Gaussianity assumptions. 

 In an attempt to design efficient algorithms for robust and private mean estimation, 
 \cite{dhar2020designing} proposed an algorithm with a mis-calculated sensitivity, which can result in violating the privacy guarantee. 
 This can be corrected by pre-processing with 
 our approach of checking the  resilience (as in Algorithm~\ref{alg:exp}), but this requires a run-time exponential in the  dimension. 

For estimating the mean of a {\em covariance bounded} distributions up to an error  of $\|\hat\mu-\mu\|
_2 =O(\tilde\alpha^{1/2})$, \cite{kamath2020private} shows that  $\Omega(d/(\tilde\alpha \varepsilon))$ samples are necessary and provides an efficient algorithm matching this up to a factor of $\log^{1/2}(1/\delta)$. 
For a more general family of distributions with bounded $k$-moment, \cite{kamath2020private} shows that an error  of $\|\hat\mu-\mu\|
_2 =O(\tilde\alpha^{(k-1)/k})$ can be achieved with $n=\widetilde{O}((d/\tilde\alpha^{2(k-1)/k})+(d\log^{1/2}(1/\delta)/(\varepsilon\tilde\alpha)))$ samples. 

However,  under $\alpha$-corruption, 
\cite{hopkins2019hard} shows that achieving an error better than $O(\alpha^{1/2})$ under $k$-th moment bound 
is as computationally hard as  
the small-set expansion problem, even without requiring DP. 
Hence, under the assumption of ${\rm P}\neq {\rm NP}$, no polynomial-time algorithm exists  
that can outperform our {\sc PRIME-ht} even if we have stronger assumptions of $k$-th moment bound. On the other hand, there exists  an exponential time algorithm for non-private robust mean estimation that achieves $\|\mu-\hat\mu\|_2=O(\alpha^{(k-1)/k})$ \cite{zhu2019generalized}.
Combining it with the bound of \cite{hopkins2019hard}, 
an interesting open question is whether there is an (exponential time) algorithm that achieves $\|\mu-\hat\mu\|_2=O(\alpha^{(k-1)/k})$ with sample complexity $n=\widetilde{O}( (d/ \alpha^{2(k-1)/k})+(d\log^{1/2}(1/\delta)/(\varepsilon \alpha)) )$ under $\alpha$-corruption and $(\varepsilon,\delta)$-DP.

\medskip\noindent{\bf Robust estimation.} 
Designing robust estimators under the presence of outliers has been considered by statistics community since 1960s~\cite{tukey1960survey, anscombe1960rejection, huber1964robust}. Recently, \cite{diakonikolas2019robust, lai2016agnostic} give the first polynomial time algorithm for mean and covariance estimation with no (or very weak) dependency on the dimensionality in the estimation error. Since then, there has been a flurry of research on robust estimation problems, including mean estimation~\cite{diakonikolas2017beingrobust, dong2019quantum, hopkins2020robust, hopkins2020mean, diakonikolas2018robustly}, covariance estimation~\cite{cheng2019faster, li2020robust}, linear regression and sparse regression \cite{bhatia2015robust,BhatiaJKK17, BDLS17, gao2020robust, prasad2018robust, klivans2018efficient, diakonikolas2019sever, liu2018high, karmalkar2019compressed,dalalyan2019outlier, mukhoty2019globally, diakonikolas2019efficient, karmalkar2019list}, principal component analysis~\cite{kong2020robust, jambulapati2020robust}, mixture models~\cite{diakonikolas2020robustly, jia2019robustly, kothari2018robust, hopkins2018mixture} and list-decodable learning \cite{diakonikolas2018list, raghavendra2020list, charikar2017learning, bakshi2020list, cherapanamjeri2020list}. See~\cite{diakonikolas2019recent} for a survey of recent work.

One line of work that is particularly related to our algorithm PRIME is ~\cite{cheng2019high,dong2019quantum, depersin2019robust, cheng2019faster, cherapanamjeri2020list}, which leverage the ideas from matrix multiplicative weight and fast SDP solver to achieve faster, sometimes nearly linear time, algorithms for mean and covariance estimation. In PRIME, we use a matrix multiplicative weight approach similar to~\cite{dong2019quantum} to reduce the iteration complexity to logarithmic, which enables us to achieve the $d^{3/2}$ dependency in the sample complexity.

The concept of \textit{resilience} is introduced in~\cite{steinhardt2018resilience} as a sufficient condition such that learning in the presence of adversarial corruption is information-theoretically possible. The idea of resilience is later generalized in \cite{zhu2019generalized} for a wider range of adversarial corruption models. While there exists simple exponential time robust estimation algorithm under resilience condition, it is challenging  to achieve differential privacy  due to high sensitivity.  We propose a novel approach to leverage the resilience property in our exponential time algorithm for sub-gaussian and heavy-tailed distributions.

\section{Main results under heavy-tailed distributions}
\label{sec:heavytail}

We consider distributions with bounded covariance as defined as follows.

\begin{asmp}
\label{asmp:adversary2}
An uncorrupted dataset $S_{\rm good}$ consists of $n$ i.i.d.~samples from a distribution with mean $\mu\in \reals^d$ and covariance $\Sigma\preceq {\mathbf I}$. For some $\alpha\in(0,1/2)$,
we are given a corrupted dataset  $S=\{x_i\}_{i=1}^n$  where an adversary adaptively inspects all  samples in $S_{\rm good}$, removes $\alpha n$ of them and replaces them with $S_{\rm bad}$ that are $\alpha n$ arbitrary points in $\reals^d$. 
\end{asmp}

Under these assumptions, Algorithm~\ref{alg:exp} achieves near optimal guarantees but takes exponential time. 
The dominant term in the sample complexity $\widetilde{\Omega}(d/(\varepsilon \alpha))$
cannot be improved as it matches that of the optimal non-robust private estimation \cite{kamath2020private}. The accuracy $O(\sqrt\alpha)$ cannot be improved as it matches that of the optimal non-private robust estimation \cite{dong2019quantum}. 
We provide a proof in Appendix~\ref{sec:proof_heavytail_exp}. 

\begin{thm}[Exponential time algorithm for covariance bounded distributions]
    \label{thm:heavytail_exp}
    Algorithm~\ref{alg:exp} is $(\varepsilon,\delta)$-differentially private. 
    Under Assumption~\ref{asmp:adversary2}, if $$n=\Omega\Big( \frac{d\log(d/\alpha^{1.5}) +d^{1/2}\log(1/\delta)}{\varepsilon \alpha} +  \frac{d^{1/2}\log^{3/2}(1/\delta)log(d/\delta)}{\varepsilon} \Big)\;,$$ 
    this algorithm achieves 
    $\|\hat\mu-\mu\|_2=O(\sqrt{\alpha})$ with probability $0.9$. 
\end{thm}

We propose an efficient algorithm {\sc PRIME-ht} and show that it achieves the same optimal accuracy but at the cost of increased sample complexity of $O(d
^{3/2}\log(1/\delta)/(\varepsilon\alpha))$. 
In the first step, we need increase the radius of the ball to $O(\sqrt{d/\alpha})$ to include a $1-\alpha$ fraction  of the clean samples, where $q_{\rm range-ht}$ returns $B=O(1/\sqrt{\alpha})$ and ${\cal B}_{\sqrt{d} B/2}(\bar{x})$ is a $\ell_2$-ball of radius $\sqrt{d} B/2$ centered at $\bar{x}$.  
This is followed by  a matrix multiplicative weight filter similar to {\sc DPMMWfilterr} but the parameter choices are tailored for covariance bounded distributions. We provide a proof in Appendix~\ref{sec:proof_heavytail_poly}.

\begin{thm}[Efficient algorithm for covariance bounded distributions]
    {\sc PRIME-ht} is  $(\varepsilon,\delta)$-differentially private. 
    Under  Assumption~\ref{asmp:adversary2} there exists a universal constant $c\in(0,0.1)$ such that if $\alpha\leq c$, and $n=\widetilde\Omega (    (d^{3/2}/(\varepsilon\alpha) )\log(1/\delta) )$, 
    then {\sc PRIME-ht} achieves $\|\hat\mu-\mu\|_2 = O(\alpha^{1/2})$ with probability $0.9$. The notation $\widetilde\Omega(\cdot)$ hides logarithmic terms in $d$, and $1/\alpha$.  
    \label{thm:heavytail_poly}
\end{thm}

 {\em Remark 1.} To boost the success probability to $1-\zeta$ for some small $\zeta>0$, we will randomly split the data into $O(\log(1/\zeta))$ subsets of equal sizes, and run Algorithm~\ref{alg:prime-ht} to obtain a mean estimation from each of the subset. Then we can apply multivariate ``mean-of-means'' type estimator~\cite{lugosi2019sub} to get $\|\hat{\mu}-\mu\|_2 = O(\alpha^{1/2})$ with probability $1-\zeta$.
This is efficient as 
we only have $O(\log 1/\zeta)$ trials and run-time of mean-of-means is dominated by the time it takes to find all pairwise distances, which is only $O(d \, ({\rm log}(1/\zeta))^2)$. There are  $({\rm log}(1/\zeta))^2$ pairs, and for each pair we compute the distance between means in $d$ operations.

\begin{algorithm2e}[h]
   \caption{PRIvate and robust Mean Estimation for covariance bounded distributions ({\sc PRIME-ht})}
   \label{alg:prime-ht} 
   	\DontPrintSemicolon 
	\KwIn{$ S  =  \{ x_{i} \in {\mathbb R}^d\}_{i=1}^{n} $, adversarial fraction $\alpha \in(0,1/2)$, 
	number of iterations $T_1=O(\log (d/\alpha)),T_2 = O(\log d)$, target privacy $(\varepsilon,\delta)$   }
	\SetKwProg{Fn}{}{:}{}
	{
	$(\bar x, B) \leftarrow q_{\rm range-ht} (S, 0.01\varepsilon,0.01\delta)$  \hfill [Algorithm~\ref{alg:DPrange-ht} in Appendix~\ref{sec:proof_heavytail}]\\
	Project the data onto the ball: $\tilde x_i \leftarrow {\cal P}_{{\cal B}_{\sqrt{d}B/2}(\bar{x})}(x_i)$, for all $i\in[n]$\;
	$\hat\mu \leftarrow$ \text{\sc DPMMWfilter-ht}($ \{\tilde x_i\}_{i=1}^n,\alpha,T_1,T_2,0.99\varepsilon,0.99\delta$) \hfill [Algorithm~\ref{alg:DPMMWfilter_ht} in Appendix~\ref{sec:proof_heavytail}]  \\
	\KwOut{$\hat\mu$}
	} 
\end{algorithm2e}

\newpage
\section{Background on (non-private) robust mean estimation}
\label{sec:nonprivate}

The following tie-breaking rule is not essential for robust estimation, but is critical for proving differential privacy, as shown later in Appendix~\ref{sec:proof_main2_dp}. 

\begin{definition}[Subset of the largest $\alpha$ fraction]
    \label{def:topset}
    Given a set of scalar values $\{\tau_i=\langle V, (x_i-\mu)(x_i-\mu)^\top\rangle \}_{i\in S'}$ for a subset $S'\subseteq[n]$,  
    define the sorted list $\pi$ of $S'$ such that $\tau_{\pi(i)} \geq \tau_{\pi(i+1)}$ for all $i\in [|S'|-1]$. When there is a tie such that $\tau_i=\tau_j$, it is broken  
    by $\pi^{-1}(i) \leq \pi^{-1}(j) \Leftrightarrow x_{i,1}\geq x_{j,1}$. Further ties are broken by comparing the remaining entries of $x_i$ and $x_j$, in an increasing order of the coordinate.  If $x_i=x_j$ ,then the tie is broken arbitrarily. 
    We define ${\cal T}_{\alpha}=\{\pi(1),\ldots,\pi(\lceil n\alpha\rceil)\}$ to be the set of largest $\lceil n\alpha\rceil $ valued samples. 
\end{definition}
With this definition of $\alpha$-tail, we can now provide a complete description of the robust mean estimation that achieves the guarantee provided in Proposition~\ref{pro:nonprivatefiltering}. 

\begin{algorithm2e}[ht]
   \caption{ Non-private robust mean estimation \cite{li-notes} }
   \label{alg:nonprivatefiltering}
   	\DontPrintSemicolon 
	\KwIn{$S =  \{ x_{i}  \}_{i=1}^{n} $, $\alpha \in(0,1)$, $S_0 = [n]$ }
	\SetKwProg{Fn}{}{:}{}
	{
	    \For{$t=1,\ldots$}{
	    \label{line:PCA1}
	    \If{$ \| \sum_{i\in S_{t-1}}(x_i-\mu_{t-1})(x_i-\mu_{t-1})^\top -{\mathbf I}\|_2 < C \alpha \log(1/\alpha) $ }{
	    \KwOut{$\hat\mu = \sum_{i\in S_{t-1}}   x_i$} \label{line:nonprivate_stop}
	    } 
	    \Else{ 
	            $\mu_t \leftarrow  (1/|S_{t-1}|)\sum_{i\in S_{t-1} }   x_i $ \label{line01}\;
            $v_t\leftarrow \text{1st principal direction of }( \, \{ (x_i-\mu_t)\}_{i\in S_{t-1} } )$  \;
    $Z_t \leftarrow {\rm Unif}([0,1])$\;
            $S_t \leftarrow S_{t-1} \setminus$  $\{ i \,|\, i\in{\cal T}_{2\alpha} \text{ for }\{\tau_j = (v_t^\top (x_j-\mu_t) )^2\}_{j\in S_{t-1}}$  and $\tau_i \geq  Z_t\,\max_{j\in S_{t-1}}(v_t^\top(x_j-\mu_t))^2  \}$, where ${\cal T}_{2\alpha}$ is defined in Definition~\ref{def:topset}. \label{line02} \;
   \label{line:PCA2}
   }
   }
    }
\end{algorithm2e}

\section{A new framework for {\em private} iterative filtering} 
\label{sec:app_novel}

We provide complete descriptions of all algorithms used in private iterative filtering. We present the {\em interactive} version first, followed by the {\em centralized} version. 

\subsection{Interactive version of the algorithm} 
\label{sec:proof_dprange_lemma}

Adaptive estimation of the range of the dataset is essential in computing private statistics of data. We use the following algorithm proposed in \cite{KV17}. It computes a private histogram of a set of  1-dimensional points and select the largest bin as the one potentially containing the mean of the data. Note that $B$ does not  need not be chosen adaptively  to include all the uncorrupted data with a high probability.

\begin{algorithm2e}[ht]
   \caption{Differentially private range estimation ($q_{\rm range}$) {\cite[Algorithm 1]{KV17}} }
   \label{alg:DPrange}
   \DontPrintSemicolon 
   	\KwIn{${\cal D}_n  =  \{ x_{i} \}_{i=1}^{n} $, $\varepsilon$, $\delta$, $\sigma=1$ }
   \SetKwProg{Fn}{}{:}{}
   {
      \For{$j\gets 1$ \KwTo $d$}{
      $R_{\rm max}^{(j)}\leftarrow \max_{i\in [n]}x_{i}^{(j)}$ and $R_{\rm min}^{(j)}\leftarrow \min_{i\in [n]}x_{i}^{(j)}$ where $x_i^{(j)}$ is the $j$-th coordinate of $x_i$\;
   	Run the histogram learner of Lemma~\ref{lem:hist-KV17} with privacy parameters $\left(\min\{\eps, 0.9\}/2\sqrt{2d\log(2/\delta)},\delta/(2d)\right)$ and bins $B_l = (2\sigma\ell,2\sigma(\ell+1)]$ for all $\ell\in\{\lceil R_{\rm min}^{(j)}/2\sigma\rceil-1,\ldots,\lceil R_{\rm max}^{(j)}/2\sigma\rceil \}$ on input $\mathcal{D}_n$ to obtain noisy estimates $\{\tilde{h}_{j,l}\}_{l=\lceil R_{\rm min}^{(j)}/2\sigma\rceil-1}^{\lceil R_{\rm max}^{(j)}/2\sigma\rceil}$\;
	$\bar x_j \leftarrow  2\sigma \cdot \argmax_{\ell\in\{\lceil R_{\rm min}^{(j)}/2\sigma\rceil-1,\ldots,\lceil R_{\rm max}^{(j)}/2\sigma\rceil \}} \tilde{h}_{j,\ell} $\;
	}
		\KwOut{$(\bar{x},B= 8\sigma\sqrt{\log(dn/\zeta)} )$}
	}
\end{algorithm2e} 
The following guarantee (and the algorithm description) is used in the analysis (and the implementation) of the query $q_{\rm range}$. 

\begin{lemma}[Histogram Learner, Lemma 2.3 in~\cite{KV17}]\label{lem:hist-KV17} For every $K\in \mathbb{N}\cup \infty$, domain $\Omega$, for every collection of disjoint bins $B_1,\ldots, B_K$ defined on $\Omega$, $n\in \mathbb{N}$, $\eps,\delta\in(0,1/n)$, $\beta>0$ and $\alpha\in (0,1)$ there exists an $(\eps,\delta)$-differentially private algorithm $M:\Omega^n\to \mathbb{R}^K$ such that for any set of data $X_1,\ldots,X_n\in \Omega^n$
\begin{enumerate}
\item $\hat{p}_k = \frac{1}{n}\sum_{X_i\in B_k}1$
\item $(\tilde{p}_1,\ldots,\tilde{p}_K)\gets M(X_1,\ldots,X_n),$ and
\item
$$
n\ge \min\left\{\frac{8}{\eps\beta}\log(2K/\alpha),\frac{8}{\eps\beta}\log(4/\alpha\delta)\right\} 
$$
\end{enumerate}
then,
$$
\mathbb{P}(|\tilde{p}_k-\hat{p}_k|\le\beta)\ge 1-\alpha
$$
\end{lemma}
\begin{proof}
This is an intermediate result in the proof of Lemma 2.3 in~\cite{KV17}.
Note that, conceptually, we are applying the private histogram algorithm to an infinite number of bins in the intervals $ \{\cdots, (-4\sigma,-2\sigma],(-2\sigma,0],(0,2\sigma],(2\sigma,4\sigma],(4\sigma,6\sigma]\cdots \}$ each of length $2\sigma$. This is possible because the algorithm only changes the bins that are occupied by at least on sample. Practically, we only need to add noise to those bins that are occupied, and hence we limit the range from $R_{\rm min}^{(j)}$ to $R_{\rm max}^{(j)}$ without loss of generality and without any changes to the privacy guarantee of the algorithm.

\end{proof}

The rest of the queries ($q_{\rm size}$, $q_{\rm mean}$, $q_{\rm PCA}$, and $q_{\rm norm}$) are provided below. The most innovative part is the repeated application of filtering that is run every time one of the queries is called. In the Filter query below, because we choose $(i)$ to use the sampling version of robust mean estimation as opposed to weighting version which assigned  a weight on each sample between zero and one measuring how good (i.e., score one) or bad (i.e., score zero) each sample point is, and 
$(ii)$ we switched the threshold to be $dB^2 Z_\ell$, we can show that this filtering with fixed parameters $\{\mu_\ell,v_\ell,Z_\ell\}_{\ell\in[t-1]}$ preserves sensitivity in Lemma~\ref{lem:DPfilter_sensitivity}. 
  This justifies the choice of noise in each output perturbation mechanism,  satisfying the desired level of $(\varepsilon,\delta)$-DP.
  We provide the complete privacy analysis in Appendix~\ref{sec:proof_main1} and also the analysis of the utility of the algorithm as measure by the accuracy. 
  
\begin{algorithm2e}[ht]
   \caption{Interactive private queries used in Algorithm~\ref{alg:DPfilter_interactive_main} }
   \label{alg:DPfilter_interactive}
   	\DontPrintSemicolon 
	\SetKwProg{Fn}{}{:}{}
	{
    } 
    \vspace*{5pt}
 \Fn{\text{\rm Filter} ($\{(\mu_\ell,v_\ell,Z_\ell)\}_{\ell \in[t-1]},\bar x, B $)}{
        $S_0 \leftarrow [n] $ \;
        	Clip the data points: $ x_i \leftarrow {\cal P}_{\bar{x}+[-B/2,B/2]^d}(x_i)$, for all $i\in[n]$\;
        \For{$\ell=1,\ldots,t-1$}{
            $S_\ell \leftarrow S_{\ell-1} \setminus$  $\{ i\in S_{\ell-1} :$  $i\in {\cal T}_{2\alpha}$ for  $\{\tau_j=(v_\ell^\top (x_j-\mu_\ell) )^2\}_{j\in S_{\ell-1}}$ and    
   $\tau_i \geq  d\,B^2 \,Z_{\ell} \}$ 
        }
    }    
	\Fn{$q_{\rm mean}(\{(\mu_\ell,v_\ell,Z_\ell)\}_{\ell \in[t-1]},\varepsilon,\bar x, B)$}{
	    Filter($\{(\mu_\ell,v_\ell,Z_\ell)\}_{\ell \in[t-1]},\bar x, B$)\;
        \KwRet{$\mu_t\leftarrow  (1/|S_{t-1}|) \big( \sum_{i\in S_{t-1}} x_i\big)  + {\rm Lap}({2B}/({n \varepsilon })) $}
	}
	\Fn{$q_{\rm PCA}(\{(\mu_\ell,v_\ell,Z_\ell)\}_{\ell \in[t-1]},\mu_t,\varepsilon,\delta,\bar x, B )$}{
        Filter($\{(\mu_\ell,v_\ell,Z_\ell)\}_{\ell \in[t-1]},\bar x, B$)\;
        {\bf return} $v_t\leftarrow$ \text{\rm top singular vector of} $\Sigma_{t-1} = $\\
        $(1/n)\sum_{i\in S_{t-1}}(x_i-\mu_t)(x_i-\mu_t)^\top + {\cal N}(0,(B^2d\sqrt{2\log(1.25/\delta)} /(n\varepsilon ) )^2{\mathbf I}_{d^2\times d^2})$
	}
	
	\Fn{$q_{\rm norm} (\{(\mu_\ell, v_\ell,Z_\ell)\}_{\ell \in[t-1]},\mu_t,\varepsilon,\bar x, B)$}{
        Filter($\{(\mu_\ell,v_\ell,Z_\ell)\}_{\ell \in[t-1]},\bar x, B$)\;
        \KwRet{$\lambda_t \gets \|(1/n)\sum_{i\in S_{t-1}} (x_i-\mu_t)(x_i-\mu_t)^\top\|_2 + {\rm Lap}(2B^2d/(n\varepsilon))$
        }
	}
\vspace*{5pt}
	\Fn{$q_{\rm size} (\{(\mu_\ell,v_\ell,Z_\ell)\}_{\ell \in[t-1]},\varepsilon,\bar x, B)$}{
        Filter($\{(\mu_\ell,v_\ell,Z_\ell)\}_{\ell \in[t-1]},\bar x, B$)\;
        \KwRet{$n_t\leftarrow |S_{t-1}|+{\rm Lap}(1/\varepsilon)$ 
        }
	}
\end{algorithm2e} 

\subsection{Centralized version of the algorithm}

In practice, one should run the centralized version of the private iterative filtering, in order to avoid multiple redundant computations of the interactive version. 
The main difference is that the redundant filtering repeated every time a query is called in the interactive version is now merged into a single run. The resulting estimation and the privacy loss are exactly the same.

\begin{algorithm2e}[H]
   \caption{Private iterative filtering (centralized version)}
   \label{alg:privatefiltering} 
   	\DontPrintSemicolon 
	\KwIn{$ S  =  \{ x_{i} \in {\mathbb R}^d\}_{i=1}^{n} $, adversarial fraction $\alpha \in(0,1)$, target probability $\eta\in(0,1)$, number of iterations $T = \widetilde{\Theta}(d)$, target privacy $(\varepsilon,\delta)$  }
	\SetKwProg{Fn}{}{:}{}
	{
	$(\bar x, B) \leftarrow q_{ \rm range}(S, 0.01\varepsilon,0.01\delta)$  \hfill [Algorithm~\ref{alg:DPrange}]\\
	Clip the data points: $\tilde x_i \leftarrow {\cal P}_{\bar{x}+[-B/2,B/2]^d}(x_i)$, for all $i\in[n]$\\
	$\hat\mu \leftarrow$ \text{ \sc DPfilter}($ \{\tilde x_i\}_{i=1}^n,\alpha,T,0.99\varepsilon,0.99\delta$) \hfill \label{line:dpfilter}  [Algorithm~\ref{alg:DPfilter}]\\
	\KwOut{$\hat\mu$}
	} 
\end{algorithm2e} 

First, $q_{\rm range}$ introduced in 
\cite{KV17}, returns a hypercube $\bar{x}+[-B,B]^d$ that is guaranteed to include all uncorrupted samples, while preserving privacy. It is followed by a private filtering  {\sc DPfilter} in Algorithm~\ref{alg:DPfilter}.

\begin{algorithm2e}[ht]
   \caption{Differentially private filtering ({\sc DPfilter}) }
   \label{alg:DPfilter}
   	\DontPrintSemicolon 
	\KwIn{$S  =  \{ x_{i} \in \bar{x}+[-B/2,B/2]^d\}_{i=1}^{n} $, $\alpha \in(0,1/2)$, $T=\widetilde{O}(d B^2\log(dB^2 /(\alpha\log(1/\alpha)))  )$, $(\varepsilon,\delta)$ }
	\SetKwProg{Fn}{}{:}{}
	{
	    $S_0 \leftarrow [n]$, $\varepsilon_1\gets \min\{\varepsilon,0.9\}/(4\sqrt{2T\log(2/\delta)}) $, 
	    $\delta_1 \gets \delta/(8T)$\;
	    {\bf if} $n<(4/\varepsilon_1)\log(1/(2\delta_1))$  {\bf then Output:} $\emptyset$ \;
	    \For{$t=1,\ldots,T$}{
	        $n_t\gets |S_{t-1}|+{\rm Lap}(1/\varepsilon_1)$\;
	        \If{$n_t < 3n/4$}{{\bf Output:} $\emptyset$}
	        $\mu_t \leftarrow (1/|S_{t-1}|) \sum_{i\in S_{t-1}} x_i + {\rm Lap}( 2B/(n\, \varepsilon_1) )$  \label{line1}\;
	        $\lambda_t \gets \|(1/n)\sum_{i\in S_{t-1}}(x_i-\mu_t)(x_i-\mu_t)^\top-{\mathbf I}\|_2 + {\rm Lap}(2B^2d/(n\varepsilon_1))$\;
	        \If{$\lambda_t\leq (C-0.01)\alpha\log(1/\alpha) $}{\KwOut{$\mu_{t}$} }
            $v_t\leftarrow \text{top singular vector of }  \Sigma_{t-1} \triangleq\frac{1}{n}\sum_{i\in S_{t-1}} (x_i-\mu_t)(x_i-\mu_t)^\top + {\cal N}(0,(B^2 d\sqrt{2\log(1.25/\delta)} / (n\varepsilon_1) )^2{\mathbf I}_{d^2\times d^2})$ \label{line2} \;
            $Z_t \leftarrow {\rm Unif}([0,1])$\;
            $S_t \leftarrow S_{t-1} \setminus$  $\{i\,|\, i\in {\cal T}_
            {2\alpha} $  for $\{\tau_j=(v_t^\top (x_j-\mu_t) )^2\}_{j\in S_{t-1}}$ and $\tau_i \geq  d\,B^2 \,Z_t \}$, where ${\cal T}_{2\alpha}$ is defined in Definition~\ref{def:topset}. \label{line3}
	    }
    }
\end{algorithm2e} 

\subsection{The analysis of private iterative filtering (Algorithms \ref{alg:DPfilter_interactive_main} and \ref{alg:privatefiltering}) and a proof of Theorem~\ref{thm:main1}}
\label{sec:proof_main1}





$q_{\rm range}$, introduced in 
\cite{KV17}, returns a hypercube $\bar{x}+[-B,B]^d$ that is guaranteed to include all uncorrupted samples, while preserving privacy. 
In the following lemma, we show that $q_{\rm range}$ is also {\em robust} to adversarial corruption. 
Such adaptive bounding of the support is critical in privacy analysis of the subsequent steps. 
We clip all data points by projecting all the points with ${\cal P}_{\bar{x}+[-B/2,B/2]^d}(x)=\arg\min_{y\in \bar{x}+[-B/2,B/2]^d} \|y-x\|_2$ to lie inside the  hypercube and pass them to {\sc DPfilter} for filtering. 
 The algorithm and a proof are provided in \S\ref{sec:proof_DPrange}.

\begin{lemma}
    \label{lem:DPrange} 
    $q_{\rm range}(S,\varepsilon,\delta)$ (Algorithm~\ref{alg:DPrange}) is $(\varepsilon,\delta)$-differentially private. Under Assumption~\ref{asmp:adversary}, 
    $q_{\rm range}(S,\varepsilon,\delta)$ returns $(\bar{x},B)$ such that if $n =  \Omega \left( (\sqrt{d\log(1/\delta)}\log(d/(\zeta\delta))/\varepsilon)\right)$ and $\alpha<0.1$, then all uncorrupted samples in $S$ are in $\bar{x}+[-B,B]^d$ with probability $1-\zeta$. 
\end{lemma}

In {\sc DPfilter}, we  make only the mean $\mu_t$ and the top principal direction $v_t$ private to decrease sensitivity. 
The analysis is now more challenging since  
  $(\mu_t,v_t)$  depends on all  past iterates $\{(\mu_j,v_j)\}_{j=1}^{t-1}$ and  internal randomness $\{Z_j\}_{j=1}^{t-1}$.
 To decrease the sensitivity, we modify the filter in 
 line~\ref{line3} 
 to use the maximum support $dB^2$ (which is data independent) 
 instead of the maximum contribution $\max_i (v_t^\top(x_i-\mu_t))^2 $ 
 (which is data dependent and sensitive).
 While one data point can 
 significantly change $\max_i (v_t^\top(x_i-\mu_t))^2 $ and the output of one step of the filter in Algorithm~\ref{alg:nonprivatefiltering}, the 
 sensitivity of the proposed filter is  bounded  conditioned on all past $\{(\mu_j,v_j)\}_{j=1}^{t-1}$, as we show in the 
following lemma. 
This follows from the fact that
conditioned on $(\mu_j,v_j)$, the proposed filter is a contraction. 
We provide a proof in Appendix~\ref{sec:proof_DPfilter_part1} and Appendix~\ref{sec:proof_DPfilter_part2}.  Putting together Lemmas~\ref{lem:DPrange} and \ref{lem:DPfilter}, we get the desired result in Theorem~\ref{thm:main1}. 

\begin{lemma}
    \label{lem:DPfilter} 
    {\sc DPfilter}$(S,\alpha,T,\varepsilon,\delta)$ is $(\varepsilon,\delta)$-differentially private.
    Under the hypotheses of Theorem~\ref{thm:main1}, 
    {\sc DPfilter}$(S,\alpha,T=\widetilde{\Theta}(B^2d),\varepsilon,\delta)$ achieves  
    $\| \hat\mu-\mu \|_2 = O(\alpha\sqrt{\log(1/\alpha)})$ with probability $0.9$, if 
$n=\widetilde{\Omega}(d/\alpha^2 + B^3 d^2\log(1/\delta)/(\varepsilon\alpha))$ 
and $B$ is large enough such that the original uncorrupted samples are inside the hypercube $\bar{x}+[-B/2,B/2]^d $. 
\end{lemma}

\noindent
{\bf Differential privacy guarantee.} To achieve $(\varepsilon_0,\delta_0)$  end-to-end target privacy guarantee, 
Algorithm~\ref{alg:privatefiltering} separates the privacy budget into two.
The ($0.01\varepsilon_0,0.01\delta_0$)-DP guarantee of $q_{\rm range}$ follows from  Lemma~\ref{lem:DPrange}. 
The ($0.99\varepsilon_0,0.99\delta_0$)-DP guarantee of 
{\sc DPfilter} follows from Lemma~\ref{lem:DPfilter}. 

\medskip
\noindent{\bf Accuracy.} 
From  Lemma~\ref{lem:DPrange} $q_{\rm range}$ is guaranteed to return a hypercube that includes all clean data in the dataset. 
It follows from Lemma~\ref{lem:DPfilter} that 
when $n=\widetilde{\Omega}(d/\alpha^2 + d^2\log(1/\delta)/(\varepsilon\alpha))$, we have 
$\|\mu-\hat\mu\|_2 = O(\alpha\sqrt{\log(1/\alpha)} )$.

\subsubsection{Proof of Lemma~\ref{lem:DPrange} and the analysis of $q_{\rm range}$ in Algorithm~\ref{alg:DPrange}} 
\label{sec:proof_DPrange}

Assuming the distribution is $\sigma^2$ sub-Gaussian, we use $\mathcal{P}$ to denote the sub-Gaussian distribution. 
Denote $I_l = [2\sigma l, 2\sigma (l+1)]$ as the interval of the $l$'th bin. Denote the population probability in the $l$'th bin $h_{j,l}=\mathbb{P}_{x\sim \mathcal{P}}[x_j\in I_l]$, empirical probability in the $l$'th bin $\tilde{h}_{j,l}=\frac{1}{n}\sum_{x_i\in \mathcal{D}}\mathbf{1}\{x_{i,j}\in I_l\}$, and the noisy version $\hat{h}_{j,l}$ computed by the histogram learner of Lemma~\ref{lem:hist-KV17}. Notice that Lemma~\ref{lem:hist-KV17} with $d$ compositions 
(Lemma~\ref{lem:composition}) 
immediately implies that our algorithm is $(\eps,\delta)$-differentially private. 

For the utility of the algorithm, we will first show that for all dimension $j\in [d]$, the output $|\bar{x}_j-\mu_j|= O(\sigma)$. Note that by the definition of $\sigma^2$-subgaussian, it holds that for all $i\in [d]$, $\mathbb{P}[|x_i-\mu_i|\ge z]\le 2\exp(-z^2/\sigma^2)$ where $x$ is drawn from distribution $\mathcal{P}$. This implies that $\mathbb{P}[|x_i-\mu_i|\ge 2\sigma]\le 2\exp(-4)\le 0.04.$ Suppose the $k$'th bin contains $\mu_j$, namely $\mu_j\in I_k$. Then it is clear that  $[\mu_j-2\sigma,\mu_j+2\sigma] \subset (I_{k-1}\cup I_{k}\cup I_{k+1})$. This implies $h_{j,k-1}+h_{j,k}+h_{j,k+1}\ge 1-0.04 = 0.96$, hence $\min(h_{j,k-1},h_{j,k},h_{j,k+1})\ge 0.32$. 

Recall that $\mathcal{G}$ is the set of clean data drawn from distribution $P$. By Dvoretzky-Kiefer-Wolfowitz inequality and an union bound over $j\in [d]$, we have that with probability $1-\zeta$, $\max_{j,l}(|h_{j,l}-\frac{1}{n}\sum_{x\in G}x_j|)\le \sqrt{\frac{\log(d/\zeta)}{n}}$. The deviation due to corruption is at most $\alpha$ on each bin, hence we have $\max_{j,l}(|h_{j,l}-\hat{h}_{j,l})\le \sqrt{\frac{\log(d/\zeta)}{n}}+\alpha$. Lemma~\ref{lem:hist-KV17} and a union bound over $j\in [d]$ implies that with probability $1-\zeta$ , $\max_{j,l}(|\tilde{h}_{j,l}-\hat{h}_{j,l}|) \le \beta$ when $n\ge \Omega\left(\frac{\sqrt{d\log(1/\delta)}}{\eps\beta}\log(d/\zeta\delta)\right)$. 

Assuming that  $n=\Omega\left(\frac{\sqrt{d\log(1/\delta)}}{\eps\beta}\log(d/\zeta\delta)\right) $, we have that with probability $1-\zeta$, $\max_{j,l}(|{h}_{j,l}-\hat{h}_{j,l}|) \le 0.01+\alpha$. Using the assumption that $\alpha\le 0.1$, since $\min(h_{j,k-1},h_{j,k},h_{j,k+1})-0.11\ge 0.31 \ge 0.04+0.11\ge \max_{l\ne k-1, k, k+1}h_{j,l}+0.11$. This implies that with probability $1-\zeta$, the algorithm choose the bin from $k-1, k, k+1$, which means the estimate $|\bar{x}_j-\mu|\le 4\sigma$. By the tail bound of sub-Gaussian distribution and a union bound over $n, d$, we have that with probability $1-\zeta$, for all $x_i\in \mathcal{D}$ and $j\in [d]$, $x_{i,j}\in [\bar{x}_j-8\sigma\sqrt{\log(nd/\zeta)}, \bar{x}_j+8\sigma\sqrt{\log(nd/\zeta)}]$. 

\subsubsection{Proofs  of the sensitivity of the filtering in  Lemma~\ref{lem:DPfilter_sensitivity} and Lemma~\ref{lem:DPMMWfilter_sensitivity}}
\label{sec:proof_sensitivity}
\noindent{\bf Proof of Lemma~\ref{lem:DPfilter_sensitivity}.} We only need to show that one step of the proposed filter is a contraction. 
To this end, we only need to show contraction for two datasets at distance 1, i.e., $d_\triangle({\cal D},{\cal D'})=1$. For fixed $(\mu,v)$ and $Z$, we apply filter to set of scalars $(v^\top({\cal D}-\mu))^2 $ and $(v^\top({\cal D'}-\mu))^2$, whose distance is also one. 
If the entries that are different (say $a\in{\cal D}$ and $a' \in {\cal D'}$) are both below the subset of the top $2n\alpha$ points (as in Definition~\ref{def:topset}), then the same set of points will be removed for both and the distance is preserved $d_\triangle(S({\cal D}),S({\cal D}'))=1$. 
If they are both above the top $2n\alpha$ subset, then either both are removed, one of them is removed, or both remain. The rest of the points that are removed coincide in both sets. Hence, 
$d_\triangle(S({\cal D}),S({\cal D}'))\leq 1$. 
If $a$ is below and $a'$ is above the top $2n\alpha$ subset of respective datasets, then either 
$a'$ is not removed (in which case $d_\triangle(S({\cal D}),S({\cal D}')) = 1$) or $a'$ is removed (in which case $S({\cal D}) = S({\cal D}') \cup \{a\}$ and the distance remains one).

Note that when there are ties, it is critical to resolve them in a consistent manner in both datasets ${\cal D}$ and ${\cal D'}$. The tie breaking rule of Definition~\ref{def:topset} is critical in sorting those samples with the same score $\tau_i$'s in a consistent manner.

\medskip
\noindent{\bf Proof of Lemma~\ref{lem:DPMMWfilter_sensitivity}.}
The analysis of contraction of the filtering step in {\sc DPMMWfilter} is analogous to that of 
private iterative filtering in Lemma~\ref{lem:DPfilter_sensitivity}.

\subsubsection{Proof of part 1 of    Lemma~\ref{lem:DPfilter} on differential privacy of  {\sc DPfilter}}

\label{sec:proof_DPfilter_part1}

We explicitly write out how many times we access the database and how much privacy is lost each time in an interactive version of {\sc DPfilter} in Algorithm~\ref{alg:DPfilter_interactive_main}, which performs the same operations as {\sc DPfilter}. 
In order to apply Lemma~\ref{lem:composition}, we cap $\varepsilon$ at 0.9 in initializing $\varepsilon_1$. 
We call  $q_{\rm mean}$, $q_{\rm PCA}$, $q_{\rm norm}$ and $q_{\rm size}$ $T$ times, each with $(\varepsilon_1,\delta_1)$ guarantee. In total this accounts for $(\varepsilon,\delta)$ privacy loss, using Lemma~\ref{lem:composition} and our choice of $\varepsilon_1$ and $\delta_1$.  

This proof is analogous to the proof of DP for {\sc DPMMWfilter} in Appendix~\ref{sec:proof_main2_dp}, and we omit the details here. 
We will assume for now that $|S_r|\geq n/2$ for all  $r\in[t]$ and prove privacy. This happens with probability larger than $1-\delta_1$, hence ensuring the privacy guarantee. 
In all sub-routines, we run Filter$(\cdot)$ in Algorithm~\ref{alg:DPfilter_interactive_main}  to simulate the filtering process so far and get the current set of samples $S_t$.
Lemma~\ref{lem:DPfilter_sensitivity}  allows us to prove privacy of all interactive mechanisms. 
This shows that the two data datasets $S_t$ and $S_t'$ are neighboring, if they are resulting from the identical filtering but starting from two neighboring datasets ${\cal D}_n$ and ${\cal D}'_n$. 
As all four sub-routines are output perturbation mechanisms with appropriately chosen sensitivities, they satisfy the desired ($\varepsilon_1,\delta_1$)-DP guarantees. Further, the probability that $n_t>3/4n$ and $|S_t|\leq n/2$ is less than $\delta_1$ for $n=\tilde\Omega((1/\varepsilon_1)\log(1/\delta_1))$.

\subsubsection{Proof of part 2 of    Lemma~\ref{lem:DPfilter} on accuracy of  {\sc DPfilter}}
\label{sec:proof_DPfilter_part2}
The following theorem analyzing {\sc DPfilter} implies the desired Lemma~\ref{lem:DPfilter} when the good set is $\alpha$-subgaussian good, which follows from \ref{lemma:alpha-good-sample-complexity} and the assumption that $n=\widetilde\Omega(d/\alpha^2)$.

\begin{thm}[Anlaysis of {\sc DPfilter}]\label{thm:DPFilter}
	Let $S$ be an $\alpha$-corrupted sub-Gaussian dataset under Assumption~\ref{asmp:adversary}, where $\alpha\leq c$ for some universal constant $c\in (0,1/2)$. Let $S_{\rm good}$ be $\alpha$-subgaussian good with respect to $\mu\in \reals^d$. Suppose  ${\cal D} = \{ x_{i} \in \bar{x}+[-B/2,B/2]^d\}_{i=1}^{n}$ be the projected dataset where all of the uncorrupted samples are contained in $\bar{x}+[-B/2,B/2]^d$. If $
	n =  \widetilde\Omega\left( 
	 {d^{2}B^3\log(1/\delta)}/{(\varepsilon\alpha) }\right)$, then {\sc DPfilter} terminates after at most $O\left(dB^2 \right)$ iterations and outputs $S_t$ such that with probability $0.9$, we have $|S_t\cap S_{\mathrm{good}}|\geq (1-10\alpha )n$ and
	\begin{eqnarray*}
		\|\mu(S_{t})-\mu\|_2\lesssim \alpha\sqrt{\log 1/\alpha}\;.
	\end{eqnarray*}
\end{thm}

To prove this theorem, 
we use the following lemma to first show that we do not remove too many uncorrupted samples. The upper bound on the accuracy follows immediately from Lemma~\ref{thm:reg} and the stopping criteria of the algorithm.

\begin{lemma}\label{lemma:invariance_pca}
If  
$
	n \gtrsim \frac{B^2d^{3/2}}{\varepsilon_1 \alpha \log1/\alpha}\log(1/\delta)
$, $\lambda_t\geq (C-0.01)\cdot \alpha\log1/\alpha$ and $|S_t\cap S_{\rm good}|\geq (1-10\alpha) n$, then there exists constant $C>0$ such that for each  iteration $t$, with probability $1-O(1/d)$,  we have Eq.~\eqref{eq:reg_pca} holds. If this condition holds, we have
\begin{eqnarray*}
 \E\left|(S_t\setminus S_{t+1})\cap S_{\rm good}\right|\leq \E\left|S_t\setminus S_{t+1}\cap S_{\rm bad}\right|\;.
\end{eqnarray*}

\end{lemma}

We measure the progress by 
by summing the number of clean samples removed 
up to iteration $t$ and the number of remaining corrupted samples, defined  as  $d_{t} \triangleq|(S_{\mathrm{good}}\cap S)\setminus S_{t}|+|S_{t}\setminus (S_{\mathrm{good}}\cap S)|$. Note that $d_1=\alpha n$, and $d_t\geq0$. At each iteration, we have
	\begin{eqnarray*}
		\E [d_{t+1}-d_{t}|d_1, d_2, \cdots, d_{t}] &=& \E\left[|S_{\mathrm{good}}\cap(S_{t}\setminus S_{t+1})|-|S_{\mathrm{bad}}\cap(S_{t}\setminus S_{t+1})|\right] 
		\;\leq\; 0,
	\end{eqnarray*}
	from the  Lemma~\ref{lemma:invariance_pca}.
	Hence, $d_t$ is a non-negative  super-martingale. By optional stopping theorem, at stopping time, we have $\E[d_t]\leq d_1=\alpha n$. By Markov inequality, $d_t$ is less than $10\alpha n$ with probability $0.9$, i.e. $|S_t\cap S_{\mathrm{good}}|\geq (1-10\alpha )n$. 
The desired bound follows from induction and Lemma~\ref{thm:reg}.

Now we bound the number of iterations under the conditions of Lemma~\ref{lemma:progress_conditions}. Let $W_t=|S_t\setminus S_{t-1}|/n$. Since Eq.~\eqref{line:final}, we have
\begin{eqnarray*}
\E[W_t] \geq \frac1n\sum_{i\in {\cal T}_{2\alpha}}\frac{\tau_i}{dB^2} \geq \frac{0.7\|M(S_{t-1})-\mathbf{I}\|_2  }{\alpha dB^2} \geq \frac{0.7C \alpha\log(1/\alpha)}{dB^2}\;.
\end{eqnarray*}

Let $T$ be the stopping time. We know $\sum_{t=1}^TW_t\leq 10\alpha$. By Wald's equation, we have
\begin{eqnarray*}
\E[\sum_{t=1}^TW_t]=\E[\sum_{t=1}^T\E[W_t]]\geq \E[T] \frac{0.7C \alpha\log(1/\alpha)}{dB^2}\;.
\end{eqnarray*}
This means $\E[T]\leq (15dB^2)/(C\log(1/\alpha))$. By Markov inequality we know with probability $0.9$, we have $T=O(dB^2/\log(1/\alpha))$.

\subsubsection{Proof of Lemma~\ref{lemma:invariance_pca}} 
\label{sec:proof_invariance_pca}

The expected number of removed good points and bad points are proportional to the $\sum_{i\in S_{\mathrm{good}}\cap {\cal T}_{2\alpha}}\tau_i$ and $\sum_{i\in S_{\mathrm{bad}}\cap {\cal T}_{2\alpha}}\tau_i$. It suffices to show
	\begin{eqnarray*}
		\sum_{i\in S_{\mathrm{good}}\cap {\cal T}_{2\alpha}}\tau_i \;\;\leq\;\; \sum_{i\in S_{\mathrm{bad}}\cap {\cal T}_{2\alpha}}\tau_i\;.
	\end{eqnarray*}

 Assuming we have $\|M(S_{t-1})-I\|_2\geq C \alpha\log1/\alpha $ for some $C>0$ sufficiently large, it suffices to show
\begin{eqnarray*}
	\frac{1}{n}\sum_{i\in S_{\mathrm{bad}}\cap {\cal T}_{2\alpha}}\tau_i \geq  \frac{1}{1000}\|M(S_{t-1})-{\mathbf I}\|_2\;.
\end{eqnarray*}

First of all, we have
\begin{eqnarray*}
	\frac{1}{n}\sum_{i\in S_{t-1}}\tau_i-1 &=& v_{t}^\top M(S_{t-1})v_t-1\\ 
	&=&  v_{t}^\top \left(M(S_{t-1})-{\mathbf I}\right)v_t\\
\end{eqnarray*}

Lemma~\ref{lem:variance_lower} shows that the magnitude of the largest eigenvalue of $M(S_{t-1})-{\mathbf I}$ is positive since the magnitudes negative eigenvalues are all less than $c\alpha  \log1/\alpha$. So we have
\begin{eqnarray}
\frac{1}{n}\sum_{i\in S_{t-1}}\tau_i-1 &\geq &  \|M(S_{t-1})-\mathbf{I}\|_2-O(\alpha\log1/\alpha) \\
&\geq & 0.9\|M(S_{t-1})-\mathbf{I}\|_2\label{eq:wT}\;, 
\end{eqnarray}
where the first inequality follows from Lemma~\ref{lemma:sample_complexity_singular_value}, and the second inequality follows from our choice of large constant $C$.
The next lemma regularity conditions for $\tau_i$'s for each iteration is satisfied.

\begin{lemma}\label{lemma:progress_conditions}
		If $n \gtrsim  \frac{B^2d^{3/2}}{\varepsilon_1 \alpha \log1/\alpha}\log(1/\delta)$, then there exists a large constant $C>0$ such that, with probability $1-O(1/d)$, we have
		
		\begin{enumerate}
		    \item \begin{eqnarray}
			\frac{1}{n}\sum_{i\in S_{\mathrm{good}}\cap {\cal T}_{2\alpha}\cap S_{t-1}}\tau_i\leq \frac{1}{1000}\left\|M(S_{t-1})-{\mathbf I}\right\|_2\;.\label{eq:reg_pca}
		\end{eqnarray}
		\item For all $i\notin {\cal T}_{2\alpha}$, 
	\begin{eqnarray*}
		\alpha\tau_i \leq  \frac{1}{1000}\|M(S_{t-1})-{\mathbf I}\|_2\;.
	\end{eqnarray*}
	\item 
	\begin{eqnarray*}
		\frac{1}{n}\sum_{i\in S_{\mathrm{good}}\cap S_{t-1}}\left(\tau_i-1\right)\leq \frac{1}{1000}\|M(S_{t-1})-{\mathbf I}\|_2\;.
	\end{eqnarray*}

		\end{enumerate}
		
	\end{lemma}

Thus, by combining with Lemma~\ref{lemma:progress_conditions}, we have
\begin{eqnarray*}
	\frac{1}{n}\sum_{i\in S_{t-1}\cap S_{\mathrm{bad}}}\tau_i	&\geq & 0.8\|M(S_{t-1})-\mathbf{I}\|_2\;. 
\end{eqnarray*}

We now have 
\begin{eqnarray}
	\frac{1}{n}\sum_{i\in S_{\mathrm{bad}}\cap {\cal T}_{2\alpha}}\tau_i &\geq & 0.8\|M(S_{t-1})-\mathbf{I}\|_2-\sum_{i\in S_{\mathrm{bad}}\cap S_{t-1}\setminus {\cal T}_{2\alpha}}\tau_i\nonumber\\
	&\geq & 0.8\|M(S_{t-1})-\mathbf{I}\|_2-\max_{i\in S_{\mathrm{bad}}\cap S_{t-1}\setminus {\cal T}_{2\alpha}}\alpha  \tau_i\nonumber\\
	&\geq & 0.8\|M(S_{t-1})-\mathbf{I}\|_2-\frac{1}{1000}\|M(S_{t-1})-\mathbf{I}\|_2 \label{line:final}\\
	&\geq &\frac{1}{n}\sum_{i\in S_{\mathrm{good}}\cap {\cal T}_{2\alpha}}\tau_i\nonumber\;,
\end{eqnarray}
which completes the proof.

\subsubsection{Proof of Lemma~\ref{lemma:progress_conditions}} 
\label{sec:proof_progress_conditions}

	By our choice of sample complexity $n$, with probability $1-O(1/dB^2)$, we have 
	$\|\mu(S_{t-1})-\mu_t\|_2^2\lesssim \alpha \log1/\alpha$, $v_t^\top \left(M(S_{t-1})-\mathbf{I}\right)v_t \gtrsim  \|M(S_{t-1})-\mathbf{I}\|_2-\alpha\log1/\alpha$ (Lemma~\ref{lemma:sample_complexity_singular_value}), and $\|M(S_{t-1})-\mathbf{I}\|_2\geq C\alpha\log 1/\alpha$ simultaneously hold before stopping. 
	
\begin{lemma}\label{lemma:sample_complexity_singular_value}
 If 
	\begin{eqnarray*}
		n\gtrsim  \frac{d^{3/2}B^2}{\eta\varepsilon_1}\sqrt{2\ln{\frac{1.25}{\delta}}}\log{\frac{1}{\zeta}}\;,
	\end{eqnarray*}
	 then with probability $1-\zeta$, we have
	 \begin{eqnarray*}
 v_t^\top \left(M(S_{t-1})-\mathbf{I}\right)v_t \geq  \|M(S_{t-1})-\mathbf{I}\|_2-2\eta-\frac{2|S_{t-1}|}{n}\|\mu_t-\mu(S_{t-1})\|_2^2
	 \end{eqnarray*}
\end{lemma}

	We first consider the upper bound of the good points.
\begin{align*}
	& \frac{1}{n}\sum_{i\in S_{\mathrm{good}}\cap {\cal T}_{2\alpha}\cap S_{t-1}}\tau_i =\frac{1}{n}\sum_{i\in S_{\mathrm{good}}\cap {\cal T}_{2\alpha}\cap S_{t-1}} \ip{x_i-\mu_t}{v_t}^2\\
	&\overset{(a)}{\leq}  \frac{2}{n}\sum_{i\in S_{\mathrm{good}}\cap {\cal T}_{2\alpha}\cap S_{t-1}} \ip{x_i-\mu}{v_t}^2+\frac{2}{n}| S_{\mathrm{good}}\cap {\cal T}_{2\alpha}\cap S_{t-1}| \ip{\mu-\mu_t}{v_t}^2\notag\\
	&\leq  O(\alpha \log1/\alpha) + \alpha  \left(\|\mu-\mu(S_{t-1})\|_2+\|\mu_t-\mu(S_{t-1})\|_2\right)^2\\
	&\overset{(b)}{\leq} O(\alpha \log1/\alpha)+\alpha  \left(O(\alpha\sqrt{\log1/\alpha})+\sqrt{\alpha \left(\|M(S_{t-1})-{\mathbf I}\|_2+O(\alpha\log1/\alpha)\right)}
+O(\sqrt{\alpha\log1/\alpha})\right)^2\\
	&\leq  O(\alpha \log1/\alpha)+\alpha^2 \|M(S_{t-1})-{\mathbf I}\|_2\\
	&\overset{(c)}{\leq}   \frac{1}{1000}\|M(S_{t-1})-I\|_2
\end{align*}
where the $(a)$ is implied by the fact that for any vector $x, y, z$, we have $(x-y)(x-y)^{\top} \preceq 2(x-z)(x-z)^{\top}+2(y-z)(y-z)^{\top}$, $(b)$ follows from Lemma~\ref{thm:reg} and $c$ follows from our choice of large constant $C$.

Since $|S_{\rm bad}\cap {\cal T}_{2\alpha}|\leq \alpha n$, we know $|S_{\rm good}\cap {\cal T}_{2\alpha}|\geq  \alpha n$, so we have for $i\notin {\cal T}_{2\alpha}$,
	\begin{eqnarray*}
		\alpha\tau_i \leq \frac{\alpha}{|S_{\rm good}\cap {\cal T}_{2\alpha}\cap S_{t-1}|}\sum_{i\in S_{\mathrm{good}}\cap {\cal T}_{2\alpha}\cap S_{t-1}}\tau_i \leq \frac{1}{1000}\|M(S_{t-1})-{\mathbf I}\|_2\;.
	\end{eqnarray*}
	
Since $|S_{\mathrm{good}}\cap S_{t-1}|\geq (1-10\alpha)n$, we have 
	\begin{align}
		&\frac{1}{n}\sum_{i\in S_{\mathrm{good}}\cap S_{t-1}}\tau_i = \frac{1}{n}\sum_{i\in S_{\mathrm{good}}\cap S_{t-1}}\ip{x_i-\mu(S_{t-1})}{v_t}^2\\
		&= \frac{1}{n}\sum_{i\in S_{\mathrm{good}}\cap S_{t-1}}\ip{x_i-\mu(S_{\mathrm{good}}\cap S_{t-1})}{v_t}^2+\frac{|S_{\mathrm{good}}\cap S_{t-1}|}{n}\ip{\mu(S_{\mathrm{good}}\cap S_{t-1})-\mu(S_{t-1})}{v_t}^2\\
		&\overset{(a)}{\leq}  c\alpha \log1/\alpha +1+ \|\mu(S_{\mathrm{good}}\cap S_{t-1})-\mu(S_{t-1})\|_2^2\\
		&\leq   c\alpha  \log1/\alpha +1+ \left(\|\mu(S_{\mathrm{good}}\cap S_{t-1})-\mu\|_2+\|\mu-\mu(S_{t-1})\|_2\right)^2\\
		&\overset{(b)}{\leq}  c\alpha  \log1/\alpha +1 +\alpha \|M(S_{t-1})-{\mathbf I}\|_2+O(\alpha\log1/\alpha)\label{eq:w0}\\
		&\overset{(c)}{\leq} \frac{1}{1000}\|M(S_{t-1})-{\mathbf I}\|_2\;,
	\end{align}
where $(a)$ follows from Lemma~\ref{lem:variance_lower}, and $(b)$ follows from Lemma~\ref{thm:reg}, and $(c)$ follows from our choice of large constant $C$.

\subsubsection{Proof of Lemma~\ref{lemma:sample_complexity_singular_value}} 
\begin{proof}
We have following identity.
	\begin{eqnarray*}
		&&\frac{1}{n}\sum_{i\in S_{t-1}}(x_i-\mu_t)(x_i-\mu_t)^\top \\
		&= & \frac{1}{n}\sum_{i\in S_{t-1}}(x_i-\mu(S_{t-1}))(x_i-\mu(S_{t-1}))^\top+\frac{|S_{t-1}|}{n}(\mu(S_{t-1})-\mu_t)(\mu(S_{t-1})-\mu_t)^\top\;.
	\end{eqnarray*}

So we have, 
\begin{eqnarray*}
	&& v_t^\top \left(M(S_{t-1}) -\mathbf{I}\right)v_t \\
\geq && v_t^\top \left(\frac{1}{n}\sum_{i\in S_{t-1}}(x_i-\mu_t)(x_i-\mu_t)^\top -\mathbf{I}\right)v_t-\frac{|S_{t-1}|}{n}\|\mu_t-\mu(S_{t-1})\|_2^2\\
\geq &&  \|M(S_{t-1})-\mathbf{I}\|_2-2\eta-\frac{2|S_{t-1}|}{n}\|\mu_t-\mu(S_{t-1})\|_2^2
\end{eqnarray*}
where the last inequality follows from Lemma~\ref{lem:variance_lower}, which shows that the magnitude of the largest eigenvalue of $M(S_{t-1})-{\mathbf I}$ must be positive.
\end{proof}



\section{PRIME: efficient algorithm for private and robust mean estimation}
\label{sec:prime_appendix}

We provide our main algorithms, Algorithm~\ref{alg:prime} and Algorithm~\ref{alg:DPMMWfilter}, in Appendix~\ref{sec:prime_algorithm} and the corresponding proof in Appendix~\ref{sec:proof_main2}. We provide our novel {\sc DPthreshold} and its anlysis in Appendix~\ref{sec:threshold_algorithm}.

We define $S_{\rm good}$ as the original set of $n$ clean samples (as defined in Assumption~\ref{asmp:adversary} and \ref{asmp:adversary2})  and $S_{\rm bad}$ as the set of corrupted samples that replace $\alpha n$ of the clean samples. The (rescaled) covariance is denoted by $M(S^{(s)}) \triangleq (1/n) \sum_{i\in S^{(s)}} (x_i-\mu(S^{(s)}))(x_i-\mu(S^{(s)}))^\top $, where 
$\mu(S^{(s)})\triangleq (1/|S^{(s)}|)\sum_{i\in S^{(s)}} x_i$ denotes the   mean.

\subsection{PRIvate and robust Mean Estimation (PRIME)}
\label{sec:prime_algorithm}

\begin{algorithm2e}[H]
   \caption{PRIvate and robust Mean Estimation (PRIME)}
   \label{alg:prime} 
   	\DontPrintSemicolon 
	\KwIn{$ S  =  \{ x_{i} \in {\mathbb R}^d\}_{i=1}^{n} $, adversarial fraction $\alpha \in(0,1/2)$, 
	number of iterations $T_1=O(\log d),T_2 = O(\log d)$, target privacy $(\varepsilon,\delta)$   }
	\SetKwProg{Fn}{}{:}{}
	{
	$(\bar x, B) \leftarrow q_{\rm range} (\{ x_i\}_{i=1}^n,0.01\varepsilon,0.01\delta)$  \hfill [Algorithm~\ref{alg:DPrange} in Appendix~\ref{sec:proof_DPrange}]\\
	Clip the data points: $\tilde x_i \leftarrow {\cal P}_{\bar{x}+[-B/2,B/2]^d}(x_i)$, for all $i\in[n]$\;
	$\hat\mu \leftarrow$ \text{ \sc DPMMWfilter}($ \{\tilde x_i\}_{i=1}^n,\alpha,T_1,T_2,0.99\varepsilon,0.99\delta$) \hfill [Algorithm~\ref{alg:DPMMWfilter}]\label{line:dpmmwfilter}\\
	\KwOut{$\hat\mu$}
	} 
\end{algorithm2e}

\begin{algorithm2e}[H]
  \caption{Differentially private filtering with matrix multiplicative weights ({\sc DPMMWfilter}) } 
  \label{alg:DPMMWfilter} 
  \DontPrintSemicolon 
  \KwIn{$ S  =  \{ x_{i} \in \bar{x}+[-B/2,B/2]^d\}_{i=1}^{n} $, $\alpha \in(0,1/2)$, $T_1 = O(\log(B\sqrt{d})), T_2 = O(\log d)$,   privacy $(\varepsilon,\delta)$  
  } 
  \SetKwProg{Fn}{}{:}{}
  { 
    Initialize $S^{(1)} \leftarrow [n]$,  $\varepsilon_1 \leftarrow \varepsilon/(4T_1)$, $\delta_1\leftarrow \delta/(4T_1) $, $\varepsilon_2 \leftarrow \min\{0.9,\varepsilon\}/(4\sqrt{10T_1 T_2\log(4/\delta)})$, $\delta_2\leftarrow \delta/(20T_1T_2)$, a large enough constant $C
    >0$\; 
    {\bf if} $n < (4/\varepsilon_1)\log(1/(2\delta_1))$  {\bf then Output:} $\emptyset$ \;
	\For{{\rm epoch} $s=1,2,\ldots, T_1$}
	{
      $ \textcolor{black}{\lambda^{(s)}} \leftarrow \| M(S^{(s)})- {\mathbf I}\|_2 + {\rm Lap}(2B^2d/(n\varepsilon_1))$ \;
      $ n^{(s)} \leftarrow |S^{(s)}| + {\rm Lap}(1/\varepsilon_1)$ \;
      {\bf if }$n^{(s)}\leq 3n/4$ {\bf then} {\rm {\bf Output:} $\emptyset$}\;
      \If
      {$ \lambda^{(s)}\leq C\,\alpha\log(1/\alpha)$ }{
        \KwOut
        {
           $\textcolor{black}{\mu^{(s)}} \leftarrow   (1/|S^{(s)}|) \big( \sum_{i\in S^{(s)} } x_i\big)  + \cN(0,(2B\sqrt{2d\log(1.25/\delta_1)}/({n\, \varepsilon_1 }))^2{\mathbf I}_{d\times d}) $
        }
      } 
      $\alpha^{(s)} \leftarrow 1/(100(0.1/C+1.01)\lambda^{(s)})$\;
      $S^{(s)}_1 \leftarrow  S^{(s)}$\;
	  \For{$t=1,2,\ldots,T_2$}
	  {
	     \label{line:mmw1} $\textcolor{black}{\lambda_t^{(s)}} \leftarrow \| M(S_t^{(s)})- {\mathbf I}\|_2  + {\rm Lap}(2B^2 d/(n\varepsilon_2))$\;
	     \eIf{$\lambda_t^{(s)}\leq 0.5 \lambda_0^{(s)}$}
	     { 
	       {\rm terminate epoch}
	     }{
	        $\textcolor{black}{\Sigma_t^{(s)}} \leftarrow  M(S_{t}^{(s)}) + \cN(0,(4B^2d\sqrt{2\log(1.25/\delta_2)}/(n\varepsilon_2))^2{\mathbf I_{d^2\times d^2}}) $\label{line:inputprivate}\;
	        $U_t^{(s)} \leftarrow (1/\Tr(\exp( \alpha^{(s)}\sum_{r=1}^{t}(\Sigma_r^{(s)}-{\mathbf I} ))))\exp( \alpha^{(s)}\sum_{r=1}^{t}(\Sigma_r^{(s)}-{\mathbf I})) $ \label{line:mmw}\;
	        $\textcolor{black}{\psi_t^{(s)}} \leftarrow \ip{M(S_t^{(s)})-{\mathbf I}}{U_t{^{(s)}}} + {\rm Lap}(2B^2 d /( n \varepsilon_2))$\;
	        \eIf{$\psi_t^{(s)}\leq (1/5.5)\lambda_t^{(s)}$}
	        {
	          $S_{t+1}^{(s)}\leftarrow S_t^{(s)}$ 
	        }
	        {
              $Z_t^{(s)} \leftarrow {\rm Unif}([0,1])$\;
              $\textcolor{black}{\mu_t^{(s)}}  \leftarrow  (1/|S_{ t }^{(s)}|) \big( \sum_{i\in S_{t}} x_i\big)  + \cN(0,(2B\sqrt{2d\log(1.25/\delta_2)}/({n\, \varepsilon_2 }){\mathbf I}_{d\times d})^2) $\;
              $\textcolor{black}{\rho_t^{(s)}} \leftarrow \text{{\sc DPthreshold}}(\mu_t^{(s)},U_t^{(s)}, \alpha,\varepsilon_2,\delta_2,S^{(s)}_t)$\hfill [Algorithm~\ref{alg:1Dfilter}]\label{line:1Dfilter}\;
              $S_{t+1}^{(s)} \leftarrow S_{t}^{(s)} \setminus$  $\{ i  \,|\, i \in {\cal T}_{2\alpha}$ for   $\{\tau_j = (x_j-\mu_t^{(s)})^\top U_t^{(s)} (x_j-\mu_t^{(s)})\}_{j\in S_{t}^{(s)}} $ and $\tau_i \geq \rho_t^{(s)} \,Z_t^{(s)} \}$,  where ${\cal T}_{2\alpha}$ is defined in Definition~\ref{def:topset}.
              \label{line:newfilter}
            } 
	        \label{line:mmw2}
	      } 
	    } 
        $S^{(s+1)} \leftarrow S^{(s)}_t$\;
	  } 
	  \KwOut{$\mu^{(T_1)}$} 
    }
\end{algorithm2e}

\subsection{Algorithm and analysis of {\sc DPthreshold}}
\label{sec:threshold_algorithm}

\begin{algorithm2e}[H]
   \caption{Differentially private estimation of the threshold  ({\sc DPthreshold}) } 
   \label{alg:1Dfilter} 
   	\DontPrintSemicolon 
	\KwIn{
	$\mu$, $U$, $\alpha \in(0,1/2)$, target privacy $(\varepsilon,\delta)$, $S=\{x_i\in\bar{x}+[-B/2,B/2]^d \}$  } 
	\SetKwProg{Fn}{}{:}{}
	{ 
	    Set $ \tau_i \leftarrow (x_i - \mu)^\top U (x_i-\mu)$ for all $i\in S$\; 
	    Set $\tilde\psi \leftarrow (1/n)\sum_{i\in S} (\tau_i-1) + {\rm Lap}(2B^2d/n\varepsilon))$\;
	    Compute a histogram over geometrically sized bins $I_1=[1/4,1/2),I_2=[1/2,1),\ldots,I_{2+\log(B^2d)}=[2^{\log(B^2d)-1},2^{\log(B^2d)}]$
	    $$h_j \leftarrow \frac{1}{n}\cdot |\{i\in S\,|\, \tau_i \in [2^{-3+j},2^{-2+j})\}| \;,\;\;\;\;\text{ for all } j= 1, \ldots, 2+\log(B^2d)$$\; 
	    Compute a privatized histogram $\tilde{h}_j \leftarrow h_j + {\cal N}(0,(4\sqrt{2\log(1.25/\delta)}/(n\varepsilon))^2)$, 
	    for all  $j\in[2+\log(B^2d)]$\;
	    Set $\tilde\tau_j\leftarrow 2^{-3+j}$, for all  $j\in[2+\log(B^2d)]$\;
	    Find the largest $\ell\in[2+\log(B^2d)]$  satisfying $ \sum_{j \geq \ell} (\tilde\tau_j-\tilde\tau_\ell)\, \tilde h_j \geq 0.31 \tilde\psi $\;
	    \KwOut{$\rho=\tilde\tau_\ell $} 
    }
\end{algorithm2e} 

\begin{lemma}[{\sc DPthreshold}: picking threshold privately]\label{lemma:1dfilter}
Algorithm {\sc DPthreshold}($\mu,U,\alpha,\varepsilon,\delta,S$) running on a dataset $\{\tau_i =  (x_i-\mu)^\top U (x_i-\mu) \}_{i\in S}$ is 
$(\varepsilon,\delta)$-DP. 
Define $\psi \triangleq \frac{1}{n}\sum_{i\in S}(\tau_i-1)$. 
If $\tau_i$'s satisfy 
\begin{eqnarray*}
\frac{1}{n} \sum_{i\in S_{\rm good}\cap{\cal T}_{2\alpha}\cap S}\tau_i &\leq & \psi/1000\\
\frac{1}{n} \sum_{i\in S_{\rm good}\cap S} (\tau_i-1) &\leq & \psi/1000\;,
\end{eqnarray*}
 and $n\geq  \widetilde\Omega\left(\frac{B^2d\sqrt{\log(1/\delta)}}{\varepsilon \alpha} \right)$, 
then {\sc DPthreshold}
 outputs a threshold $\rho$ such that with probability $1-O(1/\log^3 d )$, 
\begin{eqnarray}
\frac{1}{n}\sum_{\tau_i<\rho }(\tau_i - 1) \;\; \le \;\; 0.75 \psi  \; \text{ and }\label{eqn:1dprogress}
\end{eqnarray}
\begin{eqnarray}
2(\sum_{i\in S_{\rm good}\cap {\cal T}_{2\alpha}} \textbf{1}\{\tau_i\le \rho \}\frac{\tau_i}{\rho}+ \textbf{1}\{\tau_i > \rho \}   ) \le \sum_{i\in S_{\rm bad}\cap {\cal T}_{2\alpha}} \textbf{1}\{\tau_i\le \rho \}\frac{\tau_i}{\rho} + \textbf{1}\{\tau_i > \rho \} \;.\label{eqn:good-bad-ratio}
\end{eqnarray}
\end{lemma}

\subsection{Proof of Lemma~\ref{lemma:1dfilter}}

\textbf{1. Threshold $\rho$ sufficiently reduces the total score.} 

Let $\rho$ be the threshold picked by the algorithm. Let $\hat{\tau}_i$ denote the minimum value of the interval of  the bin that $\tau_i$ belongs to. It holds that
\begin{eqnarray*}
&& \frac{1}{n}\sum_{\tau_i\ge \rho, i\in[n]}({\tau}_i - \rho) 
 \ge \frac{1}{n}\sum_{\hat\tau_i\ge \rho, i\in [n]}({\hat\tau}_i - \rho) \nonumber\\
 = &&\sum_{\tilde\tau_j\ge\rho, j\in[2+\log(B^2d)]}(\tilde\tau_j-\rho){h}_j\nonumber\\
\overset{(a)}{\ge} &&\sum_{\tilde\tau_j\ge\rho, j\in[2+\log(B^2d)]}(\tilde\tau_j-\rho)\tilde{h}_j- O\left(\log(B^2d)\cdot B^2d\cdot \frac{\sqrt{\log(\log(B^2d)\log d)\log(1/\delta)}}{\eps n}\right)\nonumber\\
\overset{(b)}{\ge} &&0.31 \tilde\psi -  \tilde{O}(\frac{B^2d}{\eps n})\nonumber\\
\overset{(c)}{\geq}  &&0.3 \psi -  \tilde{O}(\frac{B^2d}{\eps n})\nonumber\;,
\end{eqnarray*}
where $(a)$ holds due to the accuracy of the private histogram (Lemma~\ref{lem:hist}), $(b)$ holds by the definition of $\rho$ in our algorithm, and $(c)$ holds due to the accuracy of $\tilde{\psi}$. This implies if $\rho <1$, then $\frac{1}{n}\sum_{\tau_i < \rho}({\tau}_i - 1)$ is negative and if $\rho\geq 1$, then
\begin{eqnarray*}
\frac{1}{n}\sum_{\tau_i < \rho}({\tau}_i - 1) = \psi-\frac{1}{n}\sum_{\tau_i \geq \rho}({\tau}_i - 1) \leq \psi-\frac{1}{n}\sum_{\tau_i \geq \rho}({\tau}_i - \rho) \le 0.7\psi +  \tilde{O}({B^2d/\eps n}).
\end{eqnarray*}

By Lemma~\ref{def:cond-good-data}, it holds that 
\begin{eqnarray*}
\frac{1}{n}\sum_{i \in S\setminus {\cal{T}}_{2\alpha}}({\tau}_i - 1)
 &= &\psi - \frac{1}{n}\sum_{i \in S_{\rm good}\cap{\cal T}_{2\alpha}}({\tau}_i - 1) - \frac{1}{n}\sum_{i \in S_{\rm bad}\cap{\cal T}_{2\alpha}}({\tau}_i - 1)\\
 &\le & \psi - \frac{1}{n}\sum_{i \in S_{\rm bad}\cap{\cal T}_{2\alpha}}({\tau}_i - 1)\\
&\le & (2/1000) \psi
\end{eqnarray*}
And we conclude that 
\begin{eqnarray*}
\frac{1}{n}\sum_{\tau_i < \rho \text{ or } i\notin {\cal T}_{2\alpha}}({\tau}_i - 1)  \le 0.71\psi +  \tilde{O}(B^2d/\eps n) \leq 0.75 \psi
\end{eqnarray*}

\textbf{2. Threshold $\rho$ removes more bad data points than good data points.} 

Define $C_2$ to be the threshold such that $\frac{1}{n}\sum_{\tau_i>C_2}(\tau_i-C_2)= (2/3) \psi$. Suppose $2^b\le C_2 \le 2^{b+1}$, $\frac{1}{n}\sum_{\hat\tau_i\ge 2^{b-1}} (\hat\tau_i-2^{b-1})\ge (1/3) \psi $ because $\forall \tau_i\ge C_2$, $(\hat\tau_i-2^{b-1}) \ge \frac{1}{2}(\tau_i-C_2) $. Trivially $C_2\ge 1$ due to the fact that $\frac{1}{n}\sum_{\tau_i\ge 1}\tau_i-1 \ge \psi $. Then we have the threshold picked by the algorithm $\rho\ge 2^{b-1}$, which implies $\rho \ge \frac{1}{4}C_2$. 
Suppose $\rho<C_2$, since $\rho \ge \frac{1}{4}C_2$, we have
\begin{eqnarray*}
(\sum_{i\in S_{\rm bad}\cap {\cal T}_{2\alpha}, \tau_i<\rho} \tau_i + \sum_{i\in S_{\rm bad}\cap {\cal T}_{2\alpha}, \tau_i\geq\rho} \rho) &\ge &\frac{1}{4}(\sum_{i\in S_{\rm bad}\cap {\cal T}_{2\alpha}, \tau_i<C_2} \tau_i + \sum_{i\in S_{\rm bad}\cap {\cal T}_{2\alpha}, \tau_i\geq C_2} C_2)\\
&\overset{(a)}{\ge} &\frac{10}{4}(\sum_{i\in S_{\rm good}\cap {\cal T}_{2\alpha}, \tau_i<C_2} \tau_i + \sum_{i\in S_{\rm good}\cap {\cal T}_{2\alpha}, \tau_i\geq C_2} C_2)\\
&\overset{(b)}{\ge} &\frac{10}{4}(\sum_{i\in S_{\rm good}\cap {\cal T}_{2\alpha}, \tau_i<\rho} \tau_i + \sum_{i\in S_{\rm good}\cap {\cal T}_{2\alpha}, \tau_i>=\rho} \rho),
\end{eqnarray*}
where (a) holds by Lemma~\ref{lem:bad-good-ratio}, and (b) holds since $\rho\le C_2$.
If $\rho\ge C_2$, the statement of the Lemma~\ref{lem:bad-good-ratio} directly implies Equation~\eqref{eqn:good-bad-ratio}.

\begin{lemma}\label{def:cond-good-data}[Conditions for $\tau_i$'s] Suppose \begin{eqnarray*}
\frac{1}{n} \sum_{i\in S_{\rm good}\cap S} (\tau_i-1) \le \psi/1000\\
\frac{1}{n} \sum_{i\in S_{\rm good}\cap{\cal T}_{2\alpha}} \tau_i \le \psi/1000
\end{eqnarray*}

then, we have
\begin{eqnarray*}
	\alpha\tau_{2\alpha n} &\le & \psi/1000\\
\frac{1}{n} \sum_{i\in S_{\rm bad}\cap{\cal T}_{2\alpha}} (\tau_i-1)&\ge & (998/1000)\psi
\end{eqnarray*}
\end{lemma}

\begin{proof}
Since $|S_{\rm good}\cap {\cal T}_{2\alpha}|\ge \alpha n$, it holds
\begin{eqnarray*}
\alpha\tau_{2\alpha n} \le \psi/1000.
\end{eqnarray*}

\begin{eqnarray*}
\frac{1}{n} \sum_{i\in S_{\rm bad}\cap{\cal T}_{2\alpha}} (\tau_i-1)&
= &\frac{1}{n} \sum_{i\in S_{\rm bad}\cap S} (\tau_i-1) - \frac{1}{n} \sum_{i\in S_{\rm bad}\cap S\backslash {\cal T}_{2\alpha}} (\tau_i-1)\\
&\ge &(999/1000) \psi - \frac{1}{n} \sum_{i\in S_{\rm bad}\cap S\backslash {\cal T}_{2\alpha}} (\tau_i-1)\\
&\ge &(999/1000) \psi - (1/1000)\psi\\
&=& (998/1000) \psi
\end{eqnarray*}
\end{proof}
\begin{lemma}\label{lem:bad-good-ratio}
Assuming that the conditions in Lemma~\ref{def:cond-good-data} holds, and for any $C$ such that
\begin{eqnarray*}
\frac{1}{n}\sum_{i\in S, \tau_i < C }(\tau_i-1) + \frac{1}{n}\sum_{i\in S, \tau_i \ge C}(C-1) \ge (1/3) \psi\;,
\end{eqnarray*}
we have
\begin{eqnarray*}
\sum_{i\in S_{\rm bad}\cap{\cal T}_{2\alpha}, \tau_i < C} \tau_i + \sum_{i\in S_{\rm bad}\cap{\cal T}_{2\alpha}, \tau_i \ge C} C \ge {10}( \sum_{i\in S_{\rm good}\cap{\cal T}_{2\alpha}, \tau_i < C} \tau_i + \sum_{i\in S_{\rm good}\cap{\cal T}_{2\alpha}, \tau_i \ge C} C )
\end{eqnarray*}
\end{lemma}
\begin{proof}
First we show an upper bound on $S_{\rm good}\cap{\cal T}_{2\alpha}$:

\begin{eqnarray*}
\frac{1}{n}\sum_{i\in S_{\rm good}\cap{\cal T}_{2\alpha}, \tau_i < C} \tau_i + \frac{1}{n} \sum_{i\in S_{\rm good}\cap{\cal T}_{2\alpha}, \tau_i \ge C} C  \le \frac{1}{n}\sum_{i\in S_{\rm good}\cap{\cal T}_{2\alpha}} \tau_i \le \psi/1000.
\end{eqnarray*}

Then we show an lower bound on $S_{\rm bad}\cap {\cal T}_{2\alpha}$:
\begin{eqnarray*}
&&\frac{1}{n}\sum_{i\in S_{\rm bad}\cap S, \tau_i<C} (\tau_i-1) + \frac{1}{n}\sum_{i\in S_{\rm bad}\cap S, \tau_i>C}(C-1)\\
=&& \frac{1}{n}\sum_{i\in S, \tau_i<C} (\tau_i-1) + \frac{1}{n}\sum_{i\in S, \tau_i\ge C} (C-1)\\
&&- (\frac{1}{n}\sum_{i\in S_{\rm good}\cap S, \tau_i<C} (\tau_i-1) + \frac{1}{n}\sum_{i\in S_{\rm good}\cap S, \tau_i\ge C} (C-1))\\ 
\ge &&(1/3-1/1000)\psi\;.
\end{eqnarray*}
We have
\begin{align*}
&\frac{1}{n}\sum_{i\in S_{\rm bad}\cap{\cal T}_{2\alpha}, \tau_i<C} \tau_i + \frac{1}{n}\sum_{i\in S_{\rm bad}\cap{\cal T}_{2\alpha}, \tau_i>C}C 
\ge \frac{1}{n}\sum_{i\in S_{\rm bad}\cap{\cal T}_{2\alpha}, \tau_i<C} (\tau_i-1) + \frac{1}{n}\sum_{i\in S_{\rm bad}\cap{\cal T}_{2\alpha}, \tau_i>C}(C-1)\\
&= \frac{1}{n}\sum_{i\in S_{\rm bad}\cap S, \tau_i<\rho} (\tau_i-1) + \frac{1}{n}\sum_{i\in S_{\rm bad}\cap S, \tau_i>C}(C-1)\\
&- \left(\frac{1}{n}\sum_{i\in S_{\rm bad}\cap S\backslash{\cal T}_{2\alpha}, \tau_i<C} (\tau_i-1) + \frac{1}{n}\sum_{i\in S_{\rm bad}\cap S\backslash{\cal T}_{2\alpha}, \tau_i>C}(C-1)\right)\\
&\ge (1/3-1/1000)\psi - \alpha\tau_{2\alpha n}\\
&\ge (1/3-2/1000)\psi
\end{align*}
Combing the lower bound and the upper bound yields the desired statement
\end{proof}

\newpage
\section{The analysis of PRIME and the proof of Theorem~\ref{thm:main2}}

\label{sec:proof_main2}

\subsection{Proof of part 1 of Theorem~\ref{thm:main2} on differential privacy}
\label{sec:proof_main2_dp}

Let $(\varepsilon_0,\delta_0)$ be the end-to-end target privacy guarantee. 
The ($0.01\varepsilon_0,0.01\delta_0$)-DP guarantee of $q_{\rm range}$ follows from  Lemma~\ref{lem:DPrange}. We are left to show that 
{\sc DPMMWfilter} in Algorithm~\ref{alg:DPMMWfilter} satisfy $(0.99\varepsilon_0,0.99\delta_0)$-DP. 
To this end, we explicitly write out how many times we access the database and how much privacy is lost each time in an interactive version of {\sc DPMMWfilter} in Algorithm~\ref{alg:MMWfiltering_interactive}, which performs the same operations as {\sc DPMMWfilter}.

In order to apply Lemma~\ref{lem:composition}, we cap $\varepsilon$ at 0.9 in initializing 
$\varepsilon_2$. 
We call  $q_{\rm spectral}$ and $q_{\rm size}$ $T_1$ times, each with $(\varepsilon_1,\delta_1)$ guarantee. In total this accounts for $(0.5\varepsilon,0.5\delta)$ privacy loss. 
The rest of the mechanisms are called $5T_1T_2$ times ($q_{\rm spectral}(\cdot)$ and $q_{\rm MMW}(\cdot)$ each call two DP mechanisms internally), each with $(\varepsilon_2,\delta_2)$ guarantee. In total this accounts for $(0.5\varepsilon,0.5\delta)$ privacy loss. 
Altogether, this is within the privacy budget of $(\varepsilon=0.99\varepsilon_0,\delta=0.99\delta_0)$.

We are left to show privacy of $q_{\rm spectral}$, $q_{\rm MMW}$, and $q_{\rm 1Dfilter}$, and $q_{\rm size}$ in Algorithm~\ref{alg:DPmechanisms}. 
We will assume for now that $|S^{(\ell)}_r|\geq n/2$ for all $\ell\in[T_1]$ and $r\in[T_2]$ and prove privacy. We show in the end that this happens with probability larger than $1-\delta_1$. In all sub-routines, we run Filter$(\cdot)$ in Algorithm~\ref{alg:DPmechanisms}  to simulate the filtering process so far and get the current set of samples $S^{(s)}_{t_s}$.
The following main technical lemma allows us to prove privacy of all interactive mechanisms. This is a counterpart of  Lemma~\ref{lem:DPfilter_sensitivity} used for {\sc DPfilter}.  
We provide a proof in Appendix~\ref{sec:proof_sensitivity}. 

 \begin{lemma}
     \label{lem:DPMMWfilter_sensitivity}
     Let $S({\cal D}_n)\subseteq {\cal D}_n$ denote the output of the simulated filtering process ${\rm Filter}(\cdot)$ on ${\cal D}_n$ for a given set  of parameters $(\{\{\Psi_{r}^{(\ell)}\}_{r\in[t_\ell]}\}_{\ell\in[s]},\{(\mu^{(\ell)},\lambda^{(\ell)})\}_{\ell\in[s]})$ in Algorithm~\ref{alg:DPmechanisms}.  
     Then we have 
     $d_\triangle( S({\cal D}_n), S({\cal D}')_n ) \leq d_\triangle({\cal D}_n,{\cal D'}_n)$,
     where $d_\triangle({\cal D},{\cal D}') \triangleq \max\{|{\cal D}\setminus {\cal D'}|,|{\cal D}' \setminus {\cal D}|\} $. 
 \end{lemma}
 This is a powerful tool for designing private mechanisms, as it guarantees that we can safely simulate the filtering process with privatized parameters and preserve the neighborhood of the dataset; 
 if ${\cal D}_n\sim {\cal D}'_n$ are neighboring (i.e., $d_\Delta({\cal D}_n,{\cal D}_n')\leq 1$) then so are the filtered pair $S({\cal D}_n)$ and $S({\cal D}_n') $ (i.e., $d_\Delta(S({\cal D}_n),S({\cal D}_n'))  \leq  1 $). 
 Note that  in  all the interactive mechanisms in Algorithm~\ref{alg:DPmechanisms}, the noise we need to add is proportional to the set sensitivity of  Filter$(\cdot)$ defined as $\Delta_{\rm set} \triangleq \max_{{\cal D}_n\sim {\cal D}'_n} d_\Delta(S({\cal D}_n),S({\cal D}_n'))$. If the repeated application of the Filter$(\cdot)$ is not a contraction in $d_\Delta(\cdot,\cdot)$,
 this results in a sensitivity blow-up. Fortunately, the above lemma ensures contraction of the filtering, proving that $\Delta_{\rm set}=1$. Hence, it is sufficient for us to prove privacy for two neighboring filtered sets $S\sim S'$ (as opposed to proving privacy for two neighboring original datasets before filtering ${\cal D}_n \sim {\cal D}_n'$).

In $q_{\rm spectral}$, $\lambda$ satisfy  $(\varepsilon,0)$-DP as the $L_1$ sensitivity is $\Delta_1 = (1/n)B^2d$ (Definition~\ref{def:output}) and we add ${\rm Lap}(\Delta_1/\varepsilon)$. 
The release of $\mu$ also satisfy 
$(\varepsilon,\delta)$-DP as the $L_2$ sensitivity is $\Delta_2=2B\sqrt{d}/n$, assuming $|S|\geq n/2$ as ensured by the stopping criteria, and we add ${\cal N}(0,\Delta_2(2\log(1.25/\delta))/\varepsilon)^2{\mathbf I})$.
Note that in the outer loop call of $q_{\rm spectral}$, we only release $\mu$ once in the end, and hence we count $q_{\rm spectral}$ as one access. On the other hand, in the inner loop, we use both $\mu$ and $\lambda$ from $q_{\rm spectral}$ so we count it as two accesses. 

In $q_{\rm size}$, the returned set size 
$(\varepsilon,0)$-DP as the $L_1$ sensitivity is $\Delta_1 = 1$ and we add ${\rm Lap}(\Delta_1/\varepsilon)$. One caveat is that 
we need to ensure that the stopping criteria of checking $n^{(s)}>3n/4$ ensures that $|S^{(s)}_t|>n/2$ with probability at least $1-\delta_1$. This guarantees that the rest of the private mechanisms can assume $|S^{(s)}_t|>n/2$ in analyzing the sensitivity. 
Since Laplace distribution follows $f(z)=(\varepsilon/2)e^{-\varepsilon |z|}$, we have 
${\mathbb P}(n^{(s)}>3n/4\text{ and }|S^{(s)}_t|<n/2)\leq (1/2)e^{-n\varepsilon/4}$. Hence, the desired privacy is ensured for $(1/2)e^{-n\varepsilon/4} \leq \delta_1$ 
(i.e., $n\geq(4/\varepsilon_1)\log(1/(2\delta_1))$).

In $q_{\rm MMW}$, $\Sigma$ is 
$(\varepsilon,\delta)$-DP as the $L_2$ sensitivity is $\Delta_2 = B^2d/n$, and we add ${\cal N}(0,\Delta_2(2\log(1.25/\delta))/\varepsilon)^2{\mathbf I})$. $\psi$ is $(\varepsilon,0)$-DP as the $L_1$ sensitivity is $\Delta_1 = 2B^2d/n$ and we add ${\rm Lap}(\Delta_1/\varepsilon)$. This is made formal in the following theorem with a proof. in Appendix~\ref{sec:DPPCA_proof}. 
This algorithm is identical to the MOD-SULQ algorithm introduced in \cite{blum2005practical} and analyzed in \cite[Theorem  5]{PPCA}, up to the choice of the noise variance. 
But a tighter analysis improves over the  MOD-SULQ analysis from \cite{PPCA} 
by a factor of $d$ in the variance of added Gaussian noise as noted in \cite{dwork2014analyze}.   

\begin{lemma}[Differentially Private PCA]
    \label{lem:DPPCA}
    Consider a dataset $\{x_i\in{\mathbb R}^d\}_{i=1}^n$. 
    If $\|x_i\|_2 \leq 1 $ for all $i\in[n]$, the following privatized second moment matrix satisfies  $(\varepsilon,\delta)$-differential privacy: 
    \begin{eqnarray*}
        \frac{1}{n}\sum_{i=1}^n x_ix_i^\top + Z\;,
    \end{eqnarray*}
    with $Z_{i,j}\sim {\cal N}(0,( \,(1/(n\varepsilon))\sqrt{2\log(1.25/\delta) }\, )^2 )$ for $i \geq j$ and $Z_{i,j}=Z_{j,i}$ for $i<j$. 
\end{lemma}

In $q_{\rm 1Dfilter}$, the $(\varepsilon,\delta)$ differential privacy follows from that  of {\sc DPthreshold} proved in Lemma~\ref{lemma:1dfilter}.

\subsubsection{Proof of  Lemma~\ref{lem:DPPCA}}
\label{sec:DPPCA_proof}


Consider neighboring two databases ${\cal D}=\{x_i\}_{i=1}^n$ and $\tilde{\cal D}={\cal D}\cup \{\tilde x_n\} \setminus \{x_n\}$, and let $A=(1/n)\sum_{x_i\in{\cal D}} x_ix_i^\top$ and $\tilde{A}=(1/n)\sum_{x_i\in\tilde{\cal D}} x_ix_i^\top$. Let $B$ and $\tilde{B}$ be the Gaussian noise matrix with $\beta^2$ as variance. Let $G=A+B$ and $\tilde{G}=\tilde{A}+\tilde{B}$. At point $H$, we have
\begin{eqnarray*}
	\ell_{D,\tilde D}\;=\;\log\frac{f_{G}(H)}{f_{\tilde{G}}(H)} &\;=\;&\sum_{1 \leq i \leq j \leq d}\left(-\frac{1}{2 \beta^{2}}\left(H_{i j}-A_{i j}\right)^{2}+\frac{1}{2 \beta^{2}}\left(H_{i j}-\hat{A}_{i j}\right)^{2}\right)\\
	&\;=\;& \frac{1}{2 \beta^{2}} \sum_{1 \leq i \leq j \leq d}\left(\frac{2}{n}\left(H_{i j}-A_{i j}\right)\left(x_{n, i}x_{n, j}-\hat{x}_{n, i} \hat{x}_{n, j}\right)+\frac{1}{n^{2}}\left(\hat{x}_{n, i} \hat{x}_{n, j}-x_{n, i} x_{n, j}\right)^{2}\right)\;.
\end{eqnarray*}
Since $\|x_n\|_2\leq 1$ and $\|\tilde{x}_n\|_2\leq 1$, we have $\sum_{1 \leq i \leq j \leq d}\left(\hat{x}_{n, i} \hat{x}_{n, j}-x_{n, i} x_{n, j}\right)^{2} =1/2\|\tilde{x}_n\tilde{x}_n^\top-x_nx_n^\top\|_F^2\leq 2$. 

Now we bound the first term,
\begin{eqnarray*}
	2\sum_{1 \leq i \leq j \leq d}\left(H_{i j}-A_{i j}\right)\left(x_{n, i}x_{n, j}-\hat{x}_{n, i} \hat{x}_{n, j}\right)&\;=\;& \ip{H-A}{x_nx_n^\top-\tilde{x}_n\tilde{x}_n^\top}\\
	&\;=\;& x_n^\top B x_n-\tilde{x}_n^\top B \tilde{x}_n \\
	&\;\leq \;& 2\|B\|_2 \;.
\end{eqnarray*}
So we have $|\ell_{D,\tilde{D}}|\leq \varepsilon$ whenever $\|B\|_2\leq n\varepsilon\beta^2-1/n$. 

For any fixed unit vector $\|v\|_2=1$, we have
\begin{eqnarray*}
	v^\top B v= 2\sum_{1\leq i\leq j\leq d}B_{ij}v_iv_j \sim \cN(0, 2\sum_{1\leq i\leq j\leq d}v_i^2v_j^2) \;\;=\;\; \cN(0,1)\;.
\end{eqnarray*}
Then we have
\begin{eqnarray*}
	\prob\left(|\ell_{D,\tilde{D}}|\geq \varepsilon\right) &\;\;\leq\;\;& \prob\left(\|B\|_2\geq n\varepsilon\beta^2-1/n\right)\\
	&\;\; = \;\;& \prob\left( \cN(0,1) \geq n\varepsilon\beta^2-\frac1n\right)\\
	&\;\; = \;\;& \Phi\left(\frac1n-n\varepsilon\beta^2\right)\;,
\end{eqnarray*}
where $\Phi$ is CDF of standard Gaussian. 
According to Gaussian mechanism, if $\beta=(1/(n\varepsilon ))\sqrt{2\log(1.25/\delta)}$, we have $\Phi\left(\frac1n-n\varepsilon\beta^2\right)\leq \delta$.

\begin{algorithm2e}[ht]
   \caption{Interactive differentially private mechanisms for {\sc DPMMWfilter}}
   \label{alg:DPmechanisms}
   \DontPrintSemicolon 
   	\SetKwProg{Fn}{}{:}{}
	{
\vspace*{5pt}
	\Fn{$q_{\rm spectral} (\{\{\Psi_{r}^{(\ell)}\}_{r\in[t_\ell]}\}_{\ell\in[s]},\{(\mu^{(\ell)},\lambda^{(\ell)})\}_{\ell\in[s]},\varepsilon,\delta)$}{ 
	$ S  \leftarrow {\rm Filter}
	(\{\{\Psi_{r}^{(\ell)}\}_{r\in[t_\ell]}\}_{\ell\in[s]},\{(\mu^{(\ell)},\lambda^{(\ell)})\}_{\ell\in[s]},\varepsilon,\delta)$\;
   $\mu \leftarrow  (1/|S| ) \big( \sum_{i\in S} x_i\big)  + {\cal N}(0,(2B\sqrt{2d\log(1.25/\delta)}/(n\varepsilon))^2{\mathbf I})$\;
    $\lambda \leftarrow \| M(S) - {\mathbf I} \|_2 + {\rm Lap}(2B^2d/(n\varepsilon)) $\;
    \KwRet{$(\mu,\lambda ) $}
	}
	\Fn{$q_{\rm size} (\{\{\Psi_{r}^{(\ell)}\}_{r\in[t_\ell]}\}_{\ell\in[s]},\{(\mu^{(\ell)},\lambda^{(\ell)})\}_{\ell\in[s]},\varepsilon,\delta)$}{
		$ S  \leftarrow {\rm Filter}
	(\{\{\Psi_{r}^{(\ell)}\}_{r\in[t_\ell]}\}_{\ell\in[s]},\{(\mu^{(\ell)},\lambda^{(\ell)})\}_{\ell\in[s]},\varepsilon,\delta)$\;
	\KwRet{$|S| + {\rm Lap}(1/\varepsilon)$}
	}
	\Fn{$q_{\rm MMW} (\{\{\Psi_{r}^{(\ell)}\}_{r\in[t_\ell]}\}_{\ell\in[s]},\{(\mu^{(\ell)},\lambda^{(\ell)})\}_{\ell\in[s]},\alpha^{(s)},\mu^{(s)}_t,\varepsilon,\delta)$}{
		$ S  \leftarrow {\rm Filter}
	(\{\{\Psi_{r}^{(\ell)}\}_{r\in[t_\ell]}\}_{\ell\in[s]},\{(\mu^{(\ell)},\lambda^{(\ell)})\}_{\ell\in[s]},\varepsilon,\delta)$\;
	$\Sigma^{(s)}_{t_s+1}\leftarrow M(S) + {\cal N}(0,(4B^2d\sqrt{2\log(1.25/\delta)}/(n\varepsilon))^2{\mathbf I})$\;
	$U\leftarrow (1/{\rm Tr}(\exp(\alpha^{(s)}\sum_{r=1}^{t_s+1}(\Sigma_r^{(s)}-{\mathbf I}))))\exp(\alpha^{(s)}\sum_{r=1}^{t_s+1}(\Sigma_r^{(s)}-{\mathbf I}))$\;
	$\psi\leftarrow \langle M(S)-{\mathbf I},U \rangle +{\rm Lap}(2B^2d/(n\varepsilon))$\;
	\KwRet{$(\Sigma_{t_s+1}^{(s)}  ,U,\psi  ) $}
	}
	\Fn{$q_{\rm 1Dfilter} (\{\{\Psi_{r}^{(\ell)}\}_{r\in[t_\ell]}\}_{\ell\in[s]},\{(\mu^{(\ell)},\lambda^{(\ell)})\}_{\ell\in[s]}, \mu,U ,\alpha, \varepsilon,\delta)$}{
		$ S  \leftarrow {\rm Filter}
	(\{\{\Psi_{r}^{(\ell)}\}_{r\in[t_\ell]}\}_{\ell\in[s]},\{(\mu^{(\ell)},\lambda^{(\ell)})\}_{\ell\in[s]},\varepsilon,\delta)$\;
	\KwRet{$\rho \leftarrow \text{\sc DPthreshold}(\mu,U,\alpha,\varepsilon,\delta,S)$}
	}
	
	\Fn{{\rm Filter}$(\{\{\Psi_{r}^{(\ell)}\}_{r\in[t_\ell]}\}_{\ell\in[s]},\{(\mu^{(\ell)},\lambda^{(\ell)})\}_{\ell\in[s]})$}{ 
	$S^{(1)}\leftarrow [n] $\;
	    \For{{\rm epoch} $\ell=1,\ldots,s$}{ 
	        $\alpha^{(\ell)} \leftarrow 1/(100(0.1/C+1.01)\lambda^{(\ell)}) $\;
	        $S_1^{(\ell)} \leftarrow S^{(\ell)}$\;
	        \For{$r=1,\ldots,t_s$}{ 
	        $S_{r+1}^{(\ell)} \leftarrow S^{(\ell)}_r \setminus$  $\{ i  \,|\, i \in {\cal T}_{2\alpha}$ for   $\{\tau_j = (x_j-\mu_r^{(\ell)})^\top U_r^{(\ell)} (x_j-\mu_r^{(\ell)})\}_{j\in S_{r}^{(\ell)}} $ and $\tau_i \geq \rho_r^{(\ell)} \,Z_r^{(\ell)} \}$,  where ${\cal T}_{2\alpha}$ is defined in Definition~\ref{def:topset}.
	        }
	        
	   }
	    \KwOut{$S^{(s)}_{t_s}$} 
	} 
}
\end{algorithm2e}

\begin{algorithm2e}[ht]
   \caption{Interactive version of {\sc DPMMWfilter} }
   \label{alg:MMWfiltering_interactive}
   	\DontPrintSemicolon 
	\SetKwFunction{qmean}{$q_{\rm mean}$}
	\SetKwFunction{qPCA}{$q_{\rm PCA}$}	
	\KwIn{$\alpha \in(0,1)$,  $T_1,T_2$, $\varepsilon_1 = \varepsilon /(4 T_1)$ ,
	    $\delta_1= \delta/(4T_1) $, 
	    $\varepsilon_2 = \min\{0.9,\varepsilon\}/(4\sqrt{10T_1 T_2\log(4/\delta)})$, 
	    $\delta_2= \delta/(20T_1T_2)$
	}
		\SetKwProg{Fn}{}{:}{}
	{ 
    {\bf if} $n <  (4/\varepsilon_1)\log(1/(2\delta_1))$  {\bf then Output:} $\emptyset$ \;
	    \For{{\rm epoch} $s=1,2,\ldots, T_1$}{
            $ (\mu^{(s)},\lambda^{(s)})  \leftarrow q_{\rm spectral} (\{\{\Psi_{r}^{(\ell)}\}_{r\in[t_\ell]}\}_{\ell\in[s-1]},\{(\mu^{(\ell)},\lambda^{(\ell)})\}_{\ell\in[s-1]},\varepsilon_1,\delta_1)$ \;
            $n^{(s)}\leftarrow q_{\rm size}(\{\{\Psi_{r}^{(\ell)}\}_{r\in[t_\ell]}\}_{\ell\in[s-1]},\{(\mu^{(\ell)},\lambda^{(\ell)})\}_{\ell\in[s-1]},\varepsilon_1,\delta_1)$\;
            {\bf if }$n^{(s)}\leq 3n/4$ {\bf then} {\rm terminate}\;
            \If{$ \lambda^{(s)}\leq C\alpha \log(1/\alpha) $ }{\KwOut{$ \mu^{(s)}$}} 
            $\alpha^{(s)} \leftarrow 1/(100(0.1/C+1.01)\lambda^{(s)})$\;
	        $t_s \leftarrow 0$\;
	        \;
	    \For{$t=1,2,\ldots,T_2$}{
	        $(\mu_t^{(s)},\lambda_t^{(s)}) \leftarrow q_{\rm spectral}(\{\{\Psi_{r}^{(\ell)}\}_{r\in[t_\ell]}\}_{\ell\in[s]},\{(\mu^{(\ell)},\lambda^{(\ell)})\}_{\ell\in[s]},\varepsilon_2,\delta_2) $\;
	        \eIf{$\lambda_t^{(s)}\leq 0.5 \lambda^{(s)}$}{ 
	            terminate epoch
	        }{
	            $(\Sigma_t^{(s)},U_t^{(s)}, \psi_t^{(s)} )  \leftarrow q_{\rm PMMW}( \{\{\Psi_{r}^{(\ell)}\}_{r\in[t_\ell]}\}_{\ell\in[s]},\{(\mu^{(\ell)},\lambda^{(\ell)})\}_{\ell\in[s]},\alpha^{(s)},\mu_t^{(s)},\varepsilon_2,\delta_2)$ \;
	            
	            \eIf{$\psi_t^{(s)}\leq (1/5.5)\lambda_t^{(s)}$}
	            {
	            $\alpha^{(s)}_t \leftarrow 0$\;
	            }
	            {
            $Z_t^{(s)} \leftarrow {\rm Unif}([0,1])$\;
            $\rho_t^{(s)} \leftarrow q_{\text{1Dfilter}}(\{\{\Psi_{r}^{(\ell)}\}_{r\in[t_\ell]}\}_{\ell\in[s]},\{(\mu^{(\ell)},\lambda^{(\ell)})\}_{\ell\in[s]},\mu_t^{(s)},U_t^{(s)},\alpha, \varepsilon_2,\delta_2)$\;
            $\alpha_t^{(s)} \leftarrow \alpha$
            }
	        }
            $\Psi^{(s)}_{t} \leftarrow (\mu_t^{(s)},\lambda^{(s)}_t,\Sigma_t^{(s)},  U^{(s)}_t,\psi^{(s)}_t,Z^{(s)}_t,\rho^{(s)}_t,\alpha^{(s)}_t)$\;
            $t_s \leftarrow t $\;
	    }
	    }
	    \KwOut{$\mu^{(T_1)}_{t_{T_1}}$} 
    }
\end{algorithm2e}

\subsection{Proof of part 2 of Theorem~\ref{thm:main2} on accuracy}
\label{sec:proof_main2_acc}

The accuracy of PRIME follows from 
the fact that $q_{\rm range}$ returns a hypercube that contains all the clean data with high probability (Lemma~\ref{lem:DPrange}) and that {\sc DPMMWfilter} achieves the desired accuracy (Theorem~\ref{thm:main}) if the original uncorrupted dataset $S_{\rm good}$ is $\alpha$-subgaussian good. 
$S_{\rm good}$ is  $\alpha$-subgaussian good if we have $n=\widetilde\Omega(d/\alpha^2)$ as shown in Lemma~\ref{lemma:alpha-good-sample-complexity}. 
We present the proof of Theorem~\ref{thm:main} below. 

\begin{thm}[Analysis of accuracy of  {\sc DPMMWfilter}] \label{thm:main}
	Let $S$ be an $\alpha$-corrupted sub-Gaussian dataset, where $\alpha\leq c$ for some universal constant $c\in (0,1/2)$. Let $S_{\rm good}$ be $\alpha$-subgaussian good with respect to $\mu\in \reals^d$. Suppose  ${\cal D} = \{ x_{i} \in \bar{x}+[-B/2,B/2]^d\}_{i=1}^{n}$ be the projected dataset. If $
	n \geq  \widetilde\Omega\left(\frac{d^{3/2}B^2\log(2/\delta)}{\varepsilon\alpha\log1/\alpha }\right)$, then {\sc DPMMWfilter} terminates after at most $O(\log dB^2)$ epochs and outputs $S^{(s)}$ such that with probability $0.9$, we have $|S_t^{(s)}\cap S_{\mathrm{good}}|\geq (1-10\alpha )n$ and
	\begin{eqnarray*}
		\|\mu(S^{(s)})-\mu\|_2\lesssim \alpha\sqrt{\log 1/\alpha}\;.
	\end{eqnarray*}
	Moreover, each epoch runs for at most $O(\log d)$ iterations.
\end{thm}
\begin{proof}
In $s=O(\log_{0.98} ((C\alpha\log(1/\alpha))/\|M(S^{(1)})-{\mathbf I}\|_2))$ epochs, following Lemma~\ref{lemma:invariant_main_formal} guarantees that we find a candidate set $S^{(s)}$ of  samples  with $\|M(S^{(s)}-{\mathbf I}\|_2\leq C\alpha \log (1/\alpha)$. We provide proof of Lemma~\ref{lemma:invariant_main_formal} in the Appendix~\ref{sec:proof_invariant_main}.

\begin{lemma}
\label{lemma:invariant_main_formal}
     Let $S$ be an $\alpha$-corrupted sub-Gaussian dataset under Assumption~\ref{asmp:adversary}. 
     For an epoch $s\in[T_1]$ and an iteration $t\in[T_2]$, under the hypotheses of Lemma~\ref{lemma:progress}, 
     if $S_{\rm good}$ is $\alpha$-subgaussian good with respect to $\mu\in \reals^d$ as in Definition~\ref{def:asubgaussiangood},
     $n=\widetilde\Omega(  d^{3/2}\log(1/\delta)/(\varepsilon\alpha))$, and $|S_t^{(s)}\cap S_{\rm good}|\geq (1-10\alpha)n$
    then 
	 with probability $1-O(1/\log^3 d)$ the conditions in 
	 Eqs.~\eqref{eq:1dfilter_reg1} and \eqref{eq:1dfilter_reg2} hold. 
When these two conditions hold,  more corrupted samples are removed in expectation than the uncorrupted samples, i.e., $\E|(S_t^{(s)}\setminus S_{t+1}^{(s)})\cap S_{\rm good}|\leq \E|(S_t^{(s)}\setminus S_{t+1}^{(s)})\cap S_{\rm bad}|$. 
Further, for an epoch $s\in[T_1]$ there exists a constant $C>0$  such that   if  $\|M(S^{(s)})-{\mathbf I}\|_2\geq C\,\alpha\log(1/\alpha)$, then with probability $1-O(1/\log^2 d)$, the $s$-th epoch terminates after $O(\log d)$ iterations and outputs $S^{(s+1)}$ such that $\|M(S^{(s+1)})-{\mathbf I}\|_2\leq 0.98\|M(S^{(s)})-{\mathbf I}\|_2$.
\end{lemma}

Lemma~\ref{thm:reg} ensures that we get the desired bound of $\|\mu(S^{(s)})-\mu\|_2 = O(\alpha\sqrt{\log(1/\alpha)})$ as long as $S^{(s)}$ has enough clean data, i.e.,~$|S^{(s)}\cap S_{\rm good} |\geq n(1-\alpha)$.
Since Lemma~\ref{lemma:invariant_main_formal} gets invoked at most $O((\log d)^2)$ times, we can take a union bound, and the  following argument conditions on the good events in Lemma~\ref{lemma:invariant_main_formal} holding, which happens with probability at least $0.99$. To turn the average case guarantee of Lemma~\ref{lemma:invariant_main_formal} into a constant probability guarantee, we apply the optional stopping theorem. 
Recall that the $s$-th epoch starts with a set $S^{(s)}$ and outputs a filtered set $S^{(s)}_t$ at the $t$-th inner iteration. 
We measure the progress by 
by summing the number of clean samples removed 
up to epoch $s$ and 
iteration $t$ and the number of remaining corrupted samples, 
defined  as  $d_{t}^{(s)} \triangleq|(S_{\mathrm{good}}\cap S^{(1)})\setminus S_{t}^{(s)}|+|S_{t}^{(s)}\setminus (S_{\mathrm{good}}\cap S^{(1)})|$. Note that $d_1^{(1)}=\alpha n$, and $d_t^{(s)}\geq0$. At each epoch and iteration, we have
	\begin{eqnarray*}
		\E [d_{t+1}^{(s)}-d_{t}^{(s)}|d_1^{(1)}, d_2^{(1)}, \cdots, d_{t}^{(s)}] &=& \E\left[|S_{\mathrm{good}}\cap(S_{t}^{(s)}\setminus S_{t+1}^{(s)})|-|S_{\mathrm{bad}}\cap(S_{t}^{(s)}\setminus S_{t+1}^{(s)})|\right] 
		\;\leq\; 0,
	\end{eqnarray*}
	from part 1 of Lemma~\ref{lemma:invariant_main_formal}.
	Hence, $d_t^{(s)}$ is a non-negative  super-martingale. By the  optional stopping theorem, at stopping time, we have $\E[d_t^{(s)}]\leq d_1^{(1)}=\alpha n$. By the Markov inequality, $d_t^{(s)}$ is less than $10\alpha n$ with probability $0.9$, i.e., $|S_t^{(s)}\cap S_{\mathrm{good}}|\geq (1-10\alpha )n$. 
The desired bound in Theorem~\ref{thm:main} follows from Lemma~\ref{thm:reg}.

\end{proof}

\subsection{Proof of Lemma~\ref{lemma:invariant_main_formal}} 
\label{sec:proof_invariant_main}

Lemma~\ref{lemma:invariant_main_formal} is a combination of Lemma~\ref{lemma:progress} and Lemma~\ref{lemma:progress_epoch}. We state the technical lemmas and subsequently provide the proofs. 

\begin{lemma}
\label{lemma:progress}
For an epoch $s$ and  an iteration $t$ such that $\lambda^{(s)}>C \alpha\log(1/\alpha)$, $\lambda_t^{(s)}> 0.5 \lambda_0^{(s)}$,  and $n^{(s)}>3n/4$,
if $n \gtrsim \frac{B^2 (\log B) d^{3/2}\log(1/\delta)}{\varepsilon\alpha} $   and $|S_t^{(s)}\cap S_{\rm good}|\geq (1-10\alpha)n$
then with probability $1-O(1/\log^3 d)$, the conditions in 
	 Eqs.~\eqref{eq:1dfilter_reg1} and \eqref{eq:1dfilter_reg2} hold. When these two conditions hold  we have  $\E|S_t^{(s)}\setminus S_{t+1}^{(s)}\cap S_{\rm good}|\leq \E|S_t^{(s)}\setminus S_{t+1}^{(s)}\cap S_{\rm bad}|$. If $n \gtrsim \frac{B^2 (\log B) d^{3/2}\log(1/\delta)}{\varepsilon\alpha} $,  $\psi_t^{(s)}>\frac{1}{5.5}\lambda_t^{(s)}$,  and $n^{(s)}>3n/4$, then we have with probability $1-O(1/\log^3 d)$,
	 $\ip{M(S_{t+1}^{(s)})-{\mathbf I}}{U_t^{(s)}}\leq 0.76\ip{M(S_{t}^{(s)})-{\mathbf I}}{U_t^{(s)}}$.
\end{lemma}

\begin{lemma}\label{lemma:progress_epoch}
	For an epoch $s$ and for all  $t=0,1,\cdots,T_2=O(\log d)$ if   Lemma~\ref{lemma:progress} holds,   $n^{(s)}>3n/4$, and $n \gtrsim \frac{B^2 (\log B) d^{3/2}\log(1/\delta)}{\varepsilon\alpha} $, then we have  $\|M(S^{(s+1)})-{\mathbf I}\|_2\leq 0.98\|M(S^{(s)})-{\mathbf I}\|_2$ with probability $1-O(1/\log^2 d)$.
\end{lemma}





\subsubsection{Proof of Lemma~\ref{lemma:progress}}

\begin{proof}[Proof of Lemma~\ref{lemma:progress}] 
To prove that we make progress for each iteration, we first show our dataset satisfies regularity conditions in Eqs.~\eqref{eq:1dfilter_reg1} and \eqref{eq:1dfilter_reg2} that we need for {\sc DPthreshold}.  Following Lemma~\ref{lemma:good_tau} implies with probability $1-1/(\log^3 d)$, our scores satisfies the regularity conditions needed in  Lemma~\ref{lemma:1dfilter}.

\begin{lemma}\label{lemma:good_tau}
For each epoch $s$ and iteration $t$, 
under the hypotheses of Lemma~\ref{lemma:progress},  with probability $1-O(1/\log^3 d)$, we have \begin{eqnarray}
\frac{1}{n} \sum_{i\in S_{\rm good}\cap{\cal T}_{2\alpha}}\tau_i &\le& \psi/1000 \label{eq:1dfilter_reg1}\\
\frac{1}{n} \sum_{i\in S_{\rm good}\cap S_t^{(s)}} (\tau_i-1) &\le & \psi/1000\;,\label{eq:1dfilter_reg2}
\end{eqnarray}
where $\psi\triangleq \frac1n\sum_{i\in S_t^{(s)}}(\tau_i-1)$.
\end{lemma}

Then by Lemma~\ref{lemma:1dfilter} our {\sc DPthreshold} gives us a threshold $\rho$ such that 
	\begin{eqnarray*}
	    \sum_{i\in S_{\rm good}\cap {\cal T}_{2\alpha}} \textbf{1}\{\tau_i\le \rho \}\frac{\tau_i}{\rho}+ \textbf{1}\{\tau_i > \rho \}   \le \sum_{i\in S_{\rm bad}\cap {\cal T}_{2\alpha}} \textbf{1}\{\tau_i\le \rho \}\frac{\tau_i}{\rho} + \textbf{1}\{\tau_i > \rho \}\;.
	\end{eqnarray*}
	
Conditioned on the hypotheses and the claims of Lemma~\ref{lemma:1dfilter}, according to our filter rule from Algorithm~\ref{alg:DPMMWfilter}, we have 
	\begin{eqnarray*}
	    \E|(S_t^{(s)}\setminus S_{t+1}^{(s)})\cap S_{\rm good}| \;=\; \sum_{i\in S_{\rm good}\cap {\cal T}_{2\alpha}} \textbf{1}\{\tau_i\le \rho \}\frac{\tau_i}{\rho}+ \textbf{1}\{\tau_i > \rho \} 
	\end{eqnarray*}
	and \begin{eqnarray*}
	    \E|(S_t^{(s)}\setminus S_{t+1}^{(s)})\cap S_{\rm bad}| \;=\; \sum_{i\in S_{\rm bad}\cap {\cal T}_{2\alpha}} \textbf{1}\{\tau_i\le \rho \}\frac{\tau_i}{\rho} + \textbf{1}\{\tau_i > \rho \}\;.
	\end{eqnarray*}

This implies $\E|(S_t^{(s)}\setminus S_{t+1}^{(s)})\cap S_{\rm good}|\leq  \E|(S_t^{(s)}\setminus S_{t+1}^{(s)})\cap S_{\rm bad}|$. 
At the same time, Lemma~\ref{lemma:1dfilter} gives us a $\rho$ such that
	with probability $1-O(\log^3 d)$
	\begin{eqnarray*}
	    \frac{1}{n}\sum_{i\in S_{t+1}^{(s)}}(\tau_i-1)-2\alpha\leq \frac{1}{n}\sum_{ \tau_i\leq \rho }(\tau_i-1) \leq \frac34\cdot\frac{1}{n}\sum_{i\in S_{t}^{(s)}}(\tau_i-1) \;.
	\end{eqnarray*}

	Hence, we have
	\begin{eqnarray*}
		\ip{M(S_{t}^{(s)})-{\mathbf I}}{U_t^{(s)}} - \ip{M(S_{t+1}^{(s)})-{\mathbf I}}{U_t^{(s)}} &= & 
		\frac1n\sum_{i\in S_{t}^{(s)}\setminus S_{t+1}^{(s)}}(\tau_i-1)\\
		&\geq & \frac{1}{4n}\sum_{i\in S_t^{(s)}}(\tau_i-1)-2\alpha\\
		&\overset{(a)}{\geq} & \frac{1}{4}\cdot \frac{998}{1000}\ip{M(S_{t}^{(s)})-{\mathbf I}}{U_t^{(s)}}\;,
	\end{eqnarray*}
	where $(a)$ follows from our assumption on $\lambda_t$ and stopping criteria. Rearranging the  terms completes the proof.
\end{proof}

\subsubsection{Proof of Lemma~\ref{lemma:good_tau}}
\begin{proof}[Proof of Lemma~\ref{lemma:good_tau}]
First of all, Lemma~\ref{lemma:lap_noise}, Lemma~\ref{lemma:chi-square} and Lemma~\ref{lemma:matrix_sp_norm} gives us following Lemma~\ref{lemma:samples_need}, which basically shows with enough samples, we can make sure the noises added for privacy guarantees are small enough with probability $1-O(1/\log^3 d)$.
\begin{lemma}
\label{lemma:samples_need}
For $\alpha \in (0, 0.5)$, if $n \gtrsim \frac{B^2 (\log B) d^{3/2}\log(1/\delta)}{\varepsilon\alpha}$ and $n^{(s)}>3n/4$
then we have with probability $1-O(1/\log^3 d)$, following conditions simultaneously hold:
\begin{enumerate}
    \item $\|\mu_t^{(s)}-\mu(S_{t}^{(s)})\|_2^2\leq 0.001 \alpha\log1/\alpha$
    \item $|\psi_t^{(s)}-\ip{M(S_t^{(s)})-{\mathbf I}}{U_t^{(s)}}|\leq  0.001 \alpha\log1/\alpha$
    \item $\left|\lambda_t^{(s)}-\|M(S_t^{(s)})-{\mathbf I}\|_2\right|\leq 0.001\alpha\log1/\alpha$
    \item $	\left|\lambda^{(s)}-\|M(S^{(s)})-{\mathbf I}\|_2\right|\leq 0.001\alpha\log1/\alpha$
    \item $\left\|M(S_{t+1}^{(s)})-\Sigma_t^{(s)}\right\|_2\leq 0.001\alpha\log1/\alpha $
    \item $\|\mu^{(s)}-\mu(S^{(s)})\|_2^2\leq 0.001\alpha\log1/\alpha$
\end{enumerate}
\end{lemma}

Now under above conditions, since $\lambda_1^{(s)} > C \alpha\log1/\alpha$, we have $\|M(S_{t}^{(s)})-{\mathbf I}\|_2> 0.5(C-0.002) \alpha\log1/\alpha $.
Using the fact that $ \mu(S^{(s)}_t) = (1/n)\sum_{i\in S_t^{(s)}} x_i $, we also have
\begin{eqnarray*}
&& \frac1n\sum_{i\in S_t^{(s)}}(\tau_i-1)\\
&=&\frac{1}{n}\sum_{i\in S_{t}^{(s)}}\ip{\left(x_i-\mu_t^{(s)}\right)\left(x_i-\mu_t^{(s)}\right)^\top-{\mathbf I}}{U_t^{(s)}}\\
	&=&\frac{1}{n}\sum_{i\in S_{t}^{(s)}}\ip{\left(x_i-\mu(S_t^{(s)})\right)\left(x_i-\mu(S_t^{(s)})\right)^\top-{\mathbf I}}{U_t^{(s)}}\\
	&&\;\;+\;\;\;\;\frac{|S_t^{(s)}|}{n}\ip{\left(\mu( S_t^{(s)})-\mu_t^{(s)}\right)\left(\mu(S_t^{(s)})-\mu_t^{(s)}\right)^\top}{U_t^{(s)}}\\
	&= &\ip{M(S_t^{(s)})-{\mathbf I}}{U_t^{(s)}}+\frac{|S_t^{(s)}|}{n}\ip{\left(\mu( S_t^{(s)})-\mu_t^{(s)}\right)\left(\mu(S_t^{(s)})-\mu_t^{(s)}\right)^\top}{U_t^{(s)}}\;.
\end{eqnarray*} 
Thus, from the first and the second claims in Lemma~\ref{lemma:samples_need},  we have
\begin{eqnarray}
	|\psi -\psi_t^{(s)}|\leq 0.002\;\alpha \log1/\alpha\;. \label{eq:boundG7}
\end{eqnarray}


For an epoch $s$ and an iteration $t$, since $\alpha n \leq S_{\rm good}\cap {\cal T}_{2\alpha}\cap S_t^{(s)}\leq 2\alpha n$, we have
	\begin{eqnarray*}
		&&\frac{1}{n} \sum_{i\in S_{\rm good}\cap{\cal T}_{2\alpha}\cap S_t^{(s)}} \tau_i 
		\;\; =\;\; \frac{1}{n} \sum_{i\in S_{\rm good}\cap{\cal T}_{2\alpha}\cap S_t^{(s)}} \ip{(x_i-\mu_t^{(s)})(x_i-\mu_t^{(s)})^\top}{U_t^{(s)}}\\
		&\overset{(a)}{\leq} & \frac{2}{n} \sum_{i\in S_{\rm good}\cap{\cal T}_{2\alpha}\cap S_t^{(s)}} \ip{(x_i-\mu)(x_i-\mu)^\top}{U_t^{(s)}}
		+\frac{2|S_{\rm good}\cap {\cal T}_{2\alpha}\cap S_t^{(s)}|}{n} \ip{(\mu-\mu_t^{(s)})(\mu-\mu_t^{(s)})^\top}{U_t^{(s)}} \\
		&\overset{(b)}{\leq} & O(\alpha \log 1/\alpha) +4\alpha \ip{(\mu-\mu_t^{(s)})(\mu-\mu_t^{(s)})^\top}{U_t^{(s)}} \\
		& \leq &O(\alpha \log 1/\alpha)+ 4\alpha \|\mu_t^{(s)}-\mu\|_2^2\\
		&\leq &O(\alpha \log 1/\alpha)+ 4\alpha \left(\|\mu-\mu(S_t^{(s)})\|_2+\|\mu(S_t^{(s)})-\mu_t^{(s)}\|_2\right)^2\\
		&\overset{(c)}{\leq}&O\left(\alpha\log1/\alpha\right)+ 4\alpha \left(O\left(\alpha\sqrt{\log1/\alpha}\right)+\sqrt{\alpha\left(O\left(\alpha\log1/\alpha\right)+\|M(S_{t}^{(s)})-{\mathbf I}\|_2\right)}+\|\mu(S_t^{(s)})-\mu_t^{(s)}\|_2\right)^2\\
		&\leq & O(\alpha \log 1/\alpha)+8\alpha^2\left(\|M(S_{t}^{(s)})-{\mathbf I}\|_2+O\left(\alpha\log1/\alpha\right)\right)+O(8\alpha^3 \log 1 / \alpha)+8\alpha^2\log1/\alpha\\
		&\overset{(d)}{\leq} & \frac{1}{ 1000}\left(\frac{\|M(S_{t}^{(s)})-{\mathbf I}\|_2-0.001\;\alpha\log1/\alpha}{5.5}-0.002\;\alpha\log1/\alpha\right)\\
		&\leq & \frac{\psi_t^{(s)}-0.002\;\alpha\log1/\alpha}{1000}\\
		&\leq & \frac{\psi}{1000}\;,
	\end{eqnarray*}
	
	where $(a)$ follows from the fact that for any vector $x, y, z$, we have $(x-y)(x-y)^{\top} \preceq 2(x-z)(x-z)^{\top}+2(y-z)(y-z)^{\top}$, $(b)$ follows from Lemma~\ref{lemma:gaussian_subset}, $(c)$ follows from Lemma~\ref{thm:reg}, $(d)$ follows from our choice of large constant $C$, and in the last inequality we used Eq.~\eqref{eq:boundG7}.

Similarly we have
\begin{eqnarray*}
	&&\frac{1}{n}\sum_{i\in S_{\rm good}\cap S_{t}^{(s)}}(\tau_i-1)\\
	&=&\frac{1}{n} \sum_{i\in S_{\rm good}\cap S_{t}^{(s)}} \ip{(x_i-\mu_t^{(s)})(x_i-\mu_t^{(s)})^\top-{\mathbf I}}{U_t^{(s)}} \\
	&= & \frac{1}{n} \sum_{i\in S_{\rm good}\cap S_{t}^{(s)}} \ip{\left(x_i-\mu(S_{\rm good}\cap S_t^{(s)})\right)\left(x_i-\mu(S_{\rm good}\cap S_t^{(s)})\right)^\top-{\mathbf I}}{U_t^{(s)}}\\
	&&+\frac{|S_{\rm good}\cap S_{t}^{(s)}|}{n}\ip{\left(\mu(S_{\rm good}\cap S_t^{(s)})-\mu_t^{(s)}\right)\left(\mu(S_{\rm good}\cap S_t^{(s)})-\mu_t^{(s)}\right)^\top}{U_t^{(s)}}\\
	&\overset{(a)}{\leq} & O\left(\alpha\log1/\alpha\right)+ \left\|\mu(S_{\rm good}\cap S_t^{(s)})-\mu_t^{(s)}\right\|_2^2\\
	&\leq & O\left(\alpha\log1/\alpha\right)+ \left(\left\|\mu(S_{\rm good}\cap S_t^{(s)})-\mu\right\|_2+\left\|\mu-\mu(S_t^{(s)})\right\|_2\right)^2+0.001\;\alpha\log1/\alpha\\
	&\overset{(b)}{\leq} & O\left(\alpha\log1/\alpha\right) +\left(O(\alpha\sqrt{\log1/\alpha})+\sqrt{\alpha(\|M(S_t^{(s)})-{\mathbf I}\|_2+O(\alpha\log 1/\alpha))}\right)^2+0.001\;\alpha\log1/\alpha\\
	&\leq & O\left(\alpha\log1/\alpha\right)+\alpha\left(\|M(S_{t}^{(s)})-{\mathbf I}\|_2+O\left(\alpha\log1/\alpha\right)\right)+O(\alpha^2 \log 1 / \alpha)++0.001\;\alpha\log1/\alpha\\
	&\overset{(c)}{\leq} &  \frac{1}{ 1000}\left(\frac{\|M(S_{t}^{(s)})-{\mathbf I}\|_2-0.001\;\alpha\log1/\alpha}{5.5}-0.002\;\alpha\log1/\alpha\right)\\
	&\leq & \frac{\psi_t^{(s)}-0.002\;\alpha\log1/\alpha}{1000}\\
		&\leq & \frac{\psi}{1000}\;,
	\end{eqnarray*}
	where $(a)$ follows from Lemma~\ref{lemma:gaussian_subset}, $(b)$ follows from Lemma~\ref{lemma:gaussian_subset2} and Lemma~\ref{thm:reg} and $(c)$ follows from our choice of large constant $C$.

\end{proof}

\subsubsection{Proof of Lemma~\ref{lemma:progress_epoch}}
\label{sec:proof_progress_epoch}

\begin{proof}[Proof of Lemma~\ref{lemma:progress_epoch}]


Under the conditions of Lemma~\ref{lemma:samples_need}, we have picked $n$ large enough such that with probability $1-O(1/\log^3 d)$, we have
\begin{eqnarray*}
	\|\Sigma_{t+1}^{(s)}-{\mathbf I}\|_2\approx_{0.01}\|M(S_{t+1}^{(s)})-{\mathbf I}\|_2\;.
\end{eqnarray*}

By Lemma~\ref{lemma:progress}, we now have
\begin{eqnarray}
		\ip{M(S_{t+1}^{(s)})-{\mathbf I}}{U_{t}^{(s)}}&\leq & 0.76\ip{M(S_{t}^{(s)})-{\mathbf I}}{U_{t}^{(s)}}\nonumber \\
	&\leq & 0.76\ip{M(S_{1}^{(s)})-{\mathbf I}}{U_{t}^{(s)}} \nonumber\\
	&\leq & 0.76\|M(S_{1}^{(s)})-{\mathbf I}\|_2\;. \label{eq:boundG4}
	\end{eqnarray}
 Since $\lambda_1^{(s)} > C \alpha\log1/\alpha$, we have $\|M(S_{t+1}^{(s)})-{\mathbf I}\|_2> 0.5(C-0.002) \alpha\log1/\alpha $. Combining the above inequality and the fifth claim of Lemma~\ref{lemma:samples_need} together, we have
\begin{eqnarray*}
	\ip{\Sigma_{t+1}^{(s)}-{\mathbf I}}{U_{t}^{(s)}} \leq \ip{M(S_{t+1}^{(s)})-{\mathbf I}}{U_{t}^{(s)}}+\|\Sigma_{t+1}^{(s)}-M(S_{t+1}^{(s)})\|_2\leq 0.77\|M(S_{1}^{(s)})-{\mathbf I}\|_2\;.
\end{eqnarray*}
	
	By Lemma~\ref{lemma:mono}, we have $M(S_{t+1}^{(s)})-{\mathbf I} \preceq M(S_{1}^{(s)})-{\mathbf I}$. By our choice of $\alpha^{(s)}$, we have $\alpha^{(s)}\left(M(S_{t+1}^{(s)})-{\mathbf I}\right)\preceq \frac{1}{100}{\mathbf I}$ and $\alpha^{(s)}\left(\Sigma_{t+1}^{(s)}-{\mathbf I}\right)\preceq \frac{1}{100}{\mathbf I}$. Therefore, by Lemma~\ref{lemma:regret_bound}, we have
	\begin{eqnarray*}
		&&\left\|\sum_{t=1}^{T_2}\Sigma_{t+1}^{(s)}-{\mathbf I}\right\|_{2} \\
		&\leq & \sum_{t=1}^{T_2} \ip{ \Sigma_{t+1}^{(s)}-{\mathbf I}}{U_{t}^{(s)}}+\alpha^{(s)}\sum_{t=1}^{T_2}\ip{ U_{t}^{(s)}}{\left|\Sigma_{t+1}^{(s)}-{\mathbf I}\right|} \|\Sigma_{t+1}^{(s)}-{\mathbf I}\|_2+\frac{\log (d)}{\alpha^{(s)}}\\
		&\overset{(a)}{\leq} & \sum_{t=1}^{T_2} \ip{ \Sigma_{t+1}^{(s)}-{\mathbf I}}{U_{t}^{(s)}}+\frac{1}{100}\sum_{t=1}^{T_2}\ip{ U_{t}^{(s)}}{\left|\Sigma_{t+1}^{(s)}-{\mathbf I}\right|} +200\log(d)\|M(S_1^{(s)})-{\mathbf I}\|_2\\
	\end{eqnarray*}
	where $(a)$ follows from our choice of $\alpha^{(s)}$ and  $C$. 
	By Lemma~\ref{lem:variance_lower}, $M(S_{t+1}^{(s)})-{\mathbf I}\succeq -c_1\alpha\log1/\alpha\cdot I$ for $t=1,2,\cdots,T_2$, we have
	\begin{eqnarray*}
		|M(S_{t+1}^{(s)})-{\mathbf I}|\preceq M(S_{t+1}^{(s)})-{\mathbf I}+2c_1\alpha\log1/\alpha \; {\mathbf I},
	\end{eqnarray*}
	and hence
	\begin{eqnarray*}
	\ip{ U_{t}^{(s)}}{\left|M(S_{t+1}^{(s)})-{\mathbf I}\right|} \leq \ip{ U_{t}^{(s)}}{M(S_{t+1}^{(s)})-{\mathbf I}}+2c_1\alpha\log1/\alpha
	\end{eqnarray*}
	
	Meanwhile, we have
	\begin{eqnarray*}
		 M(S_{t+1}^{(s)})-{\mathbf I}-\|\Sigma_{t+1}^{(s)}-M(S_{t+1}^{(s)})\|_2\; {\mathbf I}\preceq \Sigma_{t+1}^{(s)}-{\mathbf I}\preceq M(S_{t+1}^{(s)})-{\mathbf I}+\|\Sigma_{t+1}^{(s)}-M(S_{t+1}^{(s)})\|_2\; {\mathbf I}\;.
	\end{eqnarray*}
	Hence, 
	\begin{eqnarray*}
		|\Sigma_{t+1}^{(s)}-{\mathbf I}|\preceq M(S_{t+1}^{(s)})-{\mathbf I}+(3\|\Sigma_{t+1}^{(s)}-M(S_{t+1}^{(s)})\|_2+2c_1\alpha\log1/\alpha)\; {\mathbf I}\\
	\end{eqnarray*}
	Together with Eq.~\eqref{eq:boundG4}, we have 
	\begin{eqnarray*}
		&&\ip{ U_{t}^{(s)}}{\left|\Sigma_{t+1}^{(s)}-{\mathbf I}\right|} \\
		&\leq &\ip{ U_{t}^{(s)}}{M(S_{t+1}^{(s)})-{\mathbf I}}+3\|\Sigma_{t+1}^{(s)}-M(S_{t+1}^{(s)})\|_2+2c_1\alpha \log 1/\alpha\\
		&\leq & 0.79\;\left\|M(S_1^{(s)})-{\mathbf I}\right\|_2+2c_1\alpha\log1/\alpha\;.
	\end{eqnarray*}

	By Lemma~\ref{lem:variance_lower}, we have $M(S_{t+1}^{(s)})-{\mathbf I}\succeq -c_1\alpha\log1/\alpha\; {\mathbf I}$. Also, we know $M(S_{t+1}^{(s)})-{\mathbf I} \preceq M(S_{1}^{(s)})-{\mathbf I}$. Then we have

\begin{eqnarray*}
&&		\left\|M(S_{T_2+1}^{(s)})-{\mathbf I}\right\|_2\\
	&\leq &\frac{1}{T_2} \left\|\sum_{i=1}^{T_2}M({S_{t+1}^{(s)}})-{\mathbf I}\right\|_{2} \\
	&\leq & \frac{1}{T_2}\left\|\sum_{i=1}^{T_2}\Sigma_{t+1}^{(s)}-{\mathbf I}\right\|_{2}+0.001\;\alpha\log1/\alpha\\
	&\leq & \frac{1}{T_2}\left(\sum_{t=1}^{T_2} \ip{ \Sigma_{t+1}^{(s)}-{\mathbf I}}{U_{t}^{(s)}}+\frac{1}{100}\sum_{t=1}^{T_2}\ip{ U_{t}^{(s)}}{\left|\Sigma_{t+1}^{(s)}-{\mathbf I}\right|} +200\log(d)\|M(S_1^{(s)})-{\mathbf I}\|_2\right)+0.001\;\alpha\log1/\alpha\\
	&\leq & 0.79\|M(S_1^{(s)})-{\mathbf I}\|_2+2c_1\alpha\log1/\alpha +\frac{200\log(d)}{T_2}\|M(S_1^{(s)})-{\mathbf I}\|_2+0.001\;\alpha\log1/\alpha\\
	&\leq & 0.98\; \|M(S_{1}^{(s)})-{\mathbf I}\|_2\;,
	\end{eqnarray*}
	where the last inequality follows from our assumption that  $\lambda_0^{(s)} > C \alpha\log1/\alpha$, and conditions of Lemma~\ref{lemma:samples_need} hold and we have $\|M(S_{t+1}^{(s)})-{\mathbf I}\|_2> 0.5(C-0.002) \alpha\log1/\alpha $.
\end{proof}

\newpage
\section{Technical lemmas}
\label{sec:technicallemmas}

\subsection{Lemmata for sub-Gaussian regularity from \cite{dong2019quantum} }
\begin{lemma}[{\cite[Lemma 3.4]{dong2019quantum}} ]\label{lemma:mono}
If $S'\subset S$, then $M(S')\preceq M(S)$.
\end{lemma}
\begin{definition}[{\cite[Definition 4.1]{dong2019quantum}} ]
\label{def:asubgaussiangood}
Let $D$ be a distribution with mean $\mu\in \reals^d$ and covariance ${\mathbf I}$. For $0<\alpha<1/2$, we say a set of points $S=\{X_1, X_2,\cdots, X_n\}$ is $\alpha$-subgaussian good with respect to $\mu\in \reals^d$ if   following inequalities are satisfied:
\begin{itemize}
	\item $\|\mu(S)-\mu\|_2\lesssim \alpha\sqrt{\log1/\alpha} $ and $\left\|\frac{1}{|S|} \sum_{i \in S}\left(X_{i}-\mu(S)\right)\left(X_{i}-\mu(S)\right)^{\top}-{\mathbf I}\right\|_{2} \lesssim \alpha \log 1/\alpha$.
	\item for any subset $T\subset S$ so that $|T|=2\alpha |S|$, we have
	\begin{eqnarray*}
\left\| \frac{1}{|T|} \sum_{i \in T} X_{i}-\mu\right\|_2 \lesssim \sqrt{\log1/\alpha} \;\;\text{and}\;\;\left\|\frac{1}{|T|} \sum_{i \in T}\left(X_{i}-\mu(S)\right)\left(X_{i}-\mu(S)\right)^{\top}-{\mathbf I}\right\|_{2} \lesssim\log1/\alpha\;.
\end{eqnarray*}
\end{itemize}	
\end{definition}

\begin{lemma}[{\cite[Lemma 4.1]{dong2019quantum}} ]
\label{lemma:alpha-good-sample-complexity}
	A set of i.i.d. samples from an identity covariance sub-Gaussian distribution of size $n=\Omega\left(\frac{d+\log 1 / \delta}{\alpha^{2} \log 1 / \alpha}\right)$ is $\alpha$-subgaussian good with respect to $\mu$ with probability $1-\delta$.
\end{lemma}

\begin{lemma}[{\cite[Fact 4.2]{dong2019quantum}}
 ]
 \label{lemma:gaussian_subset}
Let $S$ be an $\alpha$-corrupted sub-Gaussian dataset under Assumption~\ref{asmp:adversary}. If $S_{\rm good}$ is $\alpha$-subgaussian good with respect to $\mu\in \reals^d$, then for any $T\subset S$ such that $|T|\leq 2\alpha|S|$, we have for any unit vector $v\in \reals^d$
\begin{eqnarray*}
	\frac{1}{|S|}\sum_{X_i\in T}\ip{\left(X_i-\mu\right)}{v}^2\lesssim \alpha\log1/\alpha\;.
\end{eqnarray*}
For any subset $T\subset S$ such that $|T|\geq (1-2\alpha)|S|$, we have
\begin{eqnarray*}
	&&\left\|\frac{1}{|S|}\sum_{i\in T}(x_i-\mu)(x_i-\mu)^\top -{\mathbf I}\right\|_2\lesssim \alpha\log1/\alpha\; \text{ and } \;,\\
	&&\left\|\frac{1}{|S|}\sum_{i\in T}(x_i-\mu(T))(x_i-\mu(T))^\top -{\mathbf I}\right\|_2\lesssim \alpha\log1/\alpha
\end{eqnarray*}
\end{lemma}

\begin{lemma}[{\cite[Corollary 4.3]{dong2019quantum}}
 ]
 \label{lemma:gaussian_subset2}
Let $S$ be an $\alpha$-corrupted sub-Gaussian dataset under Assumption~\ref{asmp:adversary}. If $S_{\rm good}$ is $\alpha$-subgaussian good with respect to $\mu\in \reals^d$, then for any $T\subset S$ such that $|T|\leq 2\alpha|S|$, we have
\begin{eqnarray*}
	\left\|\frac{1}{|S|}\sum_{X_i\in T}\left(X_i-\mu\right)\right\|_2\lesssim \alpha\sqrt{\log1/\alpha}\;.
\end{eqnarray*}
For any subset $T\subset S$ such that $|T|\geq (1-2\alpha)|S|$, we have
\begin{eqnarray*}
	\left\|\mu(T)-\mu\right\|_2\lesssim \alpha\sqrt{\log1/\alpha}\;.
\end{eqnarray*}
\end{lemma}

\begin{lemma}[{\cite[Lemma 4.5]{dong2019quantum}} ]\label{lem:variance_lower}
	Let $S$ be an $\alpha$-corrupted sub-Gaussian dataset under Assumption~\ref{asmp:adversary}. If $S_{\rm good}$ is $\alpha$-subgaussian good with respect to $\mu\in \reals^d$, then for any $T\subset S$ such that $|T\cap S_{\rm good}|\geq (1-2\alpha)|S|$, then there is some universal constant $c_1$ such that
	\begin{eqnarray*}
		\frac{1}{|S|}\sum_{i \in T}\left(x_{i}-\mu(T)\right)\left(x_{i}-\mu(T)\right)^{\top}
	\succeq   (1-c_1\alpha\log1/\alpha){\mathbf I}\;.
	\end{eqnarray*}
\end{lemma}

\begin{lemma}[\cite{dong2019quantum} Lemma 4.6 ]
\label{thm:reg}
 Let $S$ be an $\alpha$-corrupted sub-Gaussian dataset under Assumption~\ref{asmp:adversary}. If $S_{\rm good}$ is $\alpha$-subgaussian good with respect to $\mu\in \reals^d$, then for any $T\subset S$ such that $|T\cap S_{\rm good}|\geq (1-2\alpha)|S|$, we have
\begin{eqnarray*}
	\|\mu(T)-\mu\|_{2} \leq \frac{1}{1-\alpha} \cdot \left( \sqrt{\alpha\left(\left\|M(T)-{\mathbf I}\right\|_{2}+O\left(\alpha \log 1 / \alpha\right)\right)}+O\left(\alpha \sqrt{\log 1 / \alpha}\right)\right)\;.
\end{eqnarray*}
\end{lemma}

\subsection{Auxiliary Lemmas on Laplace and Gaussian mechanism}
\begin{lemma}[Theorem A.1 in~\cite{dwork2014algorithmic}]\label{lem:gaussian-mech}
Let $\eps\in (0, 1)$ be arbitrary. For $c^2 \ge 2 \ln(1.25/\delta)$, the
Gaussian Mechanism with parameter $\sigma^2 \ge c^2\Delta_2f /\eps$ is $(\eps, \delta)$-differentially
private.
\end{lemma}

\begin{lemma}
\label{lemma:lap_noise}
Let $Y\sim {\rm Lap}(b)$. Then for all $h>0$, we have $\prob(|Y|\geq hb) = e^{-h}$.
\end{lemma}

\begin{lemma}[Tail bound of $\chi$-square distribution \cite{wainwright2019high}]
\label{lemma:chi-square}
Let $x_i\sim \cN(0, \sigma^2 )$ for $i=1,2,\cdots, d$. Then for all $\zeta \in (0, 1)$, we have $\prob(\|X\|_2\geq \sigma \sqrt{d\log(1/\zeta)}) \leq \zeta$.
\end{lemma}

\begin{lemma}[{\cite[Corollary 2.3.6]{tao2012topics}} ]
\label{lemma:matrix_sp_norm}
	Let $Z\in \reals^{d\times d}$ be a matrix such that $Z_{i,j}\sim {\cal N}(0,\sigma^2 )$ for $i \geq j$ and $Z_{i,j}=Z_{j,i}$ for $i<j$. For $\forall \zeta\in (0,1)$, then with probability $1-\zeta$ we have
		$\left\|Z\right\|_2\leq  \sigma \sqrt{d}\log(1/\zeta) $.
\end{lemma}

\begin{lemma}[Accuracy of the histogram using Gaussian Mechanism]\label{lem:hist}
Let $f: \mathcal{X}^{n} \rightarrow \reals^{\mathcal{S}}$ be a histogram over $K$ bins.  For any dataset $D\in {\cal X}^n$ and $\varepsilon$, Gaussian Mechanism is an $(\varepsilon, \delta)$-differentially private algorithm $M(D)$ such that given

with probability $1-\zeta$ we have
\begin{eqnarray*}
	\|M(D)-f(D)\|_\infty \leq  O(\frac{\sqrt{\log({K}/{\zeta})\log(1/\delta)}}{\eps n})\;.
\end{eqnarray*}
\end{lemma}
\begin{proof}
First notice that the $\ell_2$ sensitivity of histogram function $f$ is $\sqrt{2}/n$. Thus, by Lemma~\ref{lem:gaussian-mech}, by adding noise $\cN(0,(\frac{2\sqrt{2\log(1.25/\delta)}}{n\varepsilon})^2)$ to each entry of $f$, we have a $(\eps,\delta)$ differentially private algorithm. Since Gaussian tail bound implies that $\mathbb{P}_{x\sim \cN(0,\sigma^2)} [x\ge \Omega(\sqrt{\log(K/\eta)}\sigma)]\le \eta/K$, we have that with probability $1-\eta$, the $\ell_\infty$ norm of the added noise is bounded by $O(\frac{\sqrt{\log(1/\delta)\log(K/\eta)}}{n\eps})$. This concludes the proof.
\end{proof}

 \begin{lemma}[Composition theorem  of {\cite[Theorem 3.4]{composition}}]
    \label{lem:composition}
    For $\varepsilon\leq0.9$, 
    an end-to-end guarantee of $(\varepsilon,\delta)$-differential privacy is satisfied if a dataset  is accessed $k$ times, each with  a $(\varepsilon/2\sqrt{2k\log(2/\delta)},\delta/2k)$-differential private mechanism. 
 \end{lemma}

\subsection{Analysis of $\|M(S_t^{(s)})-\mathbf{I}\|_2$ shrinking}

For any symmetric matrix $A = \sum_{i=1}^d \lambda_i v_iv_i^\top$, we let $|A|$ denote $|A| = \sum_{i=1}^d |\lambda_i| v_iv_i^\top$.

\begin{lemma}[Regret bound, Special case of  {\cite[Theorem 3.1]{allen2015spectral}}]
\label{lemma:regret_bound}
Let 
$$U_t \;\;=\;\; \frac{\exp(\alpha\sum_{k=1}^ {t-1}(\Sigma_k-{\mathbf I}))}{\Tr(\exp(\alpha\sum_{k=1}^ {t-1} (\Sigma_k-{\mathbf I})))}\;,$$ 
and $\alpha$ satisfies $\alpha (\Sigma_{t+1}-{\mathbf I}) \preceq I$ for all $k\in [T]$, then for all $U\succeq 0$, $\Tr(U)=1$, it holds that
	\begin{eqnarray*}
\sum_{t=1}^T\langle  (\Sigma_{t+1}-{\mathbf I}), U - U_{t}\rangle \le \alpha \sum_{t=1}^T \langle|(\Sigma_{t+1}-{\mathbf I}), U_{t}|\rangle \cdot \|(\Sigma_{t+1}-{\mathbf I})\|_2 + \frac{\log d}{\alpha}.
	\end{eqnarray*}

Rearranging terms, and taking a supremum over $U$, we obtain that
	\begin{eqnarray*}
\|\sum_{t=1}^T (\Sigma_{t+1}-{\mathbf I})\|_2 \le \sum_{t=1}^T\langle U_{t}, (\Sigma_{t+1}-{\mathbf I})\rangle + \alpha \sum_{t=1}^T \langle|(\Sigma_{t+1}-{\mathbf I}), U_{t}|\rangle \cdot \|(\Sigma_{t+1}-{\mathbf I})\|_2 + \frac{\log d}{\alpha}.
	\end{eqnarray*}
\end{lemma}

\newpage
\section{Exponential time DP robust mean estimation of sub-Gaussian and heavy tailed distributions  (Algorithm~\ref{alg:exp})}
\label{sec:proof_exp}

In this section, we give a self-contained proof of the privacy and utility of our exponential time robust mean estimation algorithm for sub-Gaussian and heavy tailed distributions. The proof relies on the resilience property of the uncorrupted data as shown in the following lemmas.

\begin{lemma}[Lemma 10 in~\cite{steinhardt2018resilience}]\label{lem:tail_resilience}
If a set of points $\{x_i\}_{i\in S}$ lying in $\reals^d$ is $(\sigma,\alpha)$-resilient around a point $\mu$, then 
$$
\|\frac{1}{|T'|}\sum_{i\in T'}(x_i-\mu)\|_2\le \frac{2-\alpha}{\alpha}\sigma. 
$$
for all sets $T'$ of size at least $\alpha |S|$.
\end{lemma}
\begin{lemma}[Finite sample resilience of sub-Gaussian distributions {\cite[Theorem G.1]{zhu2019generalized}}]
Let $S_{\rm good}$ be a set of i.i.d.~points from a sub-Gaussian distribution $\cal{D}$ with a parameter ${\mathbf I}_d$. Given that $|S_{\rm good}|=\Omega((d+\log(1/\zeta))/(\alpha^2\log1/\alpha)\,)$, $S_{\rm good}$ is $(\alpha\sqrt{\log(1/\alpha)},\alpha)$-resilient around its mean $\mu$ with probability $1-\zeta$.
\end{lemma}

\begin{lemma}[Finite sample resilience of heavy-tailed distributions {\cite[Theorem G.2]{zhu2019generalized}}]
Let $S_{\rm good}$ be a set of i.i.d. samples drawn from distribution $\cal{D}$ whose mean and covariance are $\mu,\Sigma$ respectively, and that $\Sigma\preceq I$. Given that $|S| =\Omega(d/(\zeta \alpha))$, there exists a constant $c_\zeta$ that only depends on $\zeta$ such that  $S_{\rm good}$ is $(c_\zeta \sqrt{\alpha},\alpha)$-resilient around $\mu$ with probability $1-\zeta$.
\end{lemma}

\subsection{Case of heavy-tailed distributions and a proof of Theorem~\ref{thm:heavytail_exp}} 
\label{sec:proof_heavytail_exp}

Lemma~\ref{lem:dprange-ht} ensures that $q_{\rm range-ht}$ returns samples in a bounded support of Euclidean distance $\sqrt{d}B/2$ with $B=50/\sqrt{\alpha}$ where $(1-2\alpha)n$  samples are uncorrupted ($\alpha n$ is corrupted by adversary and $\alpha n$ can be corrupted by the pre-processing step).
For a $(c_\zeta\sqrt{3\alpha},3\alpha)$-resilient dataset, we first show that $R(S)$ is robust against corruption.

\begin{lemma}[$\alpha$-corrupted data has small $R(S)$]
Let $S$ be the set of  $2\alpha$-corrupted data. Given that $n=\Omega(d/(\zeta \alpha))$, with probability $1-\zeta$, $R(S)\le c_\zeta \sqrt{3\alpha}$. 
\end{lemma}
This follows immediately by selecting $S'$ to be the uncorrupted $(1-2\alpha)$ fraction of the dataset and applying ($c_\zeta\sqrt{3\alpha},3\alpha$)-resilience. 
After pre-processing, we have that $\|x_i-\bar{x}\|_2 \le B\sqrt{d}/2$, and then clearly $R(\cdot)$ has sensitivity $\Delta_R \le B\sqrt{d}/n$.

\begin{lemma}[Sensitivity and Privacy of ${\hat{R}(S)}$]
Given that $\hat{R}(S) = R(S)+{\rm Lap}(\frac{3B\sqrt{d}}{n\eps})$, $\hat{R}(S)$ is $(\eps/3,0)$-differentially private. Further, with probability $1-\delta/3$, $|\hat{R}(S)-R(S)|\le \frac{3B\sqrt{d}\log(3/\delta)}{n\eps}$.
\end{lemma}

In the algorithm, we first compute $\hat{R}(S)$. If $\hat{R}(S)\ge 2 c_\zeta \sqrt{\alpha}$, we stop and output $\emptyset$. Otherwise, we use exponential mechanism with score function $d(\hat\mu, S)$ to find an estimate $\hat{\mu}$. We prove the privacy guarantee of our algorithm as follows.

\begin{lemma}[Privacy]
Algorithm~\ref{alg:exp} is $(\eps,\delta)$-differentially private if $n\geq 6B\sqrt{d} \log(3/\delta)/(c_\zeta \varepsilon \sqrt{\alpha})$.
\end{lemma}

\begin{proof}
We consider neighboring datasets $S$, $S'$ under the following two scenario

\begin{enumerate}

\item $R(S) > 3 c_\zeta \sqrt{\alpha}$

In this case, given that $n\ge \frac{6B\sqrt{d}\log(3/\delta)}{c_\zeta \sqrt{\alpha}\eps}$, we have $\widehat{R}(S)>2c_\zeta\sqrt{\alpha}$ and the output of the algorithm ${\cal A}(S) =\emptyset$ with probability at least $1-\delta/3$, and ${\cal A}(S') =\emptyset$ with probability at least $1-\delta/3$. Thus, for any set $Q$, $\prob[{\cal A}(S)\in Q]\le \prob[{\cal A}(S')\in Q]+\delta/3$.

\item $R(S)\le 3c_\zeta \sqrt{\alpha}$

\begin{lemma}[Sensitivity of $d(\hat\mu,S)$]
\label{lem:sen-d}
Given that $R(S) \leq 3c_\zeta \sqrt{\alpha}$, for any neighboring dataset $S'$, $|d(\hat\mu, S) - d(\hat\mu, S')| \leq  12c_\zeta  /(n\sqrt{\alpha})$. 
\label{lem:sense_exp_ht}
\end{lemma}

In this case, the privacy guarantee of $\hat{R}(S)$ yields that $\prob[{\hat{R}}(S)\in Q]\le \exp(\eps/3)\cdot \prob[{\hat{R}}(S')\in Q]$. Lemma~\ref{lem:sen-d} yields that $\prob[{\hat{\mu}}(S)\in Q]\le \exp(\eps)\cdot \prob[{\hat{\mu}}(S')\in Q]$. A simple composition of the privacy guarantee with $q_{\rm range-ht}$($\cdot$) and the exponential mechanism gives that 
$$
\prob[({\hat{R}}(S),\hat{\mu}(S))\in Q]\le \exp(\eps)\cdot \prob[({\hat{R}}(S'),\hat{\mu}(S'))\in Q] + \delta/3
$$
This implies that $\prob[{\cal A}(S)\in Q]\le \exp(\eps)\cdot \prob[{\cal A}(S')\in Q] + \delta/3$.

\end{enumerate}
\end{proof}

\begin{lemma}[Utility of the algorithm]
\label{lem:exp_utility}
For an $2\alpha$-corrupted dataset $S$, Algorithm~\ref{alg:exp} achieves 
$\|\hat\mu - \mu^* \|_2\leq c_\zeta \sqrt{\alpha}$ with probability $1-\zeta$, if $n=\Omega(d/(\alpha\zeta) + (d \log(d/\alpha^{1.5})+\log(1/\zeta)/(\varepsilon\alpha))$.
\end{lemma}
\begin{proof}[Proof of Lemma~\ref{lem:exp_utility}]
We use the following lemma showing that $d(\hat\mu,S)$ is a good approximation of $\|\hat\mu-\mu^*\|_2$.

\begin{lemma}[$d(\mu, S)$ approximates $\|\mu-\mu^*\|$]\label{lem:d-approx}
Let $S$ be the set of $2\alpha$-corrupted data. Given that $n=\Omega(d/(\zeta\alpha))$, with probability $1-\zeta$,
$$
\big|\, d(\hat\mu, S) -\|\hat\mu-\mu^*\|_2\,\big| \;\;\leq\;\; 7c_\zeta  \sqrt{\alpha}\;.
$$
\label{lem:dist_exp_ht}
\end{lemma}
This implies that the exponential mechanism achieves the following bounds.  \begin{eqnarray*}
    {\mathbb P}(\|\hat\mu-\mu^*\|\leq c_\zeta \sqrt\alpha) &\geq&
    \frac{1}{A}e^{-\frac{\varepsilon \alpha n}{3}} \, {\rm Vol}(c_\zeta\sqrt\alpha,d) \text{, and}\\
    {\mathbb P}(\|\hat\mu-\mu^*\|\geq 22 c_\zeta \sqrt\alpha) &\leq& \frac{1}{A} e^{-\frac{5\varepsilon \alpha n}{8}} \,B^d\;,
\end{eqnarray*}
where $A$ denotes the normalizing factor for the exponential mechanism and ${\rm Vol}(r,d)$ is the volume of a ball of radius $r$ in $d$ dimensions. 
It follows that 
\begin{eqnarray*}
    \log\Big(\frac{{\mathbb P}(\|\hat\mu-\mu^*\|_2\leq c_\zeta\sqrt\alpha )}{{\mathbb P}(\|\hat\mu-\mu^*\|_2\geq 22 c_\zeta\sqrt\alpha )}\Big) &\geq& 
    \frac{7}{24}\varepsilon\alpha n - C\,d\log(dB/\alpha)\\
    &\geq& \log(1/\zeta)\;,
\end{eqnarray*}
for $n=\Omega((d\log(d/\alpha^{1.5})+\log(1/\zeta))/(\varepsilon \alpha))$.

\end{proof}

\subsubsection{Proof of Lemma~\ref{lem:sense_exp_ht}}
\label{sec:sense_exp_ht}

Since $R(S) \leq 3c_\zeta \sqrt{\alpha}$,  define $S_{\rm good}$ as  the minimizing subset in Definition~\ref{def:goodness} 
such that 
$$R(S) = \max_{T\subset S_{\rm good}, |T| = (1-\alpha)|S_{\rm good}|}\|\mu(T)-\mu(S_{\rm good})\|_2\;.$$ 
By this definition of $S_{\rm good}$ and Lemma~\ref{lem:tail_resilience}, 
$$
|v^\top(\mu(S_{\rm good} \cap {\cal T}^{v})-\mu(S_{\rm good}))|\le 6 c_\zeta \sqrt{1/\alpha}
,\text{ and }$$
$$
|v^\top(\mu(S_{\rm good} \cap {\cal B}^{v})-\mu(S_{\rm good}))|\le 6 c_\zeta \sqrt{1/\alpha}.
$$
Therefore, 
$$
\min_{i\in S_{\rm good} \cap {\cal T}^{v}} |v^\top(x_i-\mu(S_{\rm good}))| \le |v^\top(\mu(S_{\rm good} \cap {\cal T}^{v})-\mu(S_{\rm good}))| \le 6c_\zeta \sqrt{1/\alpha},
$$
and similarly
$$
\min_{i\in S_{\rm good} \cap {\cal B}^{v}} |v^\top(x_i-\mu(S_{\rm good}))| \le |v^\top(\mu(S_{\rm good} \cap {\cal B}^{v})-\mu(S_{\rm good}))| \le 6c_\zeta \sqrt{1/\alpha}
$$
This implies 
\begin{eqnarray}
\min_{i\in S_{\rm good} \cap {\cal T}^{v}}  v^\top x_i - \max_{i\in S_{\rm good} \cap {\cal B}^{v}} v^\top x_i \;\; \le \;\; 12 c_\zeta \sqrt{1/\alpha}\;.\label{eq:quantile_bound}
\end{eqnarray}
This implies that  distribution of one-dimensional points $S_{(v)}=\{v^\top x_i\}$ is dense at the boundary of top and bottom $\alpha$ quantiles, and hence cannot be changed much by changing one entry. Formally, 
consider a neighboring dataset  $S'$ (and the corresponding $S'_{(v)}$) where  
one point $x_i$ in 
${\cal M}^{(v)}(S)$ is replaced by another point $\tilde{x}_i$.
If 
$v^\top \tilde{x}_i\in [\, \max_{i\in S_{\rm good} \cap {\cal B}^{v}} v^\top x_i\,,\, \min_{i\in S_{\rm good} \cap {\cal T}^{v}} v^\top x_i\,]$, then 
Eq.~\eqref{eq:quantile_bound} implies that this only changes the mean by $6c_\zeta /(\sqrt{\alpha} n)$. Otherwise, ${\cal M}^{v}(S')$ will have $x_i$ replaced by either $\arg\min_{i\in S_{\rm good} \cap {\cal T}^{v}} v^\top x_i$ or $\arg\max_{i\in S_{\rm good} \cap {\cal B}^{v}} v^\top x_i$. In both cases, Eq.~\eqref{eq:quantile_bound} implies that this only changes the mean by $12c_\zeta/(\sqrt{\alpha}  n)$. 
The other case of when the replaced sample $x_i\in S$ is not in ${\cal M}^{v}(S)$ follows similarly. 
From this, we upper bounds the maximum difference between $S$ and $S'$ when projected on $v$, that is
$$
\left|v^\top\left(\mu({\cal M}^{v}(S))-\mu({\cal M}^{v}(S'))\right)\right|\le \frac{12c_\zeta}{\sqrt{\alpha}n} \;. 
$$
This implies the sensitivity of $d(\mu,S)$ is bounded by $6c_\zeta/ (\sqrt{\alpha}n)$:
\begin{eqnarray*}
    |d(\mu,S)-d(\mu,S')| &= & 
    \Big|\,\max_{v\in{\mathbb S}^{d-1}} \,v^\top \mu(M^{v}(S))\,-\max_{\tilde{v}\in{\mathbb S}^{d-1}}\, \tilde{v}^\top \mu(M^{v}(S'))\,\,\Big|\\
    &\leq & \max_{v\in{\mathbb S}^{d-1}} 
    \big|\,v^\top (\mu(M^{v}(S))\,-  \mu(M^{v}(S'))\,)\,\big| \;\leq\; \frac{12 c_\zeta }{\sqrt{\alpha} \,n }
\end{eqnarray*}

\subsubsection{Proof of Lemma~\ref{lem:dist_exp_ht}}
\label{sec:dist_exp_ht}
First we show $|v^\top\left(\mu({\cal M}^{v})-\mu^*\right)|\le 7 c_\zeta \sqrt{\alpha}$. Notice that $|S_{\rm good} \cap {\cal T}^{v}|\le 3\alpha |S|$, and $|S_{\rm good} \cap {\cal B}^{v}|\le 3\alpha |S|$. By the ($c_\zeta \sqrt{3\alpha},3\alpha$)-resilience property, we have $|v^\top(\mu(S_{\rm good} \cap {\cal T}^{v})-\mu^*)|\le c_\zeta \sqrt{3/\alpha}$, and $|v^\top(\mu(S_{\rm good} \cap {\cal B}^{v})-\mu^*)|\le c_\zeta \sqrt{3/\alpha}$. 
Since $|S_{\rm good}\cap{\cal M}^{v}|\ge (1-8\alpha)|S_{\rm good}|$, by the $(c_\zeta\sqrt{8\alpha},8\alpha)$-resilience property, 
$$
|v^\top(\mu(S_{\rm good}\cap{\cal M}^{v})-\mu^*)|\;\;\le\;\; c_\zeta\sqrt{8\alpha}\;.
$$ 
Since ${\cal T}^{v}$, ${\cal B}^{v}$ are the largest and smallest $3\alpha n$ points respectively and $|S_{\rm bad}|\le 2\alpha n$, we get 
$$
|v^\top(\mu(S_{bad}\cap{\cal M}^{v})-\mu^*)| \;\; \le \;\;  2c_\zeta\sqrt{3/\alpha}.
$$
Combining $S_{\rm good}\cap{\cal M}^{v}$ and $S_{\rm bad}\cap{\cal M}^{v}$ we get
\begin{align*}
&|v^\top(\mu({\cal M}^{v})-\mu^*)|\\
 &\le  \frac{|S_{bad}\cap{\cal M}^{v}|}{|{\cal M}^{v}|} |v^\top(\mu(S_{bad}\cap{\cal M}^{v})-\mu^*)| + \frac{|\mu(S_{\rm good}\cap{\cal M}^{v}|}{|{\cal M}^{v}|} |v^\top(\mu(S_{\rm good}\cap{\cal M}^{v})-\mu^*)|\\
 &\le 7c_\zeta \sqrt{\alpha}.
\end{align*}

Finally we get that 
\begin{align*}
\big|\,d(\hat\mu, S) -\|\hat\mu-\mu^*\|_2\,\big|
&\overset{(a)}{=} \left|\max_{v\in \mathbb{S}^{d-1}}\left|v^\top\left(\mu({\cal M}^{(v)})-\hat\mu\right)\right| - \max_{v\in \mathbb{S}^{d-1}}|v^\top(\hat\mu-\mu^*)|\right|\\
&\overset{(b)}{\le} \max_{v\in \mathbb{S}^{d-1}}\left|v^\top\left(\mu({\cal M}^{(v)}) -\mu^*\right)\right|\\
&\le 7c_\zeta \sqrt{\alpha},
\end{align*}
where $(a)$ holds by the definition of the distance :
$$
\|\mu-\mu^*\|_2\; =\; \max_{v\in \mathbb{S}^{d-1}}|v^\top(\mu-\mu^*)|,
$$
and $(b)$ holds by triangle inequality.

\subsection{Case of sub-Gaussian distributions and a proof of Theorem~\ref{thm:subgauss}} 
\label{sec:proof_subgaussian}

Th proof is analogous to the previous section, we only state the lemmas that differ. $q_{\rm range}$ returns a hypercube $\bar{x}+[-B/2,B/2]^d$ that includes all uncorrupted data points with a high probability.  

\begin{lemma}[$\alpha$-corrupted data has small $R(S)$]
Let $S$ be the set of  $\alpha$-corrupted data. Given that $n=\Omega(\frac{d+\log(1/\zeta )}{\alpha^2\log1/\alpha})$, with probability $1-\zeta$, $R(S)\le 3\,\alpha\sqrt{\log(1/3\alpha)}$. 
\end{lemma}

\begin{lemma}[Privacy]
Algorithm~\ref{alg:exp} is $(\eps,\delta)$-differentially private if $n\geq 3B\sqrt{d} \log(3/\delta)/( \varepsilon \alpha\sqrt{\log(1/\alpha)})$.
\end{lemma} 
This follows from the following lemma. 
\begin{lemma}[Sensitivity of $d(\hat\mu,S)$]
Given that $R(S) \leq 3 \alpha \sqrt{\log(1/\alpha)}$, for any neighboring dataset $S'$, $|d(\hat\mu, S) - d(\hat\mu, S')| \leq  12\sqrt{\log1/\alpha} /n$. 
\end{lemma}

\begin{lemma}[$d(\hat\mu, S)$ approximates $\|\hat\mu-\mu^*\|$]
\label{lem:d-approx-sg}
Let $S$ be the set of $\alpha$-corrupted data. Given that $n=\Omega(\frac{d+\log(1/\zeta)}{\alpha^2\log1/\alpha})$, with probability $1-\zeta$,
$$
\big|\, d(\hat\mu, S) -\|\hat\mu-\mu^*\|_2\,\big| \;\; \le \;\; 14\,\alpha\, \sqrt{\log1/\alpha}\;.
$$
\end{lemma}
This implies the following utility bound. 
\begin{lemma}[Utility of the algorithm]
For an $\alpha$-corrupted dataset $S$, Algorithm~\ref{alg:exp} achieves 
$\|\hat\mu - \mu^* \|_2\leq \alpha \sqrt{\log1/\alpha}$ with probability $1-\zeta$, if $n=\Omega((d+\log(1/\zeta))/(\alpha^2\log(1/\alpha)) + (d \log(d\sqrt{\log(dn/\zeta)}/\alpha)+\log(1/\zeta)/(\varepsilon\alpha))$.
\end{lemma}

\section{Background on exponential time approaches for Gaussian distributions} 
\label{sec:tukey}

In this section, we provide a background on exponential time algorithms 
that achieve optimal guarantees but only applies to and heavily relies on the assumption that samples are drawn from a {\em Gaussian} distribution. 
In \S\ref{sec:subgauss},  we introduce 
a novel exponential time approach  that seamlessly generalizes to both sub-Gaussian and covariance-bounded distributions.

We introduce 
Algorithm~\ref{alg:tukey}, achieving  
the optimal sample complexity of $\widetilde{O}(d/\min\{\alpha\varepsilon,\alpha^2\} )$ (Theorem~\ref{thm:tukey}). 
The main idea is to find an approximate Tukey median (which is known to be a robust estimate of the mean \cite{zhu2020does}), 
using the exponential mechanism of  \cite{mcsherry2007mechanism} to preserve privacy. 

\medskip
\noindent 
{\bf Tukey median set.} 
For any set of points $S =\{x_i\in{\mathbb R}^d\}_{i=1}^n $ and $\hat\mu\in{\mathbb R}^d$, the 
{\em Tukey depth} is defined as 
the minimal empirical probability density on one side of a hyperplane that includes $\hat\mu$: 
\begin{eqnarray*}
    D_{\rm Tukey}(S,\hat\mu)\;=\;  \inf_{v\in{\mathbb R}^d} 
    {\mathbb P}_{x\sim \hat{p}_n}(v^\top(x-\hat\mu) \geq 0)\;,
    \label{eq:TukeyDepth}
\end{eqnarray*}
where $\hat{p}_n$ is the empirical distribution of $S$. The {\em Tukey median set} is defined as the set of points achieving  the maximum Tukey depth, which might not be unique.    
Tukey median reduces to median for $d=1$, and is a natural generalization of the median for $d>1$. Inheriting robustness of one-dimensional median, Tukey median is known to be a robust estimator of the multi-dimensional mean under an adversarial perturbation. In particular, under our model, it achieves  the optimal sample complexity and accuracy. 
This optimality  follows from the well-known fact that 
   the sample complexity of  $O((1/\alpha^2)({d+\log(1/\zeta))})$ cannot be improved upon even if we have no corruption, 
   and the fact that the accuracy of $O(\alpha)$ cannot be improved upon even if we have infinite samples \cite{zhu2020does}.
   However, finding a Tukey median takes exponential time scaling as $\tilde O(n^d)$  \cite{liu2019fast}. 
\begin{coro}[Corollary of {\cite[Theorem 3]{zhu2020does}}]
    \label{coro:tukey}
    For a dataset of $n$ i.i.d.~samples from a $d$-dimensional Gaussian distribution ${\cal N}(\mu,{\mathbf I}_d)$, an adversary corrupts an  $\alpha\in(0,1/4)$ fraction of the samples as defined in  Assumption~\ref{asmp:adversary}.
    Then, any  $\hat\mu$ in the Tukey median set of a corrupted  dataset $S$ satisfies  
    $\|\hat\mu - \mu\|_2 = O(\alpha)  $ with probability at least $1-\zeta$ if $n=\Omega((1/\alpha^2)(d+\log(1/\zeta)))$.
\end{coro}
     
  
  
\medskip
\noindent
{\bf Exponential mechanism.}
The exponential mechanism was introduced in 
\cite{mcsherry2007mechanism} 
to elicit approximate truthfulness and remains one of the most popular private mechanisms due to its broad applicability. 
It can seamlessly handle queries with non-numeric outputs, such as routing a flow or finding a graph.   
Consider a utility  function $u(S,\hat\mu)\in{\mathbb R}$ on a dataset $S$ and a variable $\hat\mu$, where higher utility is preferred. Instead of truthfully outputting $\arg\max_{\hat \mu} u(S,\hat\mu)$, the exponential mechanism outputs a randomized approximate maximizer sampled from the following  distribution:  
\begin{eqnarray}
    \label{def:exp}
    r_{S}(\hat\mu) \;=\; \frac{1}{Z_S} e^{\frac{\varepsilon}{2\,\Delta_u} u(S,\hat\mu)}\;, 
\end{eqnarray}
where $\Delta_u=\max_{\hat\mu, S\sim S'}|u(S,\hat\mu) - u(S',\hat\mu)|$ is the sensitivity of $u$ (from Definition~\ref{def:output}) and $Z_{S}$ ensures normalization to one.  This mechanism is $(\varepsilon,0)$-differentially private, since  
$e^{\frac{\varepsilon}{2\Delta_u}|u(S,\hat\mu)-u(S',\hat\mu)|} \leq e^{{\varepsilon}/{2}}$ and $e^{-\varepsilon/2} \leq Z_{S}/Z_{S'}\leq e^{\varepsilon/2}$.

\begin{propo}[{\cite[Theorem 6]{mcsherry2007mechanism}} ]
    \label{pro:exp} 
    The sampled $\hat\mu$ from the distribution  \eqref{def:exp} is $(\varepsilon,0)$-differentially private. 
\end{propo}
This naturally leads to the  following algorithm. The privacy guarantee follows immediately since the Tukey depth has sensitivity $1/n$, i.e., $|\, D_{\rm Tukey}(S_n, \hat\mu ) - D_{\rm Tukey}(S'_n,\hat\mu )\, | \leq 1/n$ for all $\hat\mu \in{\mathbb R}^d$ and two neighboring databases $S_n \sim S_n'$ of size $n$. In this section, for the analysis of private Tukey median, we assume the mean is from a known bounded set of the form $[-R,R]^d$ for some known $R>0$.

\begin{algorithm}[ht]
   \caption{Private Tukey median}
   \label{alg:tukey}
   	\DontPrintSemicolon 
	\SetKwProg{Fn}{}{:}{}
	{
	Output a random data point $\hat\mu\in[-2R,2R]^d$ sampled from a density 
	$ r(\hat\mu) \;\propto \;e^{(1/2)\varepsilon n D_{\rm Tukey}(S,\hat\mu)}\;.$
	}
\end{algorithm}

The private Tukey median achieves the following near optimal guarantee, whose proof is provided in \S\ref{sec:proof_tukey}. 
  The accuracy of $O(\alpha)$
 and sample complexity of $n=\Omega((1/\alpha^2)(d+\log(1/\zeta)))$  cannot be improved even without privacy (cf.~Corollary~\ref{coro:tukey}), and 
 $n=\tilde\Omega(d/(\alpha\varepsilon))$ 
 is necessary even without any corruption \cite[Theorem 6.5]{KLSU19}. 
\begin{thm} 
    \label{thm:tukey}
    Under the hypotheses of 
    Corollary~\ref{coro:tukey}, 
    there exists a universal  constant $c>0$ such that 
    if $\mu\in[-R,R]^d$, $\alpha \leq \min\{c,R\}$ and $n=\Omega( (1/\alpha^{2}) (d+\log(1/\zeta)) + 
    (1/\alpha \eps )d\log(dR/\zeta\alpha))$,  
    then  Algorithm~\ref{alg:tukey} is $(\eps,0)$-differentially private and achieves 
    $ \| \hat\mu -\mu \|_2 = O(\alpha) $ with probability $1-\zeta$. 
\end{thm}
The private Tukey median, however, is a conceptual algorithm since we cannot sample from $r(\hat\mu)$. 
The ${\cal A}_{FindTukey}$ algorithm from  
\cite{beimel2019private}  approximately finds the Tukey median  privately. This achieves $O(\alpha)$ accuracy with $n=\tilde\Omega(d^{3/2}\log(1/\delta)/(\alpha \varepsilon)+(1/\alpha^2)(d+\log(1/\zeta)))$, but it still requires a 
runtime of $O(n^{{\rm poly}(d)} )$. 
Alternatively, we can sample from an  $\alpha$-cover of $[-2R,2R]^d$, which has $O((dR/\alpha)^d)$ points. 
However, evaluating the Tukey depth of a point is an NP-hard problem \cite{amaldi1995complexity}, 
requiring
a runtime of $\tilde{O}(n^{d-1})$ \cite{liu2017fast}.  
The runtime of the discretized private Tukey median is $\tilde{O}(n^{-1}(dnR/\alpha)^d)$.
Similarly, \cite{bun2019private} introduced an exponential mechanism over the $\alpha$-cover with a novel  utility function  
achieving the  same guarantee  as  Theorem~\ref{thm:tukey}, but this requires a runtime of $O(n(d R/\alpha)^{2d})$. 
\section{Proof  of Theorem~\ref{thm:tukey} on the accuracy of the exponential mechanism for Tukey median}
\label{sec:proof_tukey}

First, the $(\varepsilon,0)$-differential privacy guarantee of private Tukey  median follows as a corollary of Proposition~\ref{pro:exp}, by noting that  
sensitivity of $n\,D_{\rm Tukey}({\cal D}_n, x)$ is one, where ${\cal D}_n$ is a dataset of size $n$. 
This follows from the fact that for any fixed $x$ and $v$, 
$|\{z \in {\cal D}_n: (v^\top(x-z))\geq 0\}|$ is the number of samples on one side of the hyperplane, which can change at most by one if we change one sample in ${\cal D}$.  

Next, given $n$ i.i.d samples $X_1, X_2, \ldots X_n$ from distribution $p$, denote $\hat{p}_n$ as the empirical distribution defined by the samples $X_1, X_2, \ldots X_n$. Denote $\tilde{p}_n$ as the distribution that is corrupted from $\hat{p}_n$. 
We slightly overload the definition of Tukey depth to denote $D_{\text{Tukey}}(p, x)$ as the Tukey depth of point $x\in \reals^d$ under distribution $p$, which is defined as
$$
D_{{\tukey}}(p,x) = \inf_{v\in \reals^d} \prob_{z\sim p}(v^\top(x-z)\ge 0).
$$
Note that this is the standard definition of Tukey depth. 
First we show that for $n$ large enough, the Tueky depth for the empirical distribution is close to that of the true distribution.  We provide proofs of the following lemmas later in this section. 

\begin{lemma}\label{lemma:tukey-depth-con}
With probability $1-\delta$, for any $p$ and $x\in \reals^d$,
$$
|D_{\tukey}(p, x) - D_{\tukey}(\hat{p}_n, x)|\le C\cdot \sqrt{\frac{d+1+\log(1/\delta)}{n}}.
$$
\end{lemma}
The proof of Lemma~\ref{lemma:tukey-depth-con}can be found in \S\ref{sec:tukey-depth-con-proof}. This allows us to use the known Tukey depths of a Gaussian distribution to bound the Tukey depths of the corrupted empirical one. We use this to show that 
there is a strict separation between the Tueky depth of a point in $S_1=\{x:\|x-\mu\|\leq \alpha\}$ and a point in $S_2=\{x:\|x-\mu\|\geq 10\alpha \}$. The proof of Lemma~\ref{lemma:tukey-depth-Gauss} can be found in \S\ref{sec:tukey-depth-Gauss-proof}.

\begin{lemma}
\label{lemma:tukey-depth-Gauss}
Define $p = \mathcal{N}(\mu, I)$, and assume $\alpha<0.01$.
Given that $n = \Omega(\alpha^{-2}(d+\log(1/\delta)))$, with probability $1-\delta$, 
\begin{enumerate}
\item For any point $x\in \reals^d$, $\|x-\mu\|\le \alpha$, it holds that  
$$D_\tukey(\tilde{p}_n, x)\ge \frac{1}{2}-2\alpha
$$
\item For any point $x\in \reals^d$, $\|x-\mu\|\ge 10\alpha$, it holds that 
$$
D_\tukey(\tilde{p}_n, x)\le \frac{1}{2}- 5 \alpha.
$$
\end{enumerate}
\end{lemma}
This implies that most of the probability mass of the exponential mechanism is concentrated  inside a ball of radius $O(\alpha)$ around the true mean $\mu$. Hence, with high probability, the exponential mechanism outputs an approximate mean that is $O(\alpha)$ close to the true one. 
The following lemma finishes the proof the the desired claim, whose proof can be found in \S\ref{sec:tukey-depth-utility-proof}.

\begin{lemma}[Utility]
\label{lemma:tukey-depth-utility} 
Denote $\tilde{p}_n$ as the distribution that is corrupted from $\hat{p}_n$. Suppose $x$ is sampled from $[-2R, 2R ]^d$ with density $r(x)\propto \exp(-(1/2)\varepsilon n D_{\tukey}(\tilde{p}_n, x))$, then given $n = \Omega(\,(d/(\alpha\eps))\log(dR/\eta\alpha)+(1/\alpha^2)(d+\log(1/\eta ))\,)$ and $\mu\in [-R,R]^d$, and $R \geq \alpha$, 
$$
\prob({\|x-\mu\|\le 5\alpha})\ge 1-\eta\;. 
$$
\end{lemma}

\subsection{Proof of Lemma~\ref{lemma:tukey-depth-con}}\label{sec:tukey-depth-con-proof}

From the VC inequality (\cite{devroye2012combinatorial}, Chap 2, Chapter 4.3) and the fact that the family of sets
$\{\{z|v^\top z \ge t\} | \|v\| = 1, t \in \reals, v \in \reals^d\}$ has VC dimension $d + 1$, there exists some universal constant $C$ such that with probability at least $1 - \delta$
$$
\sup_{t\in \reals,v\in \reals^d, \|v\|=1} |\prob_{z \sim p}(v^\top z\ge t)- \prob_{z\sim \hat{p}_n}(v^\top z\ge t)| \le C\cdot \sqrt{\frac{d+1+\log(1/\delta)}{n}},
$$
which implies, for any $x\in \reals^d$,
$$
\sup_{v\in \reals^d} |\prob_{z\sim p}(v^\top (x-z)\ge 0)- \prob_{z\sim \hat{p}_n}(v^\top (x-z)\ge 0)| \le C\cdot \sqrt{\frac{d+1+\log(1/\delta)}{n}},
$$
by letting $t = v^\top x$. We conclude the proof since
\begin{gather*}
|D_\tukey(p,x)-D_\tukey(\hat{p}_n,x)|\\
= |\inf_{v\in \reals^d} \prob_{z\sim p}(v^\top(x-z)\ge 0) - \inf_{v\in \reals^d} \prob_{z\sim \hat{p}_n}(v^\top(x-z)\ge 0)|\\
\le \sup_{v\in \reals^d} |\prob_{z\sim p}(v^\top (x-z)\ge 0)- \prob_{\hat{p}_n}(v^\top (x-z)\ge 0)|\\
\le C\cdot \sqrt{\frac{d+1+\log(1/\delta)}{n}}.
\end{gather*}

\subsection{Proof of Lemma~\ref{lemma:tukey-depth-Gauss}}\label{sec:tukey-depth-Gauss-proof}

For the first claim, we first prove a lower bound on $D_{\tukey}(p,x)$. Since $p = \mathcal{N}(\mu, I)$, for any $v\in \reals^d$ such that $\|v\|_2 = 1$,  
\begin{align*}
\prob_{z\sim p}(v^\top (z-x)\ge 0)\\
= \prob_{z\sim N(0,1)}(z \ge v^\top (x-\mu))\\
= \int_{v^\top(x-\mu)}^\infty \frac{1}{\sqrt{2\pi}}\exp(-z^2/2)dz\\
\ge \frac{1}{2} - \frac{1}{\sqrt{2\pi}}v^\top(x-\mu)\\
\ge \frac{1}{2} - \frac{1}{\sqrt{2\pi}}\|x-\mu\|_2\\
\ge \frac{1}{2} - \frac{1}{\sqrt{2\pi}}\alpha
\end{align*}
Thus,  
\begin{align*}
D_{\tukey}(p,x)\\
= \inf_{v\in \reals^d}\prob_{z\sim p}(v^\top (x-z)\ge 0)\\
\ge \frac{1}{2} - \frac{1}{\sqrt{2\pi}}\alpha
\end{align*}

Then Lemma~\ref{lemma:tukey-depth-con} implies that with probability $1-\delta$
$$
D_\tukey(\hat{p}_n, x) \ge \frac{1}{2} - \frac{1}{\sqrt{2\pi}}\alpha - C\cdot \sqrt{\frac{d+1+\log(1/\delta)}{n}}.
$$
Since the corruption can change at most $\alpha$ probability mass, it holds that $|D_{\tukey}(\tilde{p}_n,x)-D_{\tukey}(\hat{p}_n,x)|\le \alpha$. Setting $n = \Omega(\alpha^{-2}(d+\log(1/\delta)))$ yields 
\begin{align*}
D_\tukey(\tilde{p}_n, x) \ge \frac{1}{2} - \frac{1}{\sqrt{2\pi}}\|x-\mu\|_2 - C\cdot \sqrt{\frac{d+1+\log(1/\delta)}{n}} - \alpha\\
\ge \frac{1}{2} - 2\alpha.
\end{align*}

For the second claim, note that
\begin{eqnarray*}
&&D_{\tukey}(p, x)\\
\le &&\int_{v^\top(x-\mu)}^\infty \frac{1}{\sqrt{2\pi}}\exp(-z^2/2)dz\\
\overset{(a)}{\le}&& \frac{1}{2} - \frac{1}{\sqrt{2\pi}} \exp(-(20\alpha)^2/2)\cdot 20\alpha\\
\overset{(b)}{\le}&& \frac{1}{2} -  7 \alpha
\end{eqnarray*}
 where (a) holds since $\|x-\mu\|\ge 20 \alpha$, and it is easy to verify that (b) holds for $\alpha\le 0.01$.
The second claim holds since
\begin{eqnarray*}
&& D_{\tukey}(\tilde{p}_n, x)\\
\le &&D_{\tukey}(\hat{p}_n, x)+\alpha\\
\le &&D_{\tukey}(p, x)+\alpha + C\cdot \sqrt{\frac{d+1+\log(1/\delta)}{n}}\\
\overset{(a)}{\le} &&D_{\tukey}(p, x) + 2\alpha\\
\le &&\frac{1}{2} - 5\alpha,
\end{eqnarray*}
where $(a)$ holds by setting $n = \Omega(\alpha^{-2}(d+\log(1/\delta)))$.


\subsection{Proof of Lemma~\ref{lemma:tukey-depth-utility}}\label{sec:tukey-depth-utility-proof}

Let $r(x) = \frac{1}{A} \exp(-\eps n D_{\tukey}(\tilde{p}_n, x))$ where $A$ is the normalization factor. 
Then 
$$
\prob({\|x-\mu\|\le \alpha}) \ge \frac{1}{A}\exp(\eps n(\frac{1}{2}-2\alpha))\cdot \frac{\pi^{d/2}}{\Gamma(d/2+1)}\alpha^d,
$$
using the fact that $\mu\in[-R,R]^d$ and that $R\geq\alpha$, and
$$
\prob({\|x-\mu\|\ge 5\alpha}) \le \frac{1}{A}\exp(\eps n (\frac{1}{2}-10\alpha))\cdot {(4R)}^d.
$$
Hence 
$$
\log(\frac{\prob({\|x-\mu\|\le \alpha})}{ \prob({\|x-\mu\|\ge 5\alpha})}) \ge \eps n(3\alpha) - C\cdot d\log(d R/\alpha),
$$
where $C$ is an absolute constant.
If we set $n = \Omega(\frac{d\log(dB/\delta\alpha)}{\alpha\eps})$, we get that 
$$
\frac{\prob({\|x-\mu\|\le \alpha})}{ \prob({\|x-\mu\|\ge 5\alpha})} \ge \frac{10}{\delta},
$$
which implies that with probability at least $1-\delta$, $\|x-\mu\|\le 5\alpha$.

\newpage
\section{The algorithmic details and the analysis of {\sc PRIME-ht} for covariance bounded distributions}
\label{sec:proof_heavytail}

We provide the algorithm and the analysis for the range estimation query $q_{\rm range-ht}$, and then prove the result on analyzing {\sc PRIME-ht}.

\subsection{Range estimation with $q_{\rm range-ht}$}

\begin{algorithm2e}[ht]
   \caption{Differentially private range estimation for covariance bounded distributions ($q_{\rm range-ht}$) {\cite[Algorithm 2]{kamath2020private}} }
   \label{alg:DPrange-ht}
   \DontPrintSemicolon 
   	\KwIn{$S  =  \{ x_{i} \}_{i=1}^{n} $, $\varepsilon$, $\delta$, $\zeta$ }
   \SetKwProg{Fn}{}{:}{}
   {
   Randomly partition the dataset $S=\cup_{\ell\in[m]}S^{(\ell)}$ with $m=200\log(2/\zeta)$\;
   $\bar{x}^{(\ell)} \gets q_{\rm range}(S^{(\ell)},\varepsilon/m,\delta/m,\sigma=40$) for all  $\ell\in[m]$\;
	$\hat x_j \leftarrow {\rm median}( \{\bar{x}_j^{(\ell)}\}_{\ell\in[m]}) $ for all $j\in[d]$ \;
		\KwOut{$(\hat{x},B=50/\sqrt{\alpha})$}
	}
\end{algorithm2e} 

\begin{lemma}
    \label{lem:dprange-ht}
    $q_{\rm range-ht}$ is $(\varepsilon,\delta)$-differentially private.
    Under Assumption \ref{asmp:adversary2} and for $\alpha\in(0,0.01)$, 
    if $n=\Omega((1/\alpha)\log(1/\zeta)+(\sqrt{d\log(1/\delta)}\log(1/\zeta)\log(d/\delta)/\varepsilon) )$, $q_{\rm range-ht}$ returns a ball ${\cal B}_{\sqrt{d}B/2}(\bar{x})$ of radius $\sqrt{d}B/2$ centered at $\bar{x}$ that includes $(1-2\alpha)n$ uncorrupted samples 
    where $B=50/\sqrt{\alpha}$ with probability  $1-\zeta$.
\end{lemma}

We first show that applying the private histogram to each coordinate provides a robust estimate of the range, but with a constant probability 0.9.

\begin{lemma}[Robustness of a single private histogram]
    Under the $\alpha$-corruption model of  Assumption~\ref{asmp:adversary2}, if $n=\Omega(\sqrt{d\log(1/\delta)}\log(d/\delta)/\varepsilon )$,  for $\alpha\in(0,0.01)$,  $q_{\rm range}$ in Algorithm~\ref{alg:DPrange} with a choice of $\sigma=40$ and $B=120$ returns  intervals $\{I_j\}_{j=1}^d$ of  size $|I_j|=240$ such that $\mu_j \in I_j$  with probability 0.9 for each $j\in[d]$. 
    \label{lem:hist_single-ht}
\end{lemma}
\begin{proof}[Proof of Lemma~\ref{lem:hist_single-ht}]
The proof is analogous to Appendix~\ref{sec:proof_dprange_lemma} and we only highlight the differences here. 
By Lemma~\ref{lem:hist-KV17} we know that 
$|\tilde p_k-\hat p_k|\leq 0.01$ with the assumption on $n$. 
The corruption can change the normalized count in each bin by $\alpha\leq 0.01$ by assumption.  
It follows from
Chebyshev inequality that 
${\mathbb P}(|x_{i,j}-\mu_j|^2 >  \sigma^2)\leq 1/\sigma^2$. 
It follows from (e.g.~\cite[Lemma A.3]{kamath2020private}) that 
${\mathbb P}(|\{i:x_{i,j}\notin[\mu-\sigma,\mu+\sigma]\}|> (100/\sigma^2)n) < 0.05$. 
Hence the maximum bin has $\tilde p_k \geq 0.5(1-100/\sigma^2)-0.02$ and the true mean is in the maximum bin or in an adjacent bin. 
The largest non-adjacent bucket is at most $100/\sigma^2 + 0.02$.
Hence, the choice of $\sigma=40$ ensures that we find the $\mu$ within $3\sigma=120$.

\end{proof} 

Following 
\cite[Algorithm 2]{kamath2020private}, we partition the dataset into $m=200\log(2/\zeta)$ subsets of an equal size $n/m$ and apply the median-of-means approach. Applying Lemma~\ref{lem:hist_single-ht},  it is ensured  (e.g., by \cite[Lemma A.4]{kamath2020private}) that more than half of the partitions satisfy that the center of the interval is within 240 away from $\mu$, 
with probability $1-\zeta$. 
Therefore the median of those $m$ centers is within $240$ from the true mean in each coordinate. This requires the total sample size larger only by a factor of $\log(d/\zeta)$. 

To choose a radius $\sqrt{d} B/2$ ball around this estimated mean that includes $1-\alpha$ fraction of the points, we choose $B=25/\sqrt\alpha$.
Since $\|\hat\mu-\mu\|_2 \leq 120\sqrt{d} \ll \sqrt{d}B/2$ for $\alpha\leq0.01$, this implies that we can choose 
$\sqrt{d} B/2$-ball around the estimated mean with $B=50/\sqrt\alpha$.  

Let $z_i = {\mathbb I}(  \|x_i-\mu\|_2 > \sqrt{d} B/2)$. 
We know that ${\mathbb E}[z_i] = {\mathbb P}[(  \|x_i-\mu\|_2 > \sqrt{d} B/2)] \leq {\mathbb E}[\|x_i-\mu\|_2^2 (2/dB^2)] = (1/1250)\alpha$. Applying multiplicative Chernoff bound (e.g., in \cite[Lemma A.3]{kamath2020private}), we get $|\{i:\|x_i-\mu\|_2\leq \sqrt{d} B/2\}| \geq 1-(3/2500)\alpha$ with probability $1-\zeta$, if $n=\Omega((1/\alpha)\log(1/\zeta))$.
This ensures that with high probability, $(1-\alpha)$ fraction of the original uncorrupted points are included in the ball.
Since the adversary can corrupt $\alpha n$ samples, at least $(1-2\alpha)n$ of the remaining good points will be inside the ball.

\subsection{Proof of Theorem~\ref{thm:heavytail_poly} }

\label{sec:proof_heavytail_poly} 

The proof of the privacy guarantee of 
Algorithm~\ref{alg:DPMMWfilter_ht}
follows analogously from the proof of the privacy of PRIME and is omitted here. 
The accuracy guarantee follows form the following theorem and Lemma~\ref{lem:dprange-ht}.

\begin{thm}[Analysis of accuracy of  {\sc DPMMWfilter-ht}] 
\label{thm:accuracy_mmw_ht}
	Let $S$ be an $\alpha$-corrupted covariance bounded dataset under Assumption~\ref{asmp:adversary2}, where $\alpha\leq c$ for some universal constant $c\in (0,1/2)$. Let $S_{\rm good}$ be $\alpha$-good with respect to $\mu\in \reals^d$. Suppose  ${\cal D} = \{ x_{i} \in {\cal B}_{\sqrt{d}B/2}(\bar{x})\}_{i=1}^{n}$ be the projected dataset. If $
	n \geq  \widetilde\Omega\left(\frac{d^{3/2}B^2\log(1/\delta)}{\varepsilon }\right)$, then {\sc DPMMWfilter-ht} terminates after at most $O(\log dB^2)$ epochs and outputs $S^{(s)}$ such that with probability $0.9$, we have $|S_t^{(s)}\cap S_{\mathrm{good}}|\geq (1-10\alpha )n$ and
	\begin{eqnarray*}
		\|\mu(S^{(s)})-\mu\|_2\lesssim \sqrt{\alpha}\;.
	\end{eqnarray*}
	Moreover, each epoch runs for at most $O(\log d)$ iterations.
\end{thm}

\begin{algorithm2e}[ht]
  \caption{Differentially private filtering with matrix multiplicative weights ({\sc DPMMWfilter-ht}) for distributions with bounded covariance }
  \label{alg:DPMMWfilter_ht}
  \DontPrintSemicolon 
  \KwIn{$ S  =  \{ x_{i} \in {\cal B}_{\sqrt{d}  B/2}(\bar{x})\}_{i=1}^{n} $, $\alpha \in(0,1)$, $T_1 = O(\log B\sqrt{d}), T_2 = O(\log d)$,  $B\in{\mathbb R}_+$,  $(\varepsilon,\delta)$  
  } 
  \SetKwProg{Fn}{}{:}{}
  { 
    {\bf if} $n < (4/\varepsilon_1)\log(1/(2\delta_1))$  {\bf then Output:} $\emptyset$ \;
    Initialize $S^{(1)} \leftarrow [n]$,  $\varepsilon_1 \leftarrow \varepsilon /(4T_1)$, $\delta_1\leftarrow \delta/(4T_1) $, $\varepsilon_2 \leftarrow \min\{0.9,\varepsilon\}/(4\sqrt{10T_1 T_2\log(4/\delta)})$, $\delta_2\leftarrow \delta/(20T_1T_2)$, a large enough constant $C
    >0$\; 
	\For{{\rm epoch} $s=1,2,\ldots, T_1$}
	{
      $ \textcolor{black}{\lambda^{(s)}} \leftarrow \| M(S^{(s)})\|_2 + {\rm Lap}(2B^2d/(n\varepsilon_1))$ \;
      $ n^{(s)} \leftarrow |S^{(s)}| + {\rm Lap}(1/\varepsilon_1)$ \;
      {\bf if }$n^{(s)}\leq 3n/4$ {\bf then} {\rm terminate}\;
      \If
      {$ \lambda^{(s)}\leq C$ }{
        \KwOut
        { $\textcolor{black}{\mu^{(s)}} \leftarrow   (1/|S^{(s)}|) \big( \sum_{i\in S^{(s)} } x_i\big)  + \cN(0,(2B\sqrt{2d\log(1.25/\delta_1)}/({n\, \varepsilon_1 }))^2{\mathbf I}_{d\times d}) $
        }
      } 
      $\alpha^{(s)} \leftarrow 1/(100(0.1/C+1.05)\lambda^{(s)})$\;
      $S^{(s)}_1 \leftarrow  S^{(s)}$\;
	  \For{$t=1,2,\ldots,T_2$}
	  {
	      $\textcolor{black}{\lambda_t^{(s)}} \leftarrow \| M(S_t^{(s)})\|_2  + {\rm Lap}(2B^2 d/(n\varepsilon_2))$\;
	     \eIf{$\lambda_t^{(s)}\leq 2/3 \lambda_0^{(s)}$}
	     { 
	       {\rm terminate epoch}
	     }{
	        $\textcolor{black}{\Sigma_t^{(s)}} \leftarrow  M(S_{t}^{(s)}) + \cN(0,(4B^2d\sqrt{2\log(1.25/\delta_2)}/(n\varepsilon_2))^2{\mathbf I_{d^2\times d^2}}) $ \;
	        $U_t^{(s)} \leftarrow (1/\Tr(\exp( \alpha^{(s)}\sum_{r=1}^{t}(\Sigma_r^{(s)} ))))\exp( \alpha^{(s)}\sum_{r=1}^{t}(\Sigma_r^{(s)})) $ \;
	        $\textcolor{black}{\psi_t^{(s)}} \leftarrow \ip{M(S_t^{(s)})}{U_t{^{(s)}}} + {\rm Lap}(2B^2d  /( n \varepsilon_2))$\;
	        \eIf{$\psi_t^{(s)}\leq (1/5.5)\lambda_t^{(s)}$}
	        {
	          $S_{t+1}^{(s)}\leftarrow S_t^{(s)}$ 
	        }
	        {
              $Z_t^{(s)} \leftarrow {\rm Unif}([0,1])$\;
              $\textcolor{black}{\mu_t^{(s)}}  \leftarrow  (1/|S_{ t }^{(s)}|) \big( \sum_{i\in S_{t}} x_i\big)  + \cN(0,(2B\sqrt{2d\log(1.25/\delta_2)}/({n\, \varepsilon_2 }){\mathbf I}_{d\times d})^2) $\;
              $\textcolor{black}{\rho_t^{(s)}} \leftarrow \text{\sc DPthreshold-ht}(\mu_t^{(s)},U_t^{(s)}, \alpha,\varepsilon_2,\delta_2,S^{(s)}_t)$\hfill [Algorithm~\ref{alg:1Dfilter_ht}] \\
              $S_{t+1}^{(s)} \leftarrow S_{t}^{(s)} \setminus$  $\{ i  \,|\,$     $\{\tau_j = (x_j-\mu_t^{(s)})^\top U_t^{(s)} (x_j-\mu_t^{(s)})\}_{j\in S_{t}^{(s)}} $ and $\tau_i \geq \rho_t^{(s)} \,Z_t^{(s)} \}$.
            } 
	      } 
	    } 
        $S^{(s+1)} \leftarrow S^{(s)}_t$\;
	  } 
	  \KwOut{$\mu^{(T_1)}$} 
    }
\end{algorithm2e}

\begin{algorithm2e}[ht]
   \caption{Differentially private estimation of the threshold for bounded covariance  {\sc  DPthreshold-ht} } 
   \label{alg:1Dfilter_ht}
   	\DontPrintSemicolon 
	\KwIn{
	$\mu$, $U$, $\alpha \in(0,1)$, target privacy $(\varepsilon,\delta)$, $S=\{x_i\in{\cal B}_{B\sqrt{d}/2}({\bar x}) \}$  } 
	\SetKwProg{Fn}{}{:}{}
	{ 
	    Set $ \tau_i \leftarrow (x_i - \mu)^\top U (x_i-\mu)$ for all $i\in S$\; 
	    Set $\tilde\psi \leftarrow (1/n)\sum_{i\in S} \tau_i + {\rm Lap}(2B^2d/n\varepsilon))$\;
	    Compute a histogram over geometrically sized bins $I_1=[1/4,1/2),I_2=[1/2,1),\ldots,I_{2+\log(B^2d)}=[2^{\log(B^2d)-1},2^{\log(B^2d)}]$
	    $$h_j \leftarrow \frac{1}{n}\cdot |\{i\in S\,|\, \tau_i \in [2^{-3+j},2^{-2+j})\}| \;,\;\;\;\;\text{ for all } j= 1, \cdots, 2+\log(B^2d)$$\; 
	    Compute a privatized histogram $\tilde{h}_j \leftarrow h_j + \cN(0,(4\sqrt{2d\log(1.25/\delta)}/(n\varepsilon))^2)$, for all  $j\in[2+\log(B^2d)]$\;
	    Set $\tilde\tau_j\leftarrow 2^{-3+j}$, for all  $j\in[2+\log(B^2d)]$\;
	    Find the largest $\ell\in[2+\log(B^2d)]$  satisfying $ \sum_{j \geq \ell} (\tilde\tau_j-\tilde\tau_\ell)\, \tilde h_j \geq 0.31 \tilde\psi $\;
	    \KwOut{$\rho=\tilde\tau_\ell $} 
    }
\end{algorithm2e} 

\subsubsection{Analysis of {\sc DPMMWfilter-ht} and a proof of Theorem~\ref{thm:accuracy_mmw_ht}}

Algorithm~\ref{alg:DPMMWfilter_ht} is a similar matrix multiplicative weights based filter algorithm for distributions with bounded covariance. Similarly, we first state following Lemma~\ref{lemma:invariant_ht} and prove Theorem~\ref{thm:accuracy_mmw_ht} given Lemma~\ref{lemma:invariant_ht}

\begin{lemma}
\label{lemma:invariant_ht}
   Let $S$ be an $\alpha$-corrupted bounded covariance dataset under Assumption~\ref{asmp:adversary2}. For an epoch $s$ and  an iteration $t$ such that $\lambda^{(s)}>C$, $\lambda_t^{(s)}> 2/3 \lambda_0^{(s)}$,  and $n^{(s)}>3n/4$,
if $n \gtrsim \frac{B^2 (\log B) d^{3/2}\log(1/\delta)}{\varepsilon} $   and $|S_t^{(s)}\cap S_{\rm good}|\geq (1-10\alpha)n$, 
    then with probability $1-O(1/\log(d)^3)$, we have the condition in 
	 Eq.~\eqref{eq:1dfilter_reg1_ht} holds. When this condition holds, we have more corrupted samples are removed in expectation than the uncorrupted samples, i.e., $\E|(S_t^{(s)}\setminus S_{t+1}^{(s)})\cap S_{\rm good}|\leq \E|(S_t^{(s)}\setminus S_{t+1}^{(s)})\cap S_{\rm bad}|$. 
Further, for an epoch $s\in[T_1]$ there exists a constant $C>0$  such that   if  $\|M(S^{(s)})\|_2\geq C$, then with probability $1-O(1/\log^2 d)$, the $s$-th epoch terminates after $O(\log d)$ iterations and outputs $S^{(s+1)}$ such that $\|M(S^{(s+1)})\|_2\leq 0.98\|M(S^{(s)})\|_2$.

\end{lemma}

Now we define  $d_{t}^{(s)} \triangleq|(S_{\mathrm{good}}\cap S^{(1)})\setminus S_{t}^{(s)}|+|S_{t}^{(s)}\setminus (S_{\mathrm{good}}\cap S^{(1)})|$. Note that $d_1^{(1)}=\alpha n$, and $d_t^{(s)}\geq0$. At each epoch and iteration, we have
	\begin{eqnarray*}
		\E [d_{t+1}^{(s)}-d_{t}^{(s)}|d_1^{(1)}, d_2^{(1)}, \cdots, d_{t}^{(s)}] &=& \E\left[|S_{\mathrm{good}}\cap(S_{t}^{(s)}\setminus S_{t+1}^{(s)})|-|S_{\mathrm{bad}}\cap(S_{t}^{(s)}\setminus S_{t+1}^{(s)})|\right] 
		\;\leq\; 0,
	\end{eqnarray*}
	from the part 1 of Lemma~\ref{lemma:invariant_ht}.
	Hence, $d_t^{(s)}$ is a non-negative  super-martingale. By optional stopping theorem, at stopping time, we have $\E[d_t^{(s)}]\leq d_1^{(1)}=\alpha n$. By Markov inequality, $d_t^{(s)}$ is less than $10\alpha n$ with probability $0.9$, i.e. $|S_t^{(s)}\cap S_{\mathrm{good}}|\geq (1-10\alpha )n$. 
The desired bound in Theorem~\ref{thm:accuracy_mmw_ht} follows from Lemma~\ref{lemma:reg2}.

\subsubsection{Proof of Lemma~\ref{lemma:invariant_ht}}
Lemma~\ref{lemma:invariant_ht} is a combination of Lemma~\ref{lemma:good_tau_ht}, Lemma~\ref{lemma:progress_iteration_ht} and Lemma~\ref{lemma:progress_epoch_ht}. We state the technical lemmas and subsequently provide the proofs. 
\begin{lemma}
\label{lemma:good_tau_ht}
For each epoch $s$ and iteration $t$, under the hypotheses of Lemma~\ref{lemma:invariant_ht} then with probability $1-O(1/\log^3 d)$, we have \begin{eqnarray}
\frac{1}{n} \sum_{i\in S_{\rm good}\cap S_t^{(s)}}\tau_i &\le& \psi/1000\;,\label{eq:1dfilter_reg1_ht}
\end{eqnarray}
where $\psi\triangleq \frac1n\sum_{i\in S_t^{(s)}}\tau_i$.
\end{lemma}

\begin{lemma}
\label{lemma:progress_iteration_ht}
	For each epoch $s$ and iteration $t$, under the hypotheses of Lemma~\ref{lemma:invariant_ht},  if condition Eq.~\eqref{eq:1dfilter_reg1_ht} holds, then we have  $\E|S_t^{(s)}\setminus S_{t+1}^{(s)}\cap S_{\rm good}|\leq \E|S_t^{(s)}\setminus S_{t+1}^{(s)}\cap S_{\rm bad}|$ and with probability $1-O(1/\log^3 d)$,  and $\ip{M(S_{t+1}^{(s)})}{U_t^{(s)}}\leq 0.76\ip{M(S_{t}^{(s)})}{U_t^{(s)}}$.
\end{lemma}

\begin{lemma}
\label{lemma:progress_epoch_ht}
	For epoch $s$, suppose for $t=0,1,\cdots,T_2$ where $T_2=O(\log d)$, if Lemma~\ref{lemma:progress_iteration_ht} holds, $n \gtrsim \frac{B^2 (\log B) d^{3/2}\log(1/\delta)}{\varepsilon\alpha} $, and $n^{(s)}>3n/4$, then we have $\|M(S^{(s+1)})\|_2\leq 0.98\|M(S^{(s)})\|_2$ with probability $1-O(1/\log^2 d)$.
\end{lemma}

\subsubsection{Proof of Lemma~\ref{lemma:good_tau_ht}}
\begin{proof}
By Lemma~\ref{lemma:lap_noise}, Lemma~\ref{lemma:chi-square} and Lemma~\ref{lemma:matrix_sp_norm}, we can pick $n=\widetilde{\Omega}\left( \frac{B^2d^{3/2}\log}{\varepsilon}\right)$ such that with probability $1-O(1/\log^3 d)$, following conditions simultaneously hold:
\begin{enumerate}
    \item $\|\mu_t^{(s)}-\mu(S_{t}^{(s)})\|_2^2\leq 0.001 $
    \item $|\psi_t^{(s)}-\ip{M(S_t^{(s)})}{U_t^{(s)}}|\leq  0.001 $
    \item $\left|\lambda_t^{(s)}-\|M(S_t^{(s)})\|_2\right|\leq 0.001$
    \item $	\left|\lambda^{(s)}-\|M(S^{(s)})\|_2\right|\leq 0.001$
    \item $\left\|M(S_{t+1}^{(s)})-\Sigma_t^{(s)}\right\|_2\leq 0.001 $ 
    \item $\|\mu^{(s)}-\mu(S^{(s)})\|_2^2\leq 0.001$\;.
\end{enumerate}

	Then we have
	\begin{eqnarray*}
	\frac{1}{n} \sum_{i\in S_{\rm good}\cap S_t^{(s)}} \tau_i &=& \frac{1}{n} \sum_{i\in S_{\rm good}\cap S_t^{(s)}} \ip{(x_i-\mu_t^{(s)})(x_i-\mu_t^{(s)})^\top}{U_t^{(s)}}\\
		&\overset{(a)}{\leq} & \frac{2}{n} \sum_{i\in S_{\rm good}\cap S_t^{(s)}} \ip{(x_i-\mu(S_{\rm good}\cap S_t^{(s)}))(x_i-\mu(S_{\rm good}\cap S_t^{(s)}))^\top}{U_t^{(s)}}\\
		&&+\frac{2|S_{\rm good}\cap S_t^{(s)}|}{n} \ip{(\mu(S_{\rm good}\cap S_t^{(s)})-\mu_t^{(s)})(\mu(S_{\rm good}\cap S_t^{(s)})-\mu_t^{(s)})^\top}{U_t^{(s)}} \\
		&\leq  &2\ip{M((S_{\rm good}\cap S_t^{(s)})}{U_t^{(s)}}+ 2 \|\mu_t^{(s)}-\mu(S_{\rm good}\cap S_t^{(s)})\|_2^2\\
		&\overset{(b)}{\leq}&2+ 2 \left(\|\mu_t^{(s)}-\mu\|_2+\|\mu(S_{\rm good}\cap S_t^{(s)})-\mu\|_2\right)^2\\
		&\overset{(c)}{\leq}&2+ 2 \left(0.01+2\sqrt{\alpha\|M(S_{t}^{(s)})\|_2}+3\sqrt{\alpha}\right)^2\\
		&\leq& 3+8\alpha\|M(S_{t}^{(s)})\|_2+32\alpha\\
		&\overset{(d)}{\leq} & \frac{\psi_t^{(s)}-0.002}{1000}\\
		&\leq & \frac{\psi}{1000}\;,
\end{eqnarray*}
where $(a)$ follows from the fact that for any vector $x, y, z$, we have $(x-y)(x-y)^{\top} \preceq 2(x-z)(x-z)^{\top}+2(y-z)(y-z)^{\top}$, $(b)$ follows from $\alpha$-goodness of $S_{\rm good}$, $(c)$ follows from Lemma~\ref{lemma:reg2} and $(d)$ follows from our choice of large constant $C$ and sample complexity $n$.

\end{proof}

\subsubsection{Proof of Lemma~\ref{lemma:progress_iteration_ht}}

\begin{proof}
	Lemma~\ref{lemma:good_tau_ht} implies with probability $1-O(1/\log^3 d)$, our scores satisfies the condition in Eq.~\eqref{eq:1dfilter_reg1_ht}. Then by Lemma~\ref{lemma:1dfilter-ht} our {\sc DPthreshold-ht} gives us a threshold $\rho$ such that 
	\begin{eqnarray*}
	    \sum_{i\in S_{\rm good}\cap S_t^{(s)}} \textbf{1}\{\tau_i\le \rho \}\frac{\tau_i}{\rho}+ \textbf{1}\{\tau_i > \rho \}   \le \sum_{i\in S_{\rm bad}\cap S_t^{(s)}} \textbf{1}\{\tau_i\le \rho \}\frac{\tau_i}{\rho} + \textbf{1}\{\tau_i > \rho \}\;.
	\end{eqnarray*}
	According to our filter rule from Algorithm~\ref{alg:1Dfilter_ht}, we have 
	\begin{eqnarray*}
	    \E|(S_t^{(s)}\setminus S_{t+1}^{(s)})\cap S_{\rm good}| = \sum_{i\in S_{\rm good}\cap S_t^{(s)}} \textbf{1}\{\tau_i\le \rho \}\frac{\tau_i}{\rho}+ \textbf{1}\{\tau_i > \rho \} 
	\end{eqnarray*}
	
	and \begin{eqnarray*}
	    \E|(S_t^{(s)}\setminus S_{t+1}^{(s)})\cap S_{\rm bad}| = \sum_{i\in S_{\rm bad}\cap S_t^{(s)}} \textbf{1}\{\tau_i\le \rho \}\frac{\tau_i}{\rho} + \textbf{1}\{\tau_i > \rho \}\;.
	\end{eqnarray*}

This implies $\E|(S_t^{(s)}\setminus S_{t+1}^{(s)})\cap S_{\rm good}|\leq  \E|(S_t^{(s)}\setminus S_{t+1}^{(s)})\cap S_{\rm bad}|$. 

At the same time, Lemma~\ref{lemma:1dfilter-ht} gives us a $\rho$ such that
with probability $1-O(\log^3 d)$, we have
	\begin{eqnarray*}
	    \frac{1}{n}\sum_{i\in S_{t+1}^{(s)}}\tau_i\leq \frac{1}{n}\sum_{ \tau_i\leq \rho, i\in S_t^{(s)}}\tau_i \leq \frac34\cdot\frac{1}{n}\sum_{i\in S_{t}^{(s)}}\tau_i \;.
	\end{eqnarray*}

	Hence, we have
	\begin{eqnarray*}
		\ip{M(S_{t+1}^{(s)})}{U_t^{(s)}}& =  & \ip{\frac{1}{n}\sum_{i\in S_{t+1}^{(s)}}(x_i-\mu(S_{t+1}^{(s)}))(x_i-\mu(S_{t+1}^{(s)}))^\top}{U_t^{(s)}}\\
		&\leq  & \ip{\frac{1}{n}\sum_{i\in S_{t+1}^{(s)}}(x_i-\mu(S_{t}^{(s)}))(x_i-\mu(S_{t}^{(s)}))^\top}{U_t^{(s)}}\\
		&\leq  & 
		\frac1n\sum_{i\in S_{t+1}^{(s)}}\tau_i+\|\mu_t^{(s)}-\mu(S_t^{(s)})\|_2^2\\
		&\leq & \frac{3}{4n}\sum_{i\in S_t^{(s)}}\tau_i+0.01\\
		&\overset{(a)}{\leq} & 0.76\ip{M(S_{t}^{(s)})}{U_t^{(s)}}\;,
	\end{eqnarray*}
	where $(a)$ follows from our assumption that $\psi_t^{(s)}>\frac{1}{5.5}\lambda_t^{(s)} > \frac{2}{ 16.5}C$.
	
\end{proof}

\subsubsection{Proof of Lemma~\ref{lemma:progress_epoch_ht}}
\begin{proof}

If Lemma~\ref{lemma:progress_iteration_ht} holds, we have
\begin{eqnarray*}
		\ip{M(S_{t}^{(s)})}{U_t^{(s)}}&\leq & 0.76\ip{M(S_{t-1}^{(s)})}{U_t^{(s)}}\\
	&\leq & 0.76\ip{M(S_{1}^{(s)})}{U_t^{(s)}}\\
	&\leq & 0.76\|M(S_{1}^{(s)})\|_2
	\end{eqnarray*}

We pick $n$ large enough such that with probability $1-O(\log^3 d)$, 
\begin{eqnarray*}
	\|\Sigma_t^{(s)}\|_2\approx_{0.05}\|M(S_{t}^{(s)})\|_2\;.
\end{eqnarray*}

Thus, we have
\begin{eqnarray*}
	\ip{\Sigma_t^{(s)}}{U_t^{(s)}}   \leq 0.81\|M(S_{1}^{(s)})\|_2\;.
\end{eqnarray*}
	
	By Lemma~\ref{lemma:mono}, we have $M(S_{t}^{(s)})\preceq M(S_{1}^{(s)})$. by our choice of $\alpha^{(s)}$, we have $\alpha^{(s)}M(S_{t+1}^{(s)})\preceq \frac{1}{100}{\mathbf I}$ and $\alpha^{(s)}\Sigma_t^{(s)}\preceq \frac{1}{100}{\mathbf I}$. Therefore, by Lemma~\ref{lemma:regret_bound} we have
	\begin{eqnarray*}
		&&\left\|\sum_{i=1}^{T_2}\Sigma_t^{(s)}\right\|_{2} \\
		&\leq & \sum_{t=1}^{T_2} \ip{ \Sigma_t^{(s)}}{U_t^{(s)}}+\alpha^{(s)}\sum_{t=0}^{T_2}\ip{ U_{t}^{(s)}}{\left|\Sigma_t^{(s)}\right|} \|\Sigma_t^{(s)}\|_2+\frac{\log (d)}{\alpha^{(s)}}\\
		&\overset{(a)}{\leq} & \sum_{t=1}^{T_2} \ip{ \Sigma_t^{(s)}}{U_t^{(s)}}+\frac{1}{100}\sum_{t=1}^{T_2}\ip{ U_{t}^{(s)}}{\left|\Sigma_t^{(s)}\right|} +200\log(d)\|M(S_1^{(s)})\|_2\\
	\end{eqnarray*}
	where $(a)$ follows from our choice of $\alpha^{(s)}$,  $C$, and $n$.

	Meanwhile, we have 
	\begin{eqnarray*}
		|\Sigma_t^{(s)}|\preceq M(S_{t}^{(s)})+0.15\; {\mathbf I}\;.
	\end{eqnarray*}
	
	Thus we have
	\begin{eqnarray*}
		\ip{ U_{t}^{(s)}}{\left|\Sigma_t^{(s)}\right|} \;
		\leq  0.91\;\left\|M(S_1^{(s)})\right\|_2
	\end{eqnarray*}

Then we have
	\begin{align*}
		&\left\|M(S_{T_2}^{(s)})\right\|_2
		\leq \frac{1}{T_2} \left\|\sum_{i=1}^{T_2}M({S_t^{(s)}})\right\|_{2} \\
		&\;\;\;\;\;\leq  \frac{1}{T_2}\left\|\sum_{i=1}^{T_2}\Sigma_t^{(s)}\right\|_{2}+0.05\; \|M(S_{1}^{(s)})\|_2\\
	&\;\;\;\;\;\leq  \frac{1}{T_2}\left(\sum_{t=1}^{T_2} \ip{ \Sigma_t^{(s)}}{U_t^{(s)}}+\frac{1}{100}\sum_{t=1}^{T_2}\ip{ U_{t}^{(s)}}{\left|\Sigma_t^{(s)}\right|} +200\log(d)\|M(S_1^{(s)})\|_2\right)+0.05\; \|M(S_{1}^{(s)})\|_2\\
	&\;\;\;\;\;\leq  0.91\|M(S_1^{(s)})\|_2 +\frac{200\log(d)}{T_2}\|M(S_1^{(s)})\|_2+0.05\; \|M(S_{1}^{(s)})\|_2\\
	&\;\;\;\;\;\leq  0.98\; \|M(S_{1}^{(s)})\|_2
 	\end{align*}
	\end{proof}

\subsubsection{Proof of {\sc DPthreshold-ht} for distributions with bounded covariance}

\begin{lemma}[{\sc DPthreshold-ht}: picking threshold privately for  distributions with bounded covariance]\label{lemma:1dfilter-ht}
Algorithm {\sc DPthreshold-ht}($\mu,U,\alpha,\varepsilon,\delta,S$) running on a dataset $\{\tau_i =  (x_i-\mu)^\top U (x_i-\mu) \}_{i\in S}$ is 
$(\varepsilon,\delta)$-DP. 
Define $\psi \triangleq \frac{1}{n}\sum_{i\in S}\tau_i$. 
If $\tau_i$'s satisfy 
\begin{eqnarray*}
\frac{1}{n} \sum_{i\in S_{\rm good}\cap S}\tau_i &\leq & \psi/1000\;,
\end{eqnarray*}
and $n\geq  \tilde\Omega\left(\frac{B^2 d}{\varepsilon } \right)$
then {\sc DPthreshold-ht}
 outputs a threshold $\rho$ such that \begin{eqnarray}
2(\sum_{i\in S_{\rm good}\cap S} \textbf{1}\{\tau_i\le \rho \}\frac{\tau_i}{\rho}+ \textbf{1}\{\tau_i > \rho \}   ) \le \sum_{i\in S_{\rm bad}\cap S} \textbf{1}\{\tau_i\le \rho \}\frac{\tau_i}{\rho} + \textbf{1}\{\tau_i > \rho \} \;,\label{eqn:good-bad-ratio-ht}
\end{eqnarray}
and with probability $1-O(1/\log^3 d )$, 
\begin{eqnarray*}
\frac{1}{n}\sum_{\tau_i<\rho }\tau_i  \;\; \le \;\; 0.75 \psi  \;.
\end{eqnarray*}

\end{lemma}

\begin{proof}

\textbf{1. $\rho$ cuts enough} 

Let $\rho$ be the threshold picked by the algorithm. Let $\hat{\tau}_i$ denote the minimum value of the interval of  the bin that $\tau_i$ belongs to. It holds that
\begin{align*}
&\frac{1}{n}\sum_{\tau_i\ge \rho, i\in[n]}({\tau}_i - \rho) 
 \ge \frac{1}{n}\sum_{\tilde\tau_i\ge \rho, i\in[n]}({\hat\tau}_i - \rho) \nonumber\\
& \;\;\;\;\;= \sum_{\tilde\tau_j\ge\rho, j\in[2+\log(B^2d)]}(\tilde\tau_j-\rho){h}_j\nonumber\\
&\;\;\;\;\;\overset{(a)}{\ge} \sum_{\tilde\tau_j\ge\rho, j\in[2+\log(B^2d)]}(\tilde\tau_j-\rho)\tilde{h}_j- O\left(\log(B^2d)\cdot B^2d\cdot \frac{\sqrt{\log(\log(B^2d)\log d)\log(1/\delta)}}{\eps n}\right)\nonumber\\
&\;\;\;\;\;\overset{(b)}{\ge} 0.31 \tilde\psi -  \tilde{O}(\frac{B^2d}{\eps n})\nonumber\\
&\;\;\;\;\;\overset{(c)}{\geq}  0.3 \psi -  \tilde{O}(\frac{B^2d}{\eps n})\nonumber\;,
\end{align*}
where $(a)$ holds due to the accuracy of the private histogram (Lemma~\ref{lem:hist}), $(b)$ holds by the definition of $\rho$ in our algorithm, and $(c)$ holds due to the accuracy of $\tilde{\psi}$. This implies 
\begin{eqnarray*}
\frac{1}{n}\sum_{\tau_i < \rho}{\tau}_i \leq \psi- \frac{1}{n}\sum_{\tau_i \geq \rho}({\tau}_i-\rho) \le 0.7\psi +  \tilde{O}({B^2d/\eps n}).
\end{eqnarray*}

\textbf{2. $\rho$ doesn't cut too much} 

Define $C_2$ to be the threshold such that $\frac{1}{n}\sum_{\tau_i>C_2}(\tau_i-C_2)= (2/3) \psi$. Suppose $2^b\le C_2 \le 2^{b+1}$, we have $\sum_{\hat\tau_i\ge 2^{b-1}} (\hat\tau_i-2^{b-1})\ge (1/3) \psi $ because $\forall \tau_i\ge C_2$, $(\hat\tau_i-2^{b-1}) \ge \frac{1}{2}(\tau_i-C_2) $. Then the threshold picked by the algorithm $\rho\ge 2^{b-1}$, which implies $\rho \ge \frac{1}{4}C_2$. 
Suppose $\rho<C_2$, since $\rho \ge \frac{1}{4}C_2$
\begin{eqnarray*}
\sum_{i\in S_{\rm bad}\cap S, \tau_i<\rho} \tau_i + \sum_{i\in S_{\rm bad}\cap S , \tau_i\geq\rho} \rho &\ge &\frac{1}{4}(\sum_{i\in S_{\rm bad}\cap S, \tau_i<C_2} \tau_i + \sum_{i\in S_{\rm bad}\cap S, \tau_i\geq C_2} C_2)\\
&\overset{(a)}{\ge} &\frac{10}{4}(\sum_{i\in S_{\rm good}\cap S, \tau_i<C_2} \tau_i + \sum_{i\in S_{\rm good}\cap S, \tau_i\geq C_2} C_2)\\
&\overset{(b)}{\ge} &\frac{10}{4}(\sum_{i\in S_{\rm good}\cap S, \tau_i<\rho} \tau_i + \sum_{i\in S_{\rm good}\cap S, \tau_i\geq \rho} \rho),
\end{eqnarray*}
where (a) holds by Lemma~\ref{lemma:bad-good-ratio-ht}, and (b) holds since $\rho\le C_2$.
If $\rho\ge C_2$, the statement of the Lemma~\ref{lemma:bad-good-ratio-ht} directly implies Equation~\eqref{eqn:good-bad-ratio-ht}.

\begin{lemma}
\label{lemma:bad-good-ratio-ht}
Assuming that the condition in Eq.\eqref{eq:1dfilter_reg1_ht} holds, then for any $C$ such that
\begin{eqnarray*}
\frac{1}{n}\sum_{i\in S, \tau_i < C }\tau_i + \frac{1}{n}\sum_{i\in S, \tau_i \ge C}C\ge (1/3) \psi\;,
\end{eqnarray*}
we have
\begin{eqnarray*}
\sum_{i\in S_{\rm bad}\cap S, \tau_i < C} \tau_i + \sum_{i\in S_{\rm bad}\cap S, \tau_i \ge C} C \ge {10}( \sum_{i\in S_{\rm good}\cap S, \tau_i < C} \tau_i + \sum_{i\in S_{\rm good}\cap S, \tau_i \ge C} C )
\end{eqnarray*}
\end{lemma}
\begin{proof}
First we show an upper bound on $S_{\rm good}$:

\begin{eqnarray*}
\frac{1}{n}\sum_{i\in S_{\rm good}\cap S, \tau_i < C} \tau_i + \frac{1}{n} \sum_{i\in S_{\rm good}\cap S, \tau_i \ge C} C  \le \frac{1}{n}\sum_{i\in S_{\rm good}\cap S} \tau_i \le \psi/1000.
\end{eqnarray*} 

Then we show an lower bound on $S_{\rm bad}$:
\begin{eqnarray*}
&&\frac{1}{n}\sum_{i\in S_{\rm bad}\cap S, \tau_i<C} \tau_i + \frac{1}{n}\sum_{i\in S_{\rm bad}\cap S, \tau_i>C}C\\
=&& \frac{1}{n}\sum_{i\in S, \tau_i<C} \tau_i + \frac{1}{n}\sum_{i\in S, \tau_i\ge C} C\\
&&- (\frac{1}{n}\sum_{i\in S_{\rm good}\cap S, \tau_i<C} \tau_i + \frac{1}{n}\sum_{i\in S_{\rm good}\cap S, \tau_i\ge C} C)\\ 
\ge &&(1/3-1/1000)\psi\;.
\end{eqnarray*}

Combing the lower bound and the upper bound yields the desired statement
\end{proof}

\end{proof}

\subsubsection{Regularity lemmas for distributions with bounded covariance}

\begin{definition}[{\cite[Definition 3.1]{dong2019quantum}} ]
Let $D$ be a distribution with mean $\mu\in \reals^d$ and covariance $\Sigma\preceq {\mathbf I}$. For $0<\alpha<1/2$, we say a set of points $S=\{X_1, X_2,\cdots, X_n\}$ is $\alpha$-good with respect to $\mu\in \reals^d$ if following inequalities are satisfied:
\begin{itemize}
	\item $\|\mu(S)-\mu\|_2\leq \sqrt{\alpha} $ 
	\item  $\left\|\frac{1}{|S|} \sum_{i \in S}\left(X_{i}-\mu(S)\right)\left(X_{i}-\mu(S)\right)^{\top}\right\|_{2} \leq 1$.
\end{itemize}	
\end{definition}

\begin{lemma}[{\cite[Lemma 3.1]{dong2019quantum}} ]
Let $D$ be a distribution with mean $\mu\in \reals^d$ and covariance $\Sigma\preceq {\mathbf I}$. Let $S=\{X_1,X_2,\cdots,X_n\}$ be a set of i.i.d. samples of $D$. If $n=\Omega(d\log(d)/\alpha)$, then with probability $1-O(1)$,   there exists a set $S_{\rm good}\subseteq S$ such that $S_{\rm good}$ is $\alpha$-good with respect to $\mu$ and $|S_{\rm good}|\geq (1-\alpha)n$.
\end{lemma}

\begin{lemma}[{\cite[Lemma 3.2]{dong2019quantum}}  ]\label{lemma:reg2}
 Let $S$ be an $\alpha$-corrupted bounded covariance dataset under Assumption~\ref{asmp:adversary2}. If $S_{\rm good}$ is $\alpha$-good with respect to $\mu$, then for any $T\subset S$ such that $|T\cap S_{\rm good}|\geq (1-\alpha)|S|$, we have
\begin{eqnarray*}
	\|\mu(T)-\mu\|_{2} \leq \frac{1}{1-2\alpha} \cdot \left( 2\sqrt{\alpha\left\|M(T)\right\|_{2}}+3 \sqrt{\alpha}\right)\;.
\end{eqnarray*}
\end{lemma}

 \newpage
\section{Experiments} 
\label{sec:experiments} 

We evaluate  {\sc PRIME} and compare with a DP mean estimator of \cite{KLSU19} on synthetic dataset in Figure~\ref{fig:intro} and Figure~\ref{fig:n}, which consists of samples from $(1-\alpha)\cN(0, \mathbf{I})+\alpha\cN(\mu_{\rm bad}, \mathbf{I})$. The main focus of this evaluation is to compare the estimation error and demonstrate the robustness of {\sc PRIME} under differential privacy guarantees. Our choice of experimental  settings and hyper parameters are as follows: $ 1\leq d \leq 100$, $\mu_{\rm bad}=(1.5,1.5,\cdots, 1.5)_d$, $0.001\leq \varepsilon\leq 100$, $0.01\leq\alpha\leq 0.1$ , $C=1$. 

Figure~\ref{fig:n} shows additional experiments including the regime where we do not have enough number of samples. When $n\leq cd^{1.5}/\alpha\varepsilon$, the utility guarantee (Theorem~\ref{thm:main1}) does not hold. The noise we add on the final output becomes large as $n$ decreases and dominates the estimation error. The DP Mean \cite{KLSU19} has lower error compared to PRIME when $n$ is small because PRIME spends some privacy budget to perform  operations other than those in DP Mean in the Algorithm~\ref{alg:DPMMWfilter}. In practice, we can check whether there are enough number of samples based on known parameters $(\varepsilon, \delta, n, \alpha)$, and  choose to use DP Mean (or adjust how the privacy budget is distributed in PRIME). 

\begin{figure}[h]
    \centering
	\includegraphics[width=0.5\linewidth]{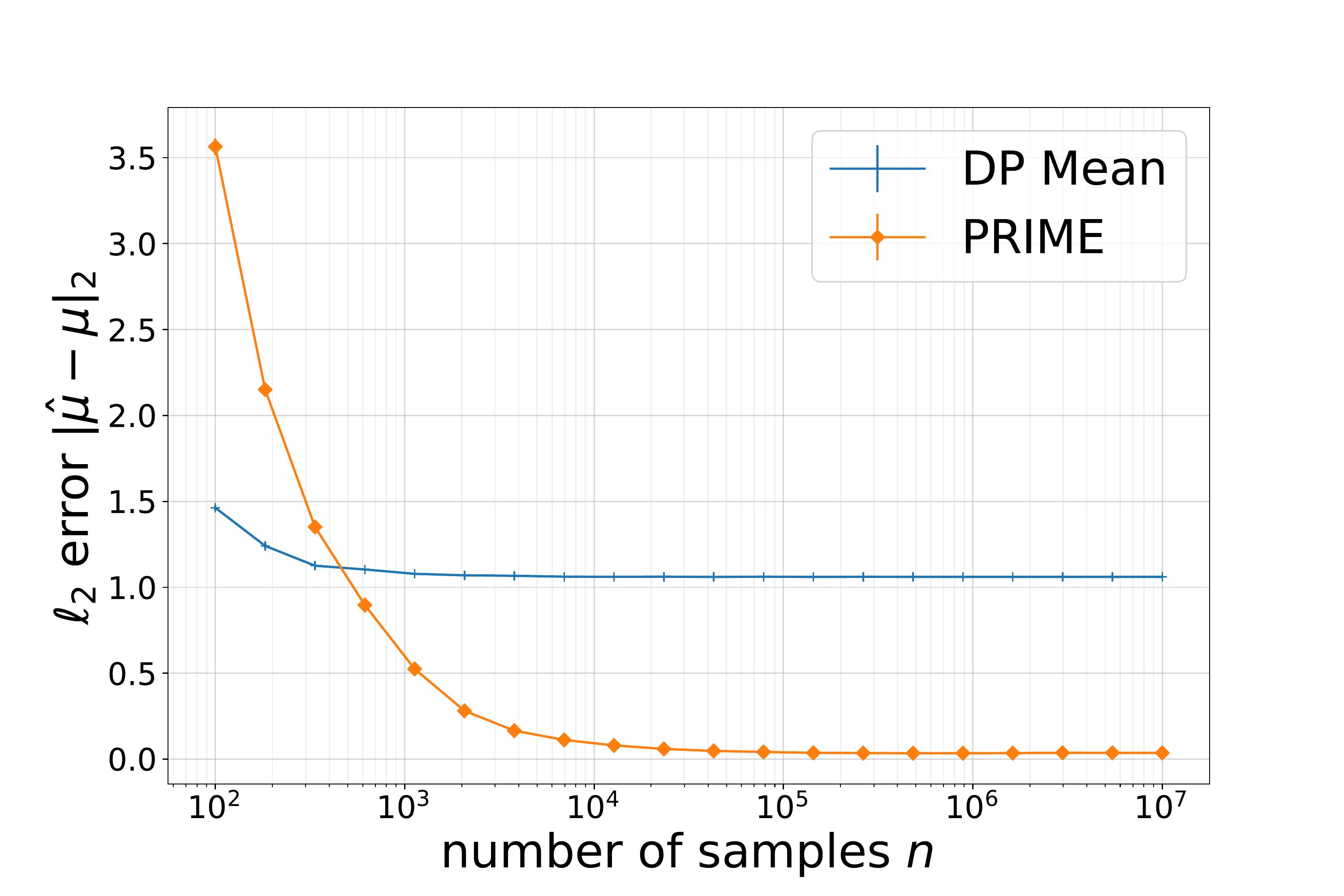}
	\caption{Estimation error achieved by PRIME significantly improves upon that of DP Mean in the large sample regime where our theoretical guarantees apply.
	In the small sample regime, the noise from the DP mechanisms dominate the error, which increases with decreasing $n$. 
	We choose $(\alpha, \varepsilon, \delta, d)= (0.1, 100, 0.01, 50)$.  Each data point is repeated 50 runs and standard error is shown in  the error bar.}
    \label{fig:n}
\end{figure}

Our implementation is based on Python with basic Numpy library. We run on a 2018 Macbook Pro machine. For each choice of $d$ in our settings, it takes less than $2$ minutes and {\sc PRIME} stops after at most $3$ epochs. We have attached our code as supplementary materials. 
\end{document}